\documentclass[10pt]{article}

\usepackage[preprint]{tmlr}

\usepackage{float}
\usepackage{tikz}
\usetikzlibrary{arrows.meta,positioning,calc,fit,shapes.geometric}

\usepackage{amsmath,amssymb,amsfonts,amsthm,mathtools,bm}
\usepackage{booktabs}
\usepackage{enumitem}
\usepackage{graphicx}
\usepackage{microtype}
\usetikzlibrary{arrows.meta,positioning,calc,fit,shapes.geometric}
\usepackage{hyperref}
\usepackage{url}
\usepackage[nameinlink,capitalise]{cleveref}
\hypersetup{
    colorlinks=true,
    linkcolor=blue,
    citecolor=blue,
    urlcolor=blue
}

\setlist{nosep,leftmargin=1.5em}

\providecommand{\E}{\mathbb{E}}
\providecommand{\KL}{\mathrm{KL}}
\providecommand{\MI}{\mathrm{I}}
\providecommand{\N}{\mathcal{N}}
\newcommand{\sg}{\operatorname{sg}}
\newcommand{\softplus}{\operatorname{softplus}}
\newcommand{\Cat}{\operatorname{Cat}}

\newcommand{\cL}{\mathcal{L}}
\newcommand{\cR}{\mathcal{R}}

\newtheorem{proposition}{Proposition}

\theoremstyle{definition}

\theoremstyle{remark}

\title{
MDN-VJEPA: Multimodal Probabilistic Joint-Embedding World Models
}

\author{
\name Yongchao Huang
\email{yongchao.huang@abdn.ac.uk} \\
\addr Department of Computing Science \\
University of Aberdeen
}

\def\month{08}
\def\year{2026}
\def\openreview{}

\begin{document}

\maketitle

\begin{abstract}
Joint-Embedding Predictive Architectures (JEPAs) learn world models by predicting future representations rather than reconstructing observations. Variational JEPA (VJEPA) gives this prediction probabilistic semantics and permits flexible conditional predictive families. We propose MDN-VJEPA, a multimodal probabilistic JEPA that instantiates the latent predictor as a mixture density network (MDN), representing future states with a Gaussian mixture whose components encode alternative futures, their probabilities, and within-mode uncertainty. For sequential prediction, we address two additional structural requirements: multi-modal futures should remain temporally coherent across a trajectory, and predictive beliefs should remain consistent when propagated to the same horizon through different temporal path compositions, as motivated by the Chapman--Kolmogorov equation. Together, these considerations define a three-level hierarchy of multimodal prediction, coherent multimodal trajectories, and compositionally consistent multimodal dynamics. Controlled experiments show that MDN-VJEPA closely recovers known branching predictive laws, that retaining the complete multimodal belief rather than its conditional mean can lead to materially different and substantially better planning decisions under branching futures, and that trajectory-level coupling and distributional consistency improve temporal coherence and multi-horizon agreement. These results support a simple principle: a probabilistic JEPA should preserve not only how uncertain the future is, but which distinct futures constitute that uncertainty, how they form coherent trajectories, and how those beliefs compose through time.
\end{abstract}

\tableofcontents
\newpage

\section{Introduction}

Joint-Embedding Predictive Architectures (JEPAs) learn representations by predicting the embedding of missing or future information from available context, avoiding the need to reconstruct all observation-level details \citep{lecun2022path,assran2023ijepa,bardes2024vjepa}. This makes JEPA attractive as a world-modeling principle: the model can concentrate capacity on predictable structure useful for understanding, prediction, and planning rather than on every pixel-level nuisance. Recent video and action-conditioned systems have strengthened this connection to physical planning \citep{assran2025vjepa2,terver2026jepawm}.\footnote{We use \emph{V-JEPA} for Video JEPA \citep{bardes2024vjepa} and \emph{VJEPA} for Variational JEPA \citep{huang2026vjepa}.}

Standard JEPA predictors are typically point-valued. When the same observed history can lead to multiple valid futures, squared-error prediction recovers their conditional mean \citep{huang2026vjepa}, which need not itself correspond to a realizable future. Variational JEPA (VJEPA) addresses this limitation by replacing point prediction with an explicit conditional distribution over future latent states \citep{huang2026vjepa}, while PIB-VJEPA retains this general probabilistic formulation and further incorporates a full predictive information bottleneck \citep{huang2026ibvjepa}. Neither VJEPA nor PIB-VJEPA is intrinsically restricted to a particular predictive distribution family.\footnote{VJEPA is not restricted to a Gaussian predictive family and explicitly identifies Gaussian-mixture predictive heads as a natural extension for multimodal futures \citep{huang2026vjepa}.} A particularly simple and direct realization is \emph{Gaussian VJEPA}, in which the predictor parameterizes a Gaussian distribution over future latent states. This formulation represents uncertainty around a single predictive center and therefore remains unimodal. When the future bifurcates into several well-separated alternatives, increasing the Gaussian covariance can express greater overall dispersion, but cannot represent that predictive probability mass is concentrated around multiple distinct futures. This motivates MDN-VJEPA, which instantiates the general probabilistic JEPA predictor as a Gaussian mixture whose parameters are inferred by a neural network.

\paragraph{A motivating robot-navigation problem.}
Consider a robot approaching an obstacle from the same observed context, as shown in \cref{fig:robot-motivation}. Because of hidden intent, interaction noise, or an unobserved variable, the future may branch: the robot may pass on the left or on the right. A deterministic predictor can average these valid futures into a single latent prediction; in observation space, the analogous mean route may pass directly through the obstacle and therefore correspond to no realizable future. A single Gaussian predictor can broaden the predictive belief around its conditional mean, but the resulting density remains unimodal and may place substantial probability mass between the valid branches. By contrast, MDN-VJEPA can represent several explicit alternatives, each equipped with both a probability and a local uncertainty region.

\begin{figure}[t]
\centering
\begin{tikzpicture}[
    x=0.92cm,y=0.92cm,
    arr/.style={-{Latex[length=1.8mm]},thick},
    darr/.style={-{Latex[length=1.8mm]},thick,dashed},
    prob/.style={draw,rounded corners=2pt,minimum width=1.45cm,minimum height=0.52cm,inner sep=1.5pt,font=\small},
    txt/.style={font=\small,align=center},
    title/.style={font=\small\bfseries,align=center}
]

\node[draw,rounded corners,minimum width=1.9cm,minimum height=0.95cm,font=\small] (obs) at (0,0.45) {obstacle};
\node[circle,draw,minimum size=0.48cm,font=\small] (robot) at (0,-1.85) {Robot};
\node[txt] at (0,-2.65) {same observed context};

\draw[arr] (robot.north west) .. controls (-0.85,-1.1) and (-1.45,-0.15) .. (-1.75,1.45);
\draw[arr] (robot.north east) .. controls (0.85,-1.1) and (1.45,-0.15) .. (1.75,1.45);
\draw[darr] (robot.north) -- (0,-0.08);

\node[prob] (leftbox) at (-1.95,2.15) {$\pi_1=0.45$};
\node[prob] (rightbox) at (1.95,2.15) {$\pi_2=0.55$};

\node[txt] at (-1.95,2.88) {$\mu_1,\Sigma_1$\\left-route belief};
\node[txt] at (1.95,2.88) {$\mu_2,\Sigma_2$\\right-route belief};

\node[txt] at (0,-0.55) {invalid mean\\prediction};

\node[title] at (0,3.58) {MDN-VJEPA: distinct prediction modes remain distinct};

\end{tikzpicture}
\caption{An motivating branching prediction problem. A point predictor may average two valid futures into an invalid route, while a single Gaussian can represent them only through a broad unimodal belief that places probability mass between the branches. A mixture predictor can instead preserve separate alternatives and quantify both their probabilities and within-mode uncertainty. The same principle applies in latent space rather than only in observation space.}
\label{fig:robot-motivation}
\end{figure}
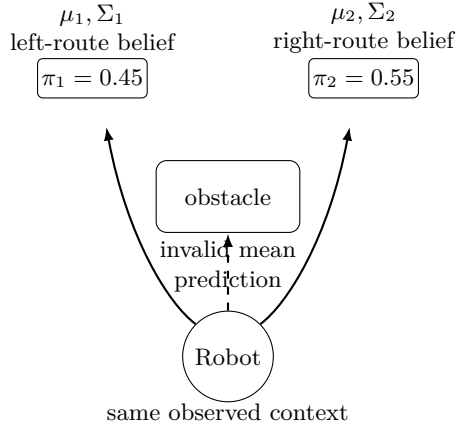

This motivates a stronger principle than uncertainty estimation alone: \textit{a probabilistic world model should preserve distinct plausible futures rather than average them together}. It should represent not only \emph{how uncertain} the future is, but also \emph{which alternatives} constitute that uncertainty. We therefore propose MDN-VJEPA, which instantiates the VJEPA predictor as a neural Gaussian-mixture transition. For prediction horizon $h$,
\begin{equation}
p_\phi
\left(
Z_{t+h}\mid Z_t,\xi
\right)
=
\sum_{k=1}^{K}
\pi_{\phi,k}(Z_t,\xi)\,
\N\left(
Z_{t+h};
\mu_{\phi,k,h}(Z_t,\xi),
\Sigma_{\phi,k,h}(Z_t,\xi)
\right),
\label{eq:intro-mdn} \tag{\ref{eq:level1-marginal}}
\end{equation}
where $\xi$ may encode elapsed time, action sequences, target positions, goals, or other side information such as physical constraints. The mixture exposes three complementary quantities: the component mean $\mu_{k,h}$ specifies where future hypothesis $k$ lies, the covariance $\Sigma_{k,h}$ quantifies uncertainty within that hypothesis, and the weight $\pi_k$ specifies the predictive probability assigned to it.

\paragraph{Why multimodality matters.}
Retaining the complete predictive mixture provides capabilities unavailable to a point or single-Gaussian prediction. It preserves separated alternatives rather than collapsing them into a conditional average, distinguishes uncertainty within a future from uncertainty over which future occurs, and provides explicitly weighted hypotheses for sampling and decision making. This distinction can be consequential in control: evaluating a cost at the expected future is generally not equivalent to averaging that cost over the possible futures. A rare but important failure branch, for example, may disappear in a mean prediction while remaining explicit in a multimodal predictive belief. This downstream distinction is a central reason to model the structure of predictive uncertainty rather than only a point forecast or scalar measure of confidence.

For video and multi-step prediction, however, representing multiple plausible futures is not enough: each predicted future must also remain \textit{temporally coherent}. If a separate mixture is modeled at each future horizon, sampling component identities independently can splice together incompatible alternatives---for example, a ``left'' future at one frame, ``right'' at the next, and ``left'' again thereafter. Accurate multimodal marginals therefore need not define an accurate joint trajectory distribution. We address this by introducing a \emph{trajectory-level mode}: a categorical variable $M$ is selected once for the prediction horizon and governs the corresponding future trajectory. Concretely,
\begin{equation}
M \sim \Cat(\pi_1,\ldots,\pi_K),
\qquad
M=k
 \Longrightarrow 
(\mu_{k,1},\mu_{k,2},\ldots,\mu_{k,H}),
\label{eq:shared-mode-intro} \tag{Eq.\ref{eq:traj-mode} \& Eq.\ref{eq:same-k-trajectory}}
\end{equation}
so that the same mixture component $k$ parameterizes the entire future block rather than selecting unrelated component indices at individual frames or tokens. Each component can therefore parameterize a complete candidate future trajectory, subject to the usual label-switching ambiguity of mixture models.\footnote{A learned GMM component may specialize to a semantically recognizable trajectory mode, but this is an empirical property rather than an identifiability guarantee. Component labels may permute, split, or change meaning across predictive contexts.}

\textit{Temporal coherence} alone is still not sufficient for a sequential world model: predictions should also remain \textit{consistent across prediction horizons} \citep{huang2026HMMVJEPA}. A future distribution predicted directly from $Z_t$ should agree, approximately, with the distribution obtained by recursively composing the learned one-step transition to the same horizon.\footnote{The comparison between direct multi-step prediction and recursively composed transitions is termed \emph{path composition} in \citet{huang2026HMMVJEPA}.} We therefore introduce a \emph{distributional path-consistency} principle, motivated by the Chapman--Kolmogorov identity, that compares complete predictive distributions rather than only their means\footnote{In our later Level-III implementation evaluated in Experiment~4, the separately supervised direct prediction acts as a stop-gradient reference belief within the CK term, while the recursively composed transition is regularized toward it using a sample-based distributional discrepancy.}.

These ideas give a natural three-stage progression of the predictive model: first preserve multiple plausible futures, then organize them into temporally coherent trajectories, and finally encourage those trajectory beliefs to compose consistently through time. Orthogonal to these three predictive levels, we retain PIB-VJEPA as the common probabilistic representation-learning foundation. The resulting model therefore combines three complementary representation and prediction principles:
{\small
\begin{equation}
\underbrace{\textit{PIB: compress non-predictive history}}_{\textit{compact predictive state}}
 + 
\underbrace{\textit{JEPA: retain predictive information}}_{\textit{representation learning}}
 + 
\underbrace{\textit{MDN: preserve alternative futures}}_{\textit{multimodal prediction}}.
\label{eq:conceptual-synthesis}
\end{equation}
}
PIB encourages a compact predictive state, JEPA preserves information needed to predict the future, and the MDN provides the predictive family needed to retain distinct plausible futures. Temporal coherence and distributional path consistency then impose additional structure on how those multimodal beliefs evolve across a prediction horizon.

\paragraph{Model hierarchy.}
The relationship among the predictive families is shown in \cref{fig:model-hierarchy}. Gaussian VJEPA is recovered from MDN-VJEPA by setting $K=1$, while MDN-VJEPA is itself one predictive-family choice within the more general VJEPA framework. Deterministic squared-error JEPA is related to Gaussian VJEPA at the objective level: with a deterministic target representation and fixed isotropic predictive covariance, Gaussian negative log-likelihood is equivalent, up to constants and positive scaling, to squared prediction error \citep{huang2026vjepa}.

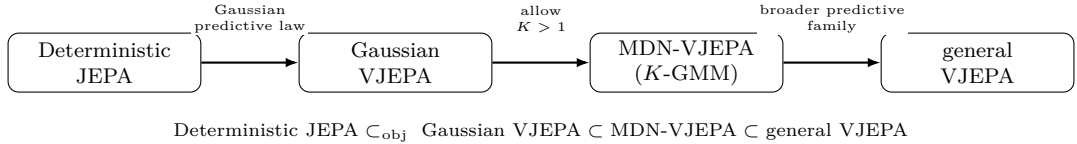
\begin{figure}[H]
\centering
\begin{tikzpicture}[
    box/.style={
        draw,
        rounded corners,
        minimum width=2.55cm,
        minimum height=0.78cm,
        align=center,
        font=\footnotesize,
        inner sep=2pt
    },
    arr/.style={-{Latex[length=1.7mm]},thick},
    lab/.style={
        font=\tiny,
        align=center,
        fill=white,
        inner sep=1pt
    }
]

\node[box] (d) at (0,0) {Deterministic\\JEPA};
\node[box] (g) at (3.85,0) {Gaussian\\VJEPA};
\node[box] (m) at (7.70,0) {MDN-VJEPA\\($K$-GMM)};
\node[box] (v) at (11.55,0) {general\\VJEPA};

\draw[arr] (d.east) -- (g.west);
\draw[arr] (g.east) -- (m.west);
\draw[arr] (m.east) -- (v.west);

\node[lab] at (1.925,0.58) {Gaussian\\predictive law};
\node[lab] at (5.775,0.58) {allow\\$K>1$};
\node[lab] at (9.625,0.58) {broader predictive\\family};

\node[font=\scriptsize,align=center] at (5.775,-0.92)
{$\text{Deterministic JEPA}
\subset_{\mathrm{obj}}
\text{ Gaussian VJEPA}
\subset
\text{MDN-VJEPA}
\subset
\text{general VJEPA}$};

\end{tikzpicture}
\caption{Nested predictive hierarchy. Gaussian VJEPA is recovered from MDN-VJEPA by setting $K=1$, and general VJEPA permits predictive families beyond Gaussian mixtures. The notation $\subset_{\mathrm{obj}}$ denotes the objective-level relation between deterministic squared-error JEPA and Gaussian VJEPA: with a deterministic target representation and fixed isotropic covariance, Gaussian negative log-likelihood reduces to squared prediction error up to an additive constant and positive scale factor.}
\label{fig:model-hierarchy}
\end{figure}

\paragraph{Novelty and contributions.}
The novelty does not lie in the Gaussian-mixture head itself: MDNs are classical \citep{bishop1994mdn}, MDN latent dynamics were already used in \emph{World Models} \citep{ha2018worldmodels}, VJEPA identifies GMM predictive heads as a natural extension for multimodal futures \citep{huang2026vjepa}, and recent Branch-JEPA work studies finite-support multimodal JEPA prediction and includes an MDN baseline \citep{song2026branchjepa}. Our contribution is instead to develop mixture prediction into \textit{a coherent, uncertainty-aware, and compositionally consistent probabilistic JEPA dynamics model}. Specifically:
\begin{enumerate}[label=(\roman*)]
    \item \textbf{Multimodal probabilistic JEPA.} We instantiate the VJEPA predictive distribution as a \textit{neural GMM} and formalize the resulting continuous-state Markov kernel, its uncertainty decomposition, and its nested relation to deterministic and Gaussian JEPA.
    \item \textbf{Coherent trajectory-level modes.} A shared latent mode controls an entire future block, so each mixture component parameterizes a complete candidate trajectory with its own probability and uncertainty ``tube'', rather than a collection of independently selected horizon-wise components.
    \item \textbf{Distributional multi-horizon consistency.} We regularize direct and recursively composed predictive distributions toward agreement using a permutation-invariant sample discrepancy, encouraging an approximate Chapman--Kolmogorov consistency condition without aligning mixture labels by hand.
    \item \textbf{Predictive-information and decision interpretation.} We combine the transition with PIB compression, separate within-mode uncertainty from between-mode uncertainty, identify distinct mixture/variance collapse failure modes, and show why retaining multimodal predictive beliefs matters for risk-sensitive planning under branching futures.
\end{enumerate}

\section{Related Work}

\subsection{JEPA, VJEPA, and predictive information}

JEPA was proposed as a representation-space alternative to observation reconstruction \citep{lecun2022path}; I-JEPA instantiated masked prediction for images \citep{assran2023ijepa}, and V-JEPA extended representation prediction to video \citep{bardes2024vjepa}. V-JEPA 2-AC and recent JEPA world-model studies further use action-conditioned latent prediction for physical planning \citep{assran2025vjepa2,terver2026jepawm}. VJEPA makes target prediction probabilistic and connects JEPA to predictive states, uncertainty-aware planning, and variational information objectives \citep{huang2026vjepa}. PIB-VJEPA further emphasizes a compressed predictive state that retains information about the future while removing unnecessary information from the observed history \citep{huang2026ibvjepa}.

Our starting point is the conditional predictive family itself, but the present work additionally asks how multimodal predictive beliefs should be coupled across time and composed across prediction horizons. A Gaussian VJEPA predictor can quantify dispersion around a single predictive center, but remains unimodal; VJEPA explicitly identifies more expressive distributions such as GMMs as a natural extension for bifurcating futures \citep{huang2026vjepa}. MDN-VJEPA develops this direction by combining a continuous multimodal predictive law with trajectory-level mode coherence and multi-horizon distributional consistency.

A related line considers probabilistic joint embeddings rather than temporal transition kernels. Gaussian Joint Embeddings (GJE) and Gaussian Mixture Joint Embeddings (GMJE) model joint context--target densities and derive conditional predictive distributions, including a conditional MDN construction \citep{huang2026gmje}. MDN-VJEPA instead treats the mixture as a sequential transition law over future JEPA states, couples its components across a future trajectory, and regularizes direct and recursively composed transition distributions.

Closely related to our multi-horizon construction, Markov chain JEPA (MCVJEPA) interprets probabilistic JEPA dynamics through a Markov-model lens and motivates consistency between direct and composed transition paths \citep{huang2026HMMVJEPA}. MDN-VJEPA develops this principle for continuous multimodal transition laws. Because nonlinear MDN transitions are generally not closed under recursive composition, the resulting direct and composed predictive laws need not remain finite GMMs of the same form; we therefore compare them distributionally through samples rather than requiring component-wise correspondence.

\subsection{Mixture density networks and multimodal latent world models}

Mixture Density Networks (MDNs) combine a neural network with an explicit mixture model to estimate a conditional density. Bishop's original motivation was precisely that conditional averages can be inappropriate for multi-valued regression problems, where several distinct outputs may be plausible for the same input \citep{bishop1994mdn}. 
The closest historical world-model precedent is the MDN-RNN of Ha and Schmidhuber \citep{ha2018worldmodels}, which predicts a Gaussian mixture over future VAE latents. MDN-VJEPA differs in its representation-learning framework and dynamics constraints: it operates on JEPA/PIB predictive representations without requiring an observation decoder or pixel likelihood, interprets the mixture as a probabilistic JEPA transition, organizes multimodal predictions into coherent and consistent future trajectories, and explicitly regularizes direct and recursively composed predictive distributions across horizons.

Recent Branch-JEPA provides an important contemporary comparison \citep{song2026branchjepa}. It represents the conditional future by a context-weighted finite atomic measure
$
\sum_{k}\pi_k\delta_{u_k},
$
and reports strong preservation of separated predictive support, including comparisons with an MDN baseline. MDN-VJEPA instead uses continuous mixture components
$
\sum_k
\pi_k
\N(\mu_k,\Sigma_k).
$
The distinction\footnote{Later in Experiment~2, we include a controlled Branch-JEPA-style finite-support predictor as an atomic multimodal reference; this controlled baseline is not intended as a reproduction of the complete Branch-JEPA training system.} is complementary rather than uniformly favorable to either representation: finite atomic branches can preserve sharp separated support particularly well, whereas Gaussian components additionally parameterize uncertainty \emph{within} each branch. 

Multimodal prediction over complete trajectories also has substantial precedent outside JEPA. MultiPath represents future behavior using probabilistically weighted trajectory hypotheses with uncertainty around the corresponding trajectories \citep{chai2019multipath}, while Trajectron and Trajectron++ model multimodal distributions over complete future trajectories in multi-agent forecasting \citep{ivanovic2019trajectron,salzmann2021trajectronpp}. These works establish the importance of maintaining multiple coherent trajectory hypotheses rather than independent multimodal predictions at individual time steps. Our contribution is different in scope: we introduce trajectory-level mixture coherence into a probabilistic JEPA world model and connect it to recursive latent transition composition through a distributional Chapman--Kolmogorov consistency objective.

Concurrent work also explores alternative belief-carrying JEPA state representations. UWM-JEPA, for example, represents uncertainty through a density-matrix latent state and propagates it using a learned unitary predictor \citep{radha2026uwmjepa}. MDN-VJEPA instead retains a conventional continuous JEPA state and places the multimodal predictive belief in its conditional transition distribution. These approaches use different state geometries and transition mechanisms, but share the broader motivation that uncertainty should remain represented during latent prediction and rollout.

\subsection{Latent dynamics and planning}

Probabilistic latent dynamics are central to model-based reinforcement learning. PlaNet combines deterministic and stochastic latent transitions and introduces latent overshooting as a multi-step variational training objective \citep{hafner2019planet}; Dreamer learns behaviors through trajectories imagined in a learned latent world model \citep{hafner2020dreamer}. These approaches demonstrate the importance of stochastic latent dynamics and long-horizon latent imagination for decision making, but typically learn their world models together with observation reconstruction or task-specific control objectives.

JEPA world models instead emphasize prediction directly in representation space without requiring pixel-level likelihoods \citep{lecun2022path,huang2026vjepa}. MDN-VJEPA retains this reconstruction-free predictive interface while making the transition explicitly multimodal. Its multi-horizon objective is also distinct from simply supervising longer latent rollouts: we compare the complete probability distribution predicted directly at a future horizon with the distribution obtained by recursively composing the learned one-step transition. The resulting constraint is therefore a distributional direct-versus-composed consistency condition rather than a conventional multi-step prediction loss.

This distinction becomes important for planning under branching futures. A stochastic latent world model is useful not only because it assigns uncertainty to a prediction, but because \textit{downstream decisions may depend on the structure of that uncertainty}. MDN-VJEPA therefore retains the complete weighted set of predictive alternatives, including their within-mode dispersion, so that expected-cost, chance-constrained, or tail-sensitive planners can reason over the predictive law rather than over its conditional mean alone.

\section{From Probabilistic JEPA to MDN-VJEPA}
\label{sec:method-base}
 
\subsection{Time-indexed probabilistic JEPA}

We develop MDN-VJEPA in a time-indexed sequential setting, including control, planning, video, and time-series prediction, where multimodal future prediction, trajectory-level coherence, and multi-horizon consistency arise naturally. This temporal formulation is used throughout the paper for concreteness. The underlying MDN construction, however, can be applied more generally: in a standard context--target JEPA, the same mixture predictor can model a multimodal conditional distribution over arbitrary target representations rather than future states.

Let $X_{\leq t}$ denote the observed history and $X_{t+1:t+H}$ a future target block. In the time-indexed VJEPA formulation considered here, the model uses a stochastic context encoder \citep{huang2026vjepa}
\begin{equation}
Z_t
\sim
q_\theta(Z_t\mid X_{\leq t}),
\label{eq:context-encoder}
\end{equation}
an EMA target encoder
\begin{equation}
Z_{t+h}^{\mathrm T}
\sim
q_{\bar\theta}(Z_{t+h}\mid X_{t+h}),
\qquad
h=1,\ldots,H,
\label{eq:target-encoder}
\end{equation}
and a conditional predictive latent model
\begin{equation}
p_\phi
\left(
Z_{t+h}\mid Z_t,\xi_{t:t+h-1}
\right),
\label{eq:general-predictor}
\end{equation}
where $\xi_{t:t+h-1}$ denotes the side information relevant to propagation from time $t$ to $t+h$, such as actions, elapsed times, goals, or physical constraints. The target-encoder parameters $\bar\theta$ are updated by exponential moving average (EMA), and target representations are treated with stop-gradient during predictive training. The online representation and predictor are optimized through the representation-space prediction objective \citep{huang2026vjepa}. Full PIB-VJEPA additionally regularizes the stochastic context representation so that it retains information predictive of the future while compressing unnecessary information from the observed history \citep{huang2026ibvjepa}.

\cref{fig:mdn-vjepa-architecture} illustrates the MDN-VJEPA predictive pipeline. A stochastic context encoder maps the observed history to the predictive state $Z_t$, which is combined with side information $\xi$ and passed to an MDN predictor. Rather than producing a single unimodal Gaussian prediction, the predictor outputs mixture weights, component means, and component covariances that together define a Gaussian-mixture distribution over the future representation. An EMA target encoder supplies the stop-gradient target representation $Z_{t+h}^{\mathrm T}$, and the predictive objective compares this target with the complete multimodal predictive law. In this way, MDN-VJEPA retains several alternative latent futures together with their relative probabilities and within-component uncertainties.

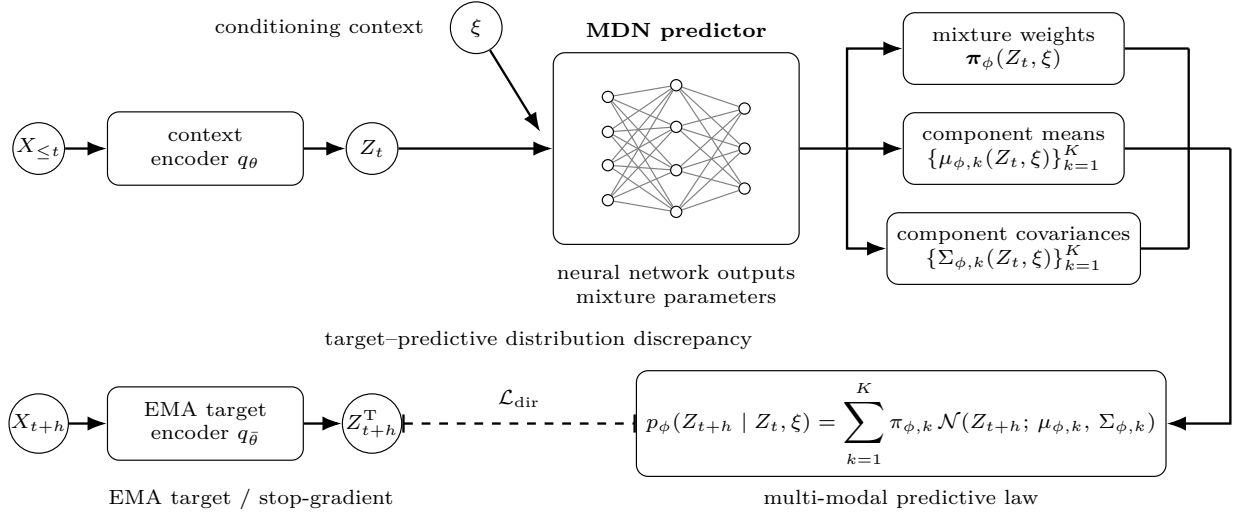
\begin{figure}[t]
\centering
\resizebox{0.98\linewidth}{!}{%
\begin{tikzpicture}[
    x=1cm,y=1cm,
    >=Latex,
    box/.style={
        draw,
        rounded corners,
        minimum width=2.35cm,
        minimum height=0.88cm,
        align=center,
        font=\scriptsize
    },
    bigbox/.style={
        draw,
        rounded corners,
        minimum width=2.95cm,
        minimum height=2.30cm,
        align=center,
        font=\scriptsize
    },
    smallbox/.style={
        draw,
        rounded corners,
        minimum width=2.65cm,
        minimum height=0.86cm,
        align=center,
        font=\scriptsize
    },
    distbox/.style={
        draw,
        rounded corners,
        minimum width=5.35cm,
        minimum height=1.12cm,
        align=center,
        font=\scriptsize
    },
    latent/.style={
        draw,
        circle,
        minimum size=0.62cm,
        inner sep=0pt,
        font=\scriptsize
    },
    arr/.style={->,thick},
    sgarrow/.style={->,thick,dashed},
    note/.style={font=\scriptsize,align=center}
]

\node[latent] (xc) at (0.0,1.75) {$X_{\le t}$};
\node[box]    (encC) at (2.0,1.75) {context\\encoder $q_\theta$};
\node[latent] (zt) at (4.0,1.75) {$Z_t$};

\draw[arr] (xc) -- (encC);
\draw[arr] (encC) -- (zt);

\node[latent] (xt) at (0.0,-1.55) {$X_{t+h}$};
\node[box]    (encT) at (2.0,-1.55) {EMA target\\encoder $q_{\bar\theta}$};
\node[latent] (yt) at (4.0,-1.55) {$Z_{t+h}^{\mathrm T}$};

\draw[arr] (xt) -- (encT);
\draw[arr] (encT) -- (yt);

\node[latent] (xi) at (5.25,3.20) {$\xi$};
\node[note, anchor=east] at (4.75,3.20) {conditioning context};

\node[bigbox] (mdn) at (7.65,1.75) {};
\node[note] at (7.65,3.15) {\textbf{MDN predictor}};

\begin{scope}[shift={(7.65,1.75)}]
    \foreach \i/\y in {1/0.62,2/0.20,3/-0.20,4/-0.62} {
        \node[draw,circle,inner sep=1.35pt] (l1\i) at (-0.82,\y) {};
    }

    \foreach \i/\y in {1/0.76,2/0.26,3/-0.26,4/-0.76} {
        \node[draw,circle,inner sep=1.35pt] (l2\i) at (0.0,\y) {};
    }

    \foreach \i/\y in {1/0.48,2/0.0,3/-0.48} {
        \node[draw,circle,inner sep=1.35pt] (l3\i) at (0.82,\y) {};
    }

    \foreach \i in {1,2,3,4} {
        \foreach \j in {1,2,3,4} {
            \draw[thin,gray] (l1\i) -- (l2\j);
        }
    }

    \foreach \i in {1,2,3,4} {
        \foreach \j in {1,2,3} {
            \draw[thin,gray] (l2\i) -- (l3\j);
        }
    }
\end{scope}

\node[note] at (7.65,0.10)
{neural network outputs\\mixture parameters};

\draw[arr] (zt) -- (mdn.west);
\draw[arr] (xi) -- ([xshift=-0.15cm,yshift=0.20cm]mdn.west);

\coordinate (fan1) at (9.70,1.75);
\draw[thick] (mdn.east) -- (fan1);

\node[smallbox] (pi) at (11.70,2.95)
{mixture weights\\
$\boldsymbol{\pi}_\phi(Z_t,\xi)$};

\node[smallbox] (mu) at (11.70,1.75)
{component means\\
$\{\mu_{\phi,k}(Z_t,\xi)\}_{k=1}^{K}$};

\node[smallbox] (sig) at (11.70,0.55)
{component covariances\\
$\{\Sigma_{\phi,k}(Z_t,\xi)\}_{k=1}^{K}$};

\draw[arr] (fan1) |- (pi.west);
\draw[arr] (fan1) -- (mu.west);
\draw[arr] (fan1) |- (sig.west);

\coordinate (collectTop) at (13.80,2.95);
\coordinate (collectMid) at (13.80,1.75);
\coordinate (collectBot) at (13.80,0.55);

\draw[thick] (pi.east) -- (collectTop);
\draw[thick] (mu.east) -- (collectMid);
\draw[thick] (sig.east) -- (collectBot);
\draw[thick] (collectTop) -- (collectBot);

\node[distbox] (mix) at (10.35,-1.55)
{$\displaystyle
p_\phi(Z_{t+h}\mid Z_t,\xi)
=
\sum_{k=1}^{K}
\pi_{\phi,k}\,
\N\!\left(
Z_{t+h};\,\mu_{\phi,k},\,\Sigma_{\phi,k}
\right)$};

\node[note] at (10.35,-2.45)
{multi-modal predictive law};

\coordinate (collectorOut) at (14.30,1.75);

\draw[thick] (13.80,1.75) -- (collectorOut);
\draw[arr] (collectorOut) |- (mix.east);

\draw[thick,dashed,|-|]
    (yt.east)
    --
    node[midway,above=2pt,font=\scriptsize]
    {$\cL_{\mathrm{dir}}$}
    (mix.west);

\node[note] at (6.00,-0.55)
{target--predictive distribution discrepancy};

\node[note] at (2.55,-2.45)
{EMA target / stop-gradient};

\end{tikzpicture}%
}
\caption{MDN-VJEPA pipeline. The trainable context encoder $q_\theta$ maps the observed history to the predictive state $Z_t$, while the EMA target encoder $q_{\bar\theta}$ produces the stop-gradient target representation $Z_{t+h}^{\mathrm T}$. Conditioned on $Z_t$ and side information $\xi$, the MDN predictor outputs mixture weights, component means, and component covariances, which together define a Gaussian-mixture predictive law over the future representation. The direct predictive loss compares this multimodal predictive distribution with the target representation.}
\label{fig:mdn-vjepa-architecture}
\end{figure}
 
\subsection{The mixture-density transition}

MDN-VJEPA specializes the probabilistic JEPA predictor to a Gaussian mixture. For the recursive dynamics, we define the \textit{one-step} conditional transition
\begin{equation}
p_\phi
\left(
Z_{t+1}\mid Z_t,\xi_t
\right)
=
\sum_{k=1}^{K}
\pi_{\phi,k}(Z_t,\xi_t)\,
\N\left(
Z_{t+1};
\mu_{\phi,k}(Z_t,\xi_t),
\Sigma_{\phi,k}(Z_t,\xi_t)
\right).
\label{eq:mdn-transition}
\end{equation}
Equivalently, writing the current state as $z$, the next state as $z'$, and suppressing parameter subscripts when unambiguous,
\begin{equation*}
p_\phi(z'\mid z,\xi)
=
\sum_{k=1}^{K}
\pi_k(z,\xi)\,
\N\left(
z';
\mu_k(z,\xi),
\Sigma_k(z,\xi)
\right).
\end{equation*}
In a non-temporal JEPA, the same construction applies by interpreting $z$ and $z'$ as context and target representations, respectively.

Each component contributes three complementary quantities: $\mu_k$ specifies the location of candidate future hypothesis $k$, $\Sigma_k$ captures uncertainty \emph{within} that hypothesis, and $\pi_k$ specifies the predictive probability assigned to it. Uncertainty arising from competing futures is therefore represented jointly through the mixture weights and separation among the component means, while each covariance describes local variability around its corresponding candidate future.

A neural predictor first computes an intermediate representation
\begin{equation*}
r_\phi
=
g_\phi(z,\xi)
\end{equation*}
and maps it to mixture logits $a_k$, means $\mu_k$, and covariance parameters. For the fixed-$K$ core model, the predictive weights are
\begin{equation}
\pi_k
=
\frac{\exp(a_k)}
{\sum_{j=1}^{K}\exp(a_j)},
\qquad
\pi_k\geq0,
\qquad
\sum_{k=1}^{K}\pi_k=1.
\label{eq:mixture-weights}
\end{equation}

The same Gaussian-mixture parameterization can also be used for a direct horizon-$h$ predictor,
\begin{equation*}
p_\phi^{\mathrm{dir},h}
\left(
Z_{t+h}\mid Z_t,\xi_{t:t+h-1}
\right),
\end{equation*}
with horizon-dependent component parameters. This direct multi-step distribution is conceptually distinct from the distribution obtained by recursively composing the one-step transition in \cref{eq:mdn-transition}; their relationship is developed later through the distributional path-consistency objective.

At inference, we retain the complete predictive mixture
\begin{equation}
\left\{
(\pi_k,\mu_k,\Sigma_k)
\right\}_{k=1}^{K},
\label{eq:inference-interface}
\end{equation}
rather than collapsing it to the mixture mean
$
\sum_k\pi_k\mu_k
$.
The component means specify candidate future locations, the covariances quantify local uncertainty around them, and the weights quantify their relative predictive support. When components specialize consistently, their means may additionally be interpreted as branch ``skeletons'', although such semantic specialization is an empirical property rather than an identifiability guarantee.

\begin{proposition}[Valid conditional transition kernel]
\label{prop:valid-kernel}
Let $\mathcal Z=\mathbb R^d$ be equipped with its Borel $\sigma$-algebra. Assume that $\pi_k(z,\xi)$, $\mu_k(z,\xi)$, and $\Sigma_k(z,\xi)$ are measurable in $(z,\xi)$, with
\begin{equation*}
\pi_k(z,\xi)\geq0,
\qquad
\sum_{k=1}^{K}\pi_k(z,\xi)=1,
\qquad
\Sigma_k(z,\xi)\succ0.
\end{equation*}
Then the MDN predictor in \cref{eq:mdn-transition} defines a valid conditional probability density over the next state and hence a conditional Markov transition kernel on $\mathcal Z$ given the side information $\xi$.
\end{proposition}

\begin{proof}
For fixed $(z,\xi)$, every Gaussian component is a non-negative normalized density. A convex combination with non-negative measurable weights summing to one is therefore also a normalized density. Consequently,
\begin{equation*}
K_\phi(z,\xi,A)
=
\int_A
p_\phi(z'\mid z,\xi)\,dz',
\qquad
A\in\mathcal B(\mathcal Z),
\end{equation*}
defines a probability kernel on $\mathcal Z$.
\end{proof}

Thus, the MDN head is not merely a collection of candidate regressors: it defines a normalized continuous transition law that can be sampled, integrated, recursively composed through time, and evaluated using probabilistic scoring and calibration criteria.
 
\subsection{Covariance parameterization}
\label{sec:covariance}

The covariance structure determines how uncertainty is represented within each predictive component. A full $d\times d$ covariance is expressive but becomes expensive as the latent dimension grows \citep{huang2025gma} and can be difficult to estimate reliably in high-dimensional JEPA spaces. We therefore use diagonal covariance as the default scalable realization and consider structured covariance as an optional extension when stronger within-component dependencies are required.

\paragraph{Diagonal covariance.}
Our core model assumes conditionally independent latent dimensions within each mixture component,
\begin{equation}
\Sigma_k
=
\operatorname{diag}
\left(
\sigma_{k,1}^2,\ldots,\sigma_{k,d}^2
\right),
\label{eq:diag-cov}
\end{equation}
with
\begin{equation}
\sigma_k
=
\sigma_{\min}
+
\softplus(s_k).
\label{eq:sigma-floor}
\end{equation}
This gives each component dimension-specific uncertainty while retaining $O(d)$ covariance complexity. The positive floor $\sigma_{\min}>0$ also prevents the likelihood from exploiting near-singular variances.

\paragraph{Low-rank plus diagonal extension.}
When within-component uncertainty is strongly correlated across latent dimensions, the covariance can instead be parameterized as
\begin{equation}
\Sigma_k
=
D_k
+
U_kU_k^\top,
\label{eq:lowrank-cov}
\end{equation}
where $D_k$ is positive diagonal and
$
U_k\in\mathbb R^{d\times r}
$
with
$
r\ll d
$.
The diagonal term preserves dimension-wise uncertainty while the low-rank factor captures dominant cross-dimensional correlations. This increases covariance parameter and storage complexity from $O(d)$ to $O(dr)$ rather than the $O(d^2)$ cost of a dense covariance matrix.

Thus, diagonal covariance provides our default scalable realization; low-rank-plus-diagonal structure offers an optional extension when correlated within-mode uncertainty matters; and fully dense covariance can be considered in sufficiently low-dimensional settings.

\subsection{Effective mixture complexity}
\label{sec:effective-k}

The core MDN-VJEPA model uses a fixed maximum number $K$ of mixture components, but this does not imply that all components must receive substantial predictive probability in every context. The softmax weights in \cref{eq:mixture-weights} already permit the effective complexity of the predictive mixture to vary with $(Z_t,\xi)$. For example, one context may produce
\begin{equation*}
\boldsymbol{\pi}(Z_t,\xi)
\approx
(0.98,0.01,0.01),
\end{equation*}
and therefore behave approximately unimodally, whereas another may produce
\begin{equation*}
\boldsymbol{\pi}(Z_{t'},\xi')
\approx
(0.48,0.46,0.06),
\end{equation*}
with predictive mass distributed over multiple components.

We quantify this context-dependent complexity using the \textit{entropy effective component count}
\begin{equation}
K_{\mathrm{eff}}(Z_t,\xi)
=
\exp
\left[
H\!\left(
\boldsymbol{\pi}(Z_t,\xi)
\right)
\right],
\label{eq:effective-k}
\end{equation}
where
\begin{equation*}
H(\boldsymbol{\pi})
=
-
\sum_{k=1}^{K}
\pi_k\log\pi_k.
\end{equation*}
Hence,
\begin{equation*}
1
\leq
K_{\mathrm{eff}}
\leq
K,
\end{equation*}
with $K_{\mathrm{eff}}\approx1$ when predictive mass is concentrated on one component and $K_{\mathrm{eff}}=K$ when the weights are uniform.

Importantly, $K_{\mathrm{eff}}$ measures the effective complexity or utilization of the predictive mixture rather than the number of \textit{semantically} distinct futures. Several Gaussian components may cooperate to approximate one complicated future branch, while a sufficiently broad component may cover several weakly separated outcomes. We therefore use $K_{\mathrm{eff}}$ primarily as a mixture-utilization diagnostic; semantic interpretation of individual components requires separate empirical evidence.

\subsection{Nested model hierarchy}

The proposed predictive family forms a natural hierarchy with Gaussian and deterministic JEPA. The Gaussian model is obtained directly by restricting the mixture family, while the relation to deterministic JEPA is an objective-level equivalence under a fixed-covariance probabilistic realization.

\begin{proposition}[Nested predictive-family hierarchy]
\label{prop:nesting}
Consider the MDN-VJEPA predictive distribution in \cref{eq:mdn-transition}. Then:
\begin{enumerate}[label=(\roman*)]
    \item setting $K=1$ recovers a Gaussian VJEPA predictor;
    \item for $K=1$, a deterministic target representation, and fixed isotropic predictive covariance $\Sigma_\phi=\sigma^2I$, Gaussian negative log-likelihood is equivalent, up to an additive constant and positive multiplicative factor, to the squared-error predictive objective used by deterministic JEPA; and
    \item MDN-VJEPA is itself a special case of the general VJEPA framework obtained by choosing a finite Gaussian-mixture predictive family.
\end{enumerate}
\end{proposition}

\begin{proof}
For $K=1$, the mixture weight satisfies $\pi_1=1$, and \cref{eq:mdn-transition} reduces to
\begin{equation*}
p_\phi(z'\mid z,\xi)
=
\N\left(
z';
\mu_\phi(z,\xi),
\Sigma_\phi(z,\xi)
\right),
\end{equation*}
which is a Gaussian VJEPA predictor.

Now let the target representation be deterministic,
$
Z^{\mathrm T}=\bar z
$,
and fix
$
\Sigma_\phi=\sigma^2I
$.
The Gaussian negative log-likelihood is \citep{Carl2005GP}
\begin{equation*}
-\log p_\phi(\bar z\mid z,\xi)
=
\frac{1}{2\sigma^2}
\left\|
\bar z-\mu_\phi(z,\xi)
\right\|_2^2
+
\frac{d}{2}\log(2\pi\sigma^2).
\end{equation*}
The second term is constant with respect to the predictor parameters and
$
1/(2\sigma^2)>0
$.
Minimizing this negative log-likelihood is therefore equivalent to minimizing squared prediction error, up to an additive constant and positive scale factor.

Finally, VJEPA permits a general conditional predictive distribution
$
p_\phi(Z_T\mid Z_C,\xi)
$.
Choosing that distribution to be a finite Gaussian mixture yields MDN-VJEPA.
\end{proof}

The resulting hierarchy is summarized as
\begin{equation}
\text{Deterministic JEPA}
\subset_{\mathrm{obj}}
\text{Gaussian VJEPA}
\subset
\text{MDN-VJEPA}
\subset
\text{general VJEPA},
\label{eq:nested-hierarchy}
\end{equation}
where $\subset_{\mathrm{obj}}$ denotes an objective-level relation rather than literal architectural inclusion: with a deterministic target representation and fixed isotropic covariance, Gaussian negative log-likelihood reduces to squared prediction error up to an additive constant and positive scale factor. The remaining inclusions correspond to increasingly expressive choices of predictive family. This hierarchy also motivates matched ablations: the effect of multimodal predictive capacity can be isolated by comparing $K=1$ with $K>1$ under the same surrounding probabilistic JEPA framework and, where relevant, matched predictor capacity.

\section{Coherent Multimodal Latent Dynamics}
\label{sec:coherent-dynamics}

The MDN construction above gives VJEPA a more expressive predictive family: instead of collapsing multiple plausible futures into a single point or unimodal belief, it can represent several alternatives together with their probabilities and within-mode uncertainties. However, \emph{multimodal expressiveness alone does not guarantee temporal coherence}. In sequential settings such as video, control, or trajectory forecasting, accurate multimodal predictions at individual horizons can still combine into an implausible joint future. We therefore distinguish two constructions: horizon-wise multimodal prediction and trajectory-level multimodal prediction.

\subsection{Level I: Why horizon-wise mixtures are insufficient}

Suppose a direct predictor models each future horizon with its own Gaussian mixture,
\begin{equation}
p_\phi
\left(
Z_{t+h}\mid Z_t,\xi_{t:t+H-1}
\right)
=
\sum_{k=1}^{K}
\pi_{\phi,k,h}
\N\left(
Z_{t+h};
\mu_{\phi,k,h},
\Sigma_{\phi,k,h}
\right),
\qquad
h=1,\ldots,H,
\label{eq:level1-marginal} \tag{Eq.\ref{eq:mdn-transition}b}
\end{equation}
where the dependence of
$\pi_{\phi,k,h}$,
$\mu_{\phi,k,h}$, and
$\Sigma_{\phi,k,h}$
on $(Z_t,\xi_{t:t+H-1})$ is suppressed for readability. If the horizons are coupled only through the common conditioning context, the resulting trajectory law factorizes as follows\footnote{We use the superscripts $\mathrm I$ and $\mathrm{II}$ to denote Levels~I and II of our predictive hierarchy, respectively: $p_\phi^{\mathrm I}$ denotes the horizon-wise construction introduced here, while $p_\phi^{\mathrm{II}}$ denotes the trajectory-level construction introduced in the next subsection.}:
\begin{equation}
\begin{aligned}
p_\phi^{\mathrm I}
\left(
Z_{t+1:t+H}\mid Z_t,\xi_{t:t+H-1}
\right)
&=
\prod_{h=1}^{H}
p_\phi
\left(
Z_{t+h}\mid Z_t,\xi_{t:t+H-1}
\right)
\\
&=
\prod_{h=1}^{H}
\left[
\sum_{k=1}^{K}
\pi_{\phi,k,h}
\N\left(
Z_{t+h};
\mu_{\phi,k,h},
\Sigma_{\phi,k,h}
\right)
\right].
\end{aligned}
\label{eq:level1-trajectory}
\end{equation}

Each horizon could therefore have a well-calibrated multimodal marginal while the sampled trajectory remains incoherent. Independently selecting a component at every horizon may, for example, produce ``left'' at one frame, ``right'' at the next, and ``left'' again thereafter. The same issue arises for groups of target tokens intended to describe one future scene: independent component choices can splice together mutually incompatible hypotheses.

The problem is structural. Accurate multimodal marginals do not uniquely determine an accurate multimodal trajectory law; what is missing is a coupling of the alternatives across horizons. In particular, the \textit{product-of-mixtures} form
$
\prod_{h=1}^{H}
\left[
\sum_{k=1}^{K}(\cdot)
\right]
$
allows the mixture identity to vary independently with $h$.

\subsection{Level II: Trajectory-level modes}

To preserve \textit{temporal coherence}, we instead introduce a single categorical variable for the directly predicted future block,
\begin{equation}
M
\sim
\Cat
\left(
\pi_{\phi,1},\ldots,\pi_{\phi,K}
\right),
\label{eq:traj-mode}
\end{equation}
where the trajectory-level weights depend on the current predictive context,
\begin{equation*}
\pi_{\phi,k}
=
\pi_{\phi,k}
\left(
Z_t,\xi_{t:t+H-1}
\right).
\end{equation*}
Conditioning the complete target block on the \textit{same} $M$ gives
\begin{equation}
\begin{aligned}
p_\phi^{\mathrm{II}}
\left(
Z_{t+1:t+H}\mid
Z_t,\xi_{t:t+H-1}
\right)
&=
\sum_{k=1}^{K}
\pi_{\phi,k}
\prod_{h=1}^{H}
\N\left(
Z_{t+h};
\mu_{\phi,k,h},
\Sigma_{\phi,k,h}
\right).
\end{aligned}
\label{eq:trajectory-mixture}
\end{equation}
Thus, Level II replaces the Level-I product of mixtures by a \textit{mixture of trajectory products}:
\begin{equation}
\underbrace{
\prod_{h=1}^{H}
\left[
\sum_{k=1}^{K}(\cdot)
\right]
}_{\text{Level I: horizon-wise mixture identities}}
\qquad\longrightarrow\qquad
\underbrace{
\sum_{k=1}^{K}
\left[
\prod_{h=1}^{H}(\cdot)
\right]
}_{\text{Level II: shared trajectory-level identity}}.
\label{eq:level1-level2-contrast}
\end{equation}

The conditional independence across horizons within a mode in \cref{eq:trajectory-mixture} is a deliberately simple realization. When stronger within-trajectory dependencies are required, the distribution conditioned on $M$ can instead use a joint covariance over the concatenated trajectory representation, for example through a low-rank-plus-diagonal parameterization analogous to \cref{eq:lowrank-cov}.

The central architectural constraint is
\begin{equation}
M=k
\Longrightarrow
\left(
\mu_{k,1},
\mu_{k,2},
\ldots,
\mu_{k,H}
\right),
\label{eq:same-k-trajectory}
\end{equation}
so that \textit{one mixture component parameterizes one complete candidate trajectory over the prediction horizon} rather than selecting unrelated GMM component identities at individual future frames or tokens. This does not require component $k$ to possess a universal semantic interpretation across all contexts.

Importantly, \textit{trajectory-level coupling need not sacrifice horizon-wise multimodality}. Marginalizing \cref{eq:trajectory-mixture} at any horizon $h$ gives\footnote{A detailed derivation, both by direct marginalization and through the latent trajectory mode $M$, is given in \cref{app:level2-marginalization}.}
\begin{equation}
p_\phi^{\mathrm{II}}
\left(
Z_{t+h}\mid
Z_t,\xi_{t:t+H-1}
\right)
=
\sum_{k=1}^{K}
\pi_{\phi,k}
\N\left(
Z_{t+h};
\mu_{\phi,k,h},
\Sigma_{\phi,k,h}
\right),
\label{eq:level2-marginal}
\end{equation}
which remains a Gaussian mixture. The distinction between Levels~I and II therefore lies primarily in the \emph{joint coupling across horizons}, not in whether each individual horizon can represent multiple futures.

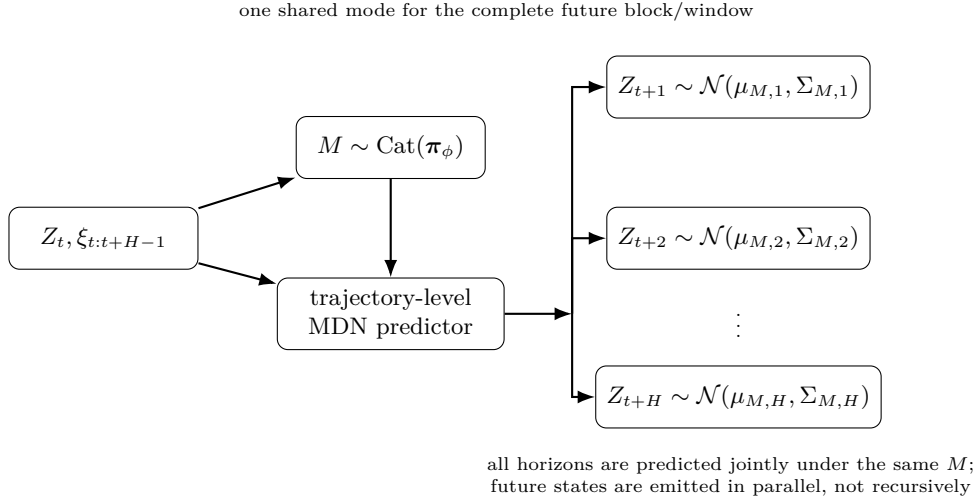
\begin{figure}[t]
\centering
\begin{tikzpicture}[
    x=1cm,y=1cm,
    >=Latex,
    box/.style={
        draw,
        rounded corners,
        align=center,
        minimum height=0.82cm,
        inner sep=4pt,
        font=\small
    },
    state/.style={
        draw,
        rounded corners,
        align=center,
        minimum height=0.82cm,
        inner sep=4pt,
        font=\small
    },
    arr/.style={->,thick},
    note/.style={font=\scriptsize,align=center}
]

\node[box, minimum width=2.5cm] (ctx) at (0,0)
{$Z_t, \xi_{t:t+H-1}$};

\node[box, minimum width=2.5cm] (mode) at (3.8,1.2)
{$M\sim\Cat(\boldsymbol{\pi}_\phi)$};

\node[box, minimum width=3.0cm] (head) at (3.8,-1.0)
{trajectory-level\\MDN predictor};

\node[state, minimum width=3.1cm] (z1) at (8.4,2.0)
{$Z_{t+1}\sim \N(\mu_{M,1},\Sigma_{M,1})$};

\node[state, minimum width=3.1cm] (z2) at (8.4,0.0)
{$Z_{t+2}\sim \N(\mu_{M,2},\Sigma_{M,2})$};

\node[note] (vdots) at (8.4,-1.05) {$\vdots$};

\node[state, minimum width=3.1cm] (zH) at (8.4,-2.1)
{$Z_{t+H}\sim \N(\mu_{M,H},\Sigma_{M,H})$};

\coordinate (fan) at (6.2,-1.0);

\draw[arr] (ctx) -- (mode);
\draw[arr] (ctx) -- (head);
\draw[arr] (mode) -- (head);

\draw[arr] (head.east) -- (fan);

\draw[arr] (fan) |- (z1.west);
\draw[arr] (fan) |- (z2.west);
\draw[arr] (fan) |- (zH.west);

\node[note] at (5.2,3.0)
{one shared mode for the complete future block/window};

\node[note] at (8.3,-3.15)
{all horizons are predicted jointly under the same $M$;\\
future states are emitted in parallel, not recursively};

\end{tikzpicture}
\caption{Trajectory-level mixture semantics. The current predictive context defines a distribution over a shared trajectory mode $M$. The context together with that same sampled or marginalized mode parameterizes all horizon-specific Gaussian predictions in the directly predicted future block. The future states are emitted in parallel rather than recursively, distinguishing Level-II trajectory coherence from the recursive transition composition introduced at Level III.}
\label{fig:trajectory-modes}
\end{figure}

For video representations, all target tokens assigned to the same predicted future block by the JEPA target-construction scheme can likewise share \footnote{The block assignment is determined externally by the target construction, for example through a common future temporal window or spatiotemporal target region. The latent mode $M$ does not determine token membership; conditioned on a defined target block, one trajectory-level mode is shared across its tokens and may be sampled from, or marginalized under, $\Cat(\boldsymbol{\pi}_\phi)$.} $M$. The requirement is therefore \emph{local trajectory coherence}: a component may specialize to a recognizable future regime within a predictive context without being required to retain the same semantic meaning globally. As with finite mixture models generally, component labels remain exchangeable and may permute, split, or share semantic regimes.

\subsection{Level III: Distributional consistency across prediction horizons}
\label{sec:level3-consistency}

\textit{Temporal coherence} within a directly predicted trajectory does not by itself guarantee that the learned dynamics are \emph{consistent} across prediction horizons. A model may produce a plausible coherent future when predicting $h$ steps ahead directly, yet produce a different predictive belief when the learned one-step transition is recursively propagated to the same horizon. Level III therefore introduces a second structural requirement: \textit{the direct and recursively composed predictive distributions should approximately describe the same future law}.

Let
\begin{equation}
p_\phi^{\mathrm{dir},h}
\left(
Z_{t+h}
\mid
Z_t,\xi_{t:t+h-1}
\right)
\label{eq:direct-horizon-law}
\end{equation}
denote the horizon-$h$ marginal obtained directly from the trajectory-level predictor. In contrast, recursively composing the one-step MDN transition in \cref{eq:mdn-transition} gives
\begin{equation}
p_\phi^{\mathrm{comp},h}
\left(
z_{t+h}\mid z_t,\xi_{t:t+h-1}
\right)
=
\int
\prod_{j=0}^{h-1}
p_\phi^{(1)}
\left(
z_{t+j+1}
\mid
z_{t+j},\xi_{t+j}
\right)
\,dz_{t+1:t+h-1}.
\label{eq:composed-density}
\end{equation}
For a correctly specified Markov transition model, these two computational paths target the same $h$-step conditional distribution. We therefore seek the approximate distributional relation
\begin{equation}
p_\phi^{\mathrm{dir},h}
\approx
p_\phi^{\mathrm{comp},h},
\qquad
h>1,
\label{eq:ck-approx}
\end{equation}
which is the learned analogue of the Chapman--Kolmogorov consistency condition \citep{huang2026HMMVJEPA}.

\begin{figure}[t]
\centering
\begin{tikzpicture}[
    x=1cm,y=1cm,
    >=Latex,
    box/.style={
        draw,
        rounded corners,
        align=center,
        minimum height=0.82cm,
        inner sep=4pt,
        font=\small
    },
    dist/.style={
        draw,
        rounded corners,
        align=center,
        minimum height=0.82cm,
        inner sep=4pt,
        font=\small
    },
    arr/.style={->,thick},
    darr/.style={->,thick,dashed},
    note/.style={font=\scriptsize,align=center}
]

\node[box, minimum width=1.5cm] (zt) at (0,0) {$Z_t$};

\node[box, minimum width=3.0cm] (direct) at (3.4,1.6)
{direct horizon-$h$\\trajectory predictor};

\node[dist, minimum width=3.2cm] (pdir) at (8.0,1.6)
{$p_\phi^{\mathrm{dir},h}(Z_{t+h}\mid Z_t,\xi)$};

\node[box, minimum width=2.2cm] (step) at (3.0,-1.6)
{one-step\\MDN};

\node[box, minimum width=1.8cm] (mid) at (5.4,-1.6)
{$Z_{t+1},\ldots$};

\node[box, minimum width=2.4cm] (comp) at (7.9,-1.6)
{recursive\\composition};

\node[dist, minimum width=3.4cm] (pcomp) at (12.2,-1.6)
{$p_\phi^{\mathrm{comp},h}(Z_{t+h}\mid Z_t,\xi)$};

\draw[arr] (zt) -- (direct);
\draw[arr] (direct) -- (pdir);

\draw[arr] (zt) -- (step);
\draw[arr] (step) -- (mid);
\draw[arr] (mid) -- (comp);
\draw[arr] (comp) -- (pcomp);

\draw[darr] (pdir.south east) to[bend left=12] node[right,note]
{distributional\\agreement} (pcomp.north east);
\draw[darr] (pcomp.north west) to[bend left=12] (pdir.south west);

\node[note] at (5.6,2.45)
{one direct prediction to horizon $h$};

\node[note] at (7.6,-2.65)
{$h$-step law obtained by recursively applying the learned one-step transition};

\end{tikzpicture}
\caption{Level-III distributional path consistency. The same horizon-$h$ predictive law can be reached by two computational paths: a direct prediction from $Z_t$ (upper) or recursive composition of the learned one-step MDN transition (lower). Level III regularizes the resulting distributions towards agreement. The comparison is distributional and therefore does not require component labels to align between the two paths.}
\label{fig:level3-path-consistency}
\end{figure}
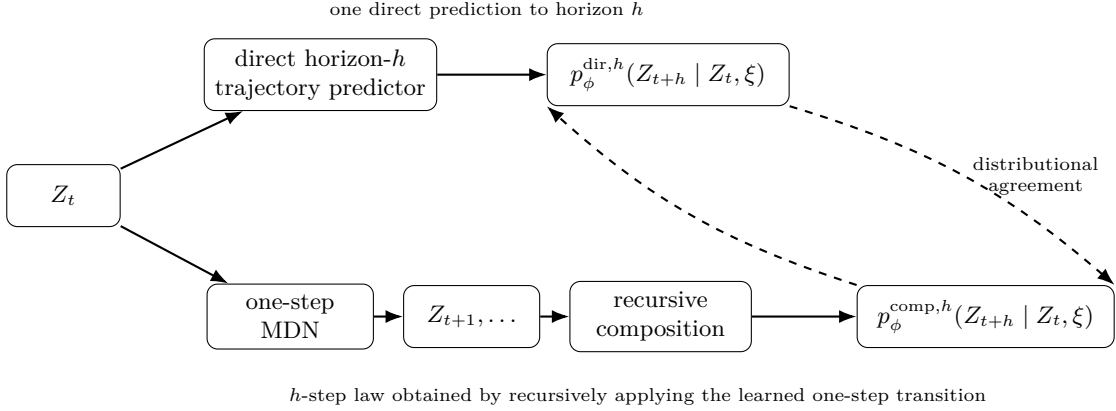

The distinction from Level II is important. Level II concerns \emph{coherence within one directly predicted future block}: the same trajectory-level mode couples the horizons of that prediction. Level III instead concerns \emph{agreement between different prediction paths}: a directly predicted future belief should be compatible with the belief obtained by recursively propagating the learned dynamics.

For nonlinear MDN transitions, recursive composition is generally not closed within a finite $K$-component GMM family. The mixture weights, means, and covariances at later steps depend on random intermediate states, so the composed law cannot in general be obtained simply by manipulating the parameters of the original $K$ components. We therefore represent $p_\phi^{\mathrm{comp},h}$ by ancestral samples and compare it with samples from $p_\phi^{\mathrm{dir},h}$ using a permutation-invariant distributional discrepancy.

Let $\mathcal E(P,Q)$ denote the \textit{energy distance} between distributions $P$ and $Q$. We define the multi-horizon consistency objective
\begin{equation}
\cL_{\mathrm{CK}}
=
\sum_{\substack{h\in\mathcal H\\h>1}}
w_h\,
\mathcal E
\left(
p_\phi^{\mathrm{dir},h},
p_\phi^{\mathrm{comp},h}
\right),
\label{eq:ck-loss}
\end{equation}
where $\mathcal H$ is the set of horizons at which consistency is evaluated and $w_h\geq0$ controls their relative importance. Because energy distance acts on complete predictive distributions rather than labeled mixture components, it is invariant to permutations of the component indices and does not require component-wise alignment. Other sample-based discrepancies, such as maximum mean discrepancy (MMD) \citep{Gretton2012MMD}, can be used analogously.

\begin{proposition}[Chapman--Kolmogorov target]
\label{prop:ck}
Suppose that the one-step MDN transition in \cref{eq:mdn-transition} equals the true time-indexed Markov transition kernel and that the direct horizon predictor represents the corresponding true $h$-step transition law. Then
\begin{equation*}
p_\phi^{\mathrm{dir},h}
=
p_\phi^{\mathrm{comp},h}
\end{equation*}
for every $h$, and therefore any metric discrepancy between the two predictive distributions is zero.
\end{proposition}

\begin{proof}
For a Markov process, the true $h$-step transition law is obtained by recursively composing the true one-step kernel and integrating over the intermediate states. Hence, if the learned one-step transition equals the true kernel,
\[
p_\phi^{\mathrm{comp},h}
=
p^\star_h.
\]
If the separately parameterized direct predictor also represents the same true horizon-$h$ law, then
\[
p_\phi^{\mathrm{dir},h}
=
p_\phi^{\mathrm{comp},h}
=
p^\star_h.
\]
Any metric discrepancy between the two distributions is therefore zero.
\end{proof}

The proposition describes the population target of $\cL_{\mathrm{CK}}$: the purpose of the consistency term is not to force two arbitrary predictions to agree, but to encourage two computational paths to \textit{represent the same underlying multi-step dynamics}. Agreement alone is not sufficient for predictive correctness, so both the direct and composed distributions must also remain faithful to the observed target law.

\subsection{Latent transition-mode interpretation}
\label{sec:hmm-simple}

The trajectory-level mode above describes the structure of a \emph{direct multi-horizon prediction}. The one-step MDN kernel from \cref{eq:mdn-transition} has a related but distinct latent-variable interpretation. For a single transition, introduce a local mixture indicator
\begin{equation}
M_{t+1}
\sim
\Cat\left(
\boldsymbol{\pi}_\phi(Z_t,\xi_t)
\right),
\qquad
Z_{t+1}
\mid
M_{t+1}=k,Z_t,\xi_t
\sim
\N\left(
\mu_{\phi,k}(Z_t,\xi_t),
\Sigma_{\phi,k}(Z_t,\xi_t)
\right).
\label{eq:hmm-mode}
\end{equation}
Here, $Z_t$ is the current continuous predictive JEPA state, $M_{t+1}$ is a discrete latent \emph{transition mode}, and $Z_{t+1}$ is the resulting next predictive state.

The corresponding conditional structure is more accurately written as
\begin{equation}
(Z_t,\xi_t)
\longrightarrow
M_{t+1},
\qquad
(Z_t,\xi_t,M_{t+1})
\longrightarrow
Z_{t+1},
\label{eq:transition-mode-graph}
\end{equation}
because the mode-specific mean and covariance of $Z_{t+1}$ depend directly on $(Z_t,\xi_t)$ as well as on the selected mixture component.

The trajectory-level variable $M$ in \cref{eq:traj-mode} and the local one-step indicator $M_{t+1}$ in \cref{eq:hmm-mode} therefore play related but different roles. The former indexes one complete \emph{directly predicted} future trajectory, whereas the latter indexes one stochastic transition of the recursive Markov kernel. In particular, recursive rollout of the one-step kernel does not require the same discrete label to persist across time. MDN-VJEPA does not introduce a separate dynamics model
\begin{equation*}
p(M_{t+1}\mid M_t),
\end{equation*}
nor does it assume that the component labels form a persistent discrete Markov chain.

This gives a limited connection to hidden-state models: the stochastic transition contains an unobserved categorical variable that selects among mode-specific continuous transition laws. The analogy should not be taken to imply a conventional HMM with persistent discrete states; the core dynamical state remains the continuous JEPA representation.

\paragraph{Semantic interpretation.}
The probabilities $\pi_k(Z_t,\xi)$ quantify predictive support for alternative latent transition hypotheses. In favorable problems such as the robot navigation setting, components may specialize into recognizable regimes such as ``turn left'', ``turn right'', ``stop'', or ``continue''. Such specialization is an empirical property rather than an identifiability guarantee: components may permute, split, merge, or share a semantic regime. We therefore assess semantic specialization through mode-conditioned probes, assignment stability, or known event labels rather than assigning meaning to a numerical component index a priori.
 
\subsection{Decomposing predictive uncertainty}

The latent-mode representation separates two qualitatively different contributions to conditional predictive uncertainty. After conditioning on a trajectory mode, each covariance $\Sigma_{k,h}$ describes local variability around that candidate future. Before the mode is known, additional uncertainty arises because the candidate future may themselves disagree.
For a fixed horizon $h$, define the Level-II conditional predictive mean
\begin{equation}
\bar\mu_h
=
\sum_{k=1}^{K}
\pi_k\mu_{k,h},
\label{eq:mixture-mean-h}
\end{equation}
where the dependence of the mixture parameters on the predictive context is again suppressed for readability. Applying the conditional law of total covariance with respect to $M$ yields the following decomposition.

\begin{proposition}[Within-mode and between-mode uncertainty]
\label{prop:uncertainty-decomp}
For the horizon-$h$ marginal of the trajectory mixture in \cref{eq:trajectory-mixture},
\begin{equation}
\operatorname{Cov}
\left(
Z_{t+h}
\mid
Z_t,\xi_{t:t+H-1}
\right)
=
\underbrace{
\sum_{k=1}^{K}
\pi_k\Sigma_{k,h}
}_{\text{within-mode uncertainty}}
+
\underbrace{
\sum_{k=1}^{K}
\pi_k
\left(
\mu_{k,h}-\bar\mu_h
\right)
\left(
\mu_{k,h}-\bar\mu_h
\right)^\top
}_{\text{between-mode / branching uncertainty}}.
\label{eq:uncertainty-decomp}
\end{equation}
\end{proposition}

\begin{proof}
Conditioned on $M=k$, the horizon-$h$ predictive mean and covariance are $\mu_{k,h}$ and $\Sigma_{k,h}$. The conditional law of total covariance gives
\begin{align*}
\operatorname{Cov}(Z_{t+h}\mid Z_t,\xi)
&=
\E\!\left[
\operatorname{Cov}(Z_{t+h}\mid M,Z_t,\xi)
\mid Z_t,\xi
\right]
\\
&\quad+
\operatorname{Cov}\!\left(
\E[Z_{t+h}\mid M,Z_t,\xi]
\mid Z_t,\xi
\right),
\end{align*}
which evaluates respectively to
$
\sum_k\pi_k\Sigma_{k,h}
$
and
$
\sum_k
\pi_k
(\mu_{k,h}-\bar\mu_h)
(\mu_{k,h}-\bar\mu_h)^\top
$.
\end{proof}

The first term in Eq.\ref{eq:uncertainty-decomp} is the expected local uncertainty remaining after conditioning on a future mode; the second measures uncertainty induced by separation among the competing future hypotheses. A single Gaussian may match the aggregate first two moments of such a mixture, but it does not expose whether its dispersion arises from local variability around one future or from several individually concentrated but separated alternatives. MDN-VJEPA represents these contributions explicitly.

This decomposition concerns \emph{conditional predictive uncertainty} under the learned transition model. It should not by itself be interpreted as a decomposition into aleatoric and epistemic uncertainty: uncertainty in the model parameters or representation function would require an additional Bayesian, ensemble, or related treatment.

\section{A Compositional Training Objective}
\label{sec:objective}

The preceding section specifies \emph{what} the three levels of MDN-VJEPA should represent: Level~I preserves multimodal horizon-wise predictions, Level~II couples those predictions into \emph{coherent} trajectory-level alternatives, and Level~III further seeks \emph{consistent} agreement between direct and recursively composed predictive laws. We now describe \emph{how} these models are trained. The level-wise predictive likelihood determines the structure of the learned future distribution; predictive-information regularization and mixture safeguards act across levels; and Level~III additionally introduces a recursively trained one-step transition together with distributional path consistency.

\subsection{Level-wise predictive likelihoods}

Let
\[
Y_{t,H}^{\mathrm T}
=
\left(
Z_{t+1}^{\mathrm T},
\ldots,
Z_{t+H}^{\mathrm T}
\right)
\]
denote a stochastic target trajectory obtained from the EMA target encoder, with the target representations treated as stop-gradient for predictive supervision.

\paragraph{Level I: horizon-wise multimodal likelihood.}
Level~I uses the \emph{product-of-mixtures} construction in \cref{eq:level1-trajectory}. Because the future horizons are conditionally factorized given the current predictive context, its direct predictive \emph{negative log-likelihood} (NLL) is
\begin{equation}
\cL_{\mathrm{dir}}^{\mathrm I}
=
\E
\left[
-
\sum_{h=1}^{H}
\log
p_\phi
\left(
Z_{t+h}^{\mathrm T}
\mid
Z_t,\xi_{t:t+H-1}
\right)
\right],
\label{eq:level1-nll}
\end{equation}
where each horizon-wise predictive marginal $p_\phi$ is the Gaussian mixture in \ref{eq:level1-marginal}. Thus, each horizon can learn a multimodal predictive distribution, but the likelihood does not require the same mixture identity to explain the complete future trajectory.

\paragraph{Level II: coherent trajectory likelihood.}
Level~II instead trains the shared-mode trajectory distribution in \cref{eq:trajectory-mixture}. We define
\begin{equation}
\cL_{\mathrm{dir}}^{\mathrm{II}}
\equiv
\cL_{\mathrm{MDN}}
=
\E_{Z_t\sim q_\theta,\,
Y_{t,H}^{\mathrm T}\sim q_{\bar\theta}}
\left[
-\log
p_\phi^{\mathrm{II}}
\left(
Y_{t,H}^{\mathrm T}
\mid
Z_t,\xi_{t:t+H-1}
\right)
\right].
\label{eq:mdn-nll}
\end{equation}
For a target trajectory $Y_{t,H}^{\mathrm T}$, the corresponding log likelihood can be written as\footnote{Substituting the Level-II trajectory mixture from \cref{eq:trajectory-mixture} gives
\[
p_\phi^{\mathrm{II}}(Y_{t,H}^{\mathrm T}\mid Z_t,\xi)
=
\sum_{k=1}^{K}
\pi_{\phi,k}
\prod_{h=1}^{H}
\N\!\left(
Z_{t+h}^{\mathrm T};
\mu_{\phi,k,h},
\Sigma_{\phi,k,h}
\right).
\]
Taking the logarithm yields
\[
\log p_\phi^{\mathrm{II}}(Y_{t,H}^{\mathrm T}\mid Z_t,\xi)
=
\log
\sum_{k=1}^{K}
\pi_{\phi,k}
\prod_{h=1}^{H}
\N\!\left(
Z_{t+h}^{\mathrm T};
\mu_{\phi,k,h},
\Sigma_{\phi,k,h}
\right).
\]
The logarithm cannot be distributed over the outer sum. Instead, for each component define
\[
a_k
=
\log\pi_{\phi,k}
+
\sum_{h=1}^{H}
\log
\N\!\left(
Z_{t+h}^{\mathrm T};
\mu_{\phi,k,h},
\Sigma_{\phi,k,h}
\right),
\]
where the product over horizons becomes a sum because
$
\log\prod_h p_h=\sum_h\log p_h
$.
Then
\[
\log p_\phi^{\mathrm{II}}(Y_{t,H}^{\mathrm T}\mid Z_t,\xi)
=
\log\sum_{k=1}^{K}\exp(a_k)
=
\operatorname{logsumexp}_{k}(a_k).
\]
Thus, the sum over horizons occurs \emph{within} each trajectory component, whereas the outer mixture over components is retained through the log-sum-exp operation.}
\begin{equation}
\begin{aligned}
\log
p_\phi^{\mathrm{II}}
\left(
Y_{t,H}^{\mathrm T}
\mid
Z_t,\xi_{t:t+H-1}
\right)
=
\operatorname{logsumexp}_{k}
\Bigg[
\log \pi_{\phi,k}
+
\sum_{h=1}^{H}
\log
\N\left(
Z_{t+h}^{\mathrm T};
\mu_{\phi,k,h},
\Sigma_{\phi,k,h}
\right)
\Bigg].
\end{aligned}
\label{eq:level_two_logsumexp}
\end{equation}
The same component index $k$ appears throughout the sum over horizons. The likelihood therefore asks \emph{one mixture component to explain the complete target trajectory} rather than allowing an independently selected component at every future horizon.

For numerical stability, when evaluating \cref{eq:level_two_logsumexp}, we use the standard shifted log-sum-exp evaluation\footnote{For scalars $x_1,\ldots,x_K$,
\[
\operatorname{logsumexp}(x_1,\ldots,x_K)
=
m+
\log
\sum_{k=1}^{K}
\exp(x_k-m),
\qquad
m=\max_k x_k.
\]
The largest exponent is therefore zero and all remaining exponentials are at most one, reducing numerical underflow or overflow when component likelihoods differ substantially in scale \citep{huang2025gma}.}.

\paragraph{One-step transition likelihood.}
Level~III requires a recursive predictive path in addition to the direct trajectory predictor. The one-step MDN kernel in \cref{eq:mdn-transition} must therefore itself be fitted to observed latent transitions rather than being trained only indirectly through the consistency objective. We use
\begin{equation}
\cL_{\mathrm{step}}
=
\E
\left[
-\log
p_\phi^{(1)}
\left(
Z_{t+1}^{\mathrm T}
\mid
Z_t,\xi_t
\right)
\right],
\label{eq:step-nll}
\end{equation}
where $p_\phi^{(1)}$ denotes the learned \emph{one-step} transition kernel.\footnote{The superscript $(1)$ denotes a one-step transition and should not be confused with the Roman superscript $\mathrm I$ used for Level~I of the predictive hierarchy. Thus, $p_\phi^{(1)}$ is the transition kernel recursively composed in Level~III, whereas $p_\phi^{\mathrm I}$ denotes the horizon-wise product-of-mixtures construction.}
When the available training trajectory contains several teacher-forced one-step transitions, \cref{eq:step-nll} is averaged over the corresponding transition pairs.

The direct Level-III predictor retains the Level-II trajectory semantics,
\[
\cL_{\mathrm{dir}}^{\mathrm{III}}
\equiv
\cL_{\mathrm{dir}}^{\mathrm{II}},
\]
while the separately supervised one-step kernel supplies the recursive path required by \cref{eq:composed-density}. A one-step transition may also be trained for a Level-II model when recursive rollout is required for evaluation; when no recursive transition is instantiated, its loss coefficient is set to zero. Here and below, $\phi$ collectively denotes the predictive parameters, including the direct and one-step heads when both are present.

\subsection{Predictive-information regularization}

Predictive likelihood specifies how future latent states should be modeled, but does not by itself determine what information the learned representation should retain from the observations. We therefore retain the probabilistic representation regularization of VJEPA \citep{huang2026vjepa} together with the explicit context compression of PIB-VJEPA \citep{huang2026ibvjepa}. Because the predictive and representation objectives use the encoder in different roles, we first make the corresponding gradient paths explicit.

\paragraph{Three encoder roles.}
We distinguish three encoder evaluations:
\[
\underbrace{
q_\theta
\left(
Z_t\mid X_{\leq t}
\right)
}_{\text{online context path}},
\qquad
\underbrace{
q_\theta^{\mathrm T}
\left(
Z_{t+h}\mid X_{t+h}
\right)
}_{\text{online target-view path}},
\qquad
\underbrace{
q_{\bar\theta}
\left(
Z_{t+h}\mid X_{t+h}
\right)
}_{\text{EMA target path}}.
\]
These do not correspond to three independently parameterized encoders. The first two use the same trainable parameters $\theta$ but evaluate the encoder on different inputs and for different purposes. The superscript $\mathrm T$ in $q_\theta^{\mathrm T}$ denotes the \emph{target-view role}, not a separate parameter set. The third distribution uses the EMA parameters $\bar\theta$.

The online context path
\[
q_\theta
\left(
Z_t\mid X_{\leq t}
\right)
\]
maps the observed history to the current predictive state $Z_t$ supplied to the MDN predictor. It is trainable and therefore receives gradients from predictive learning and context-side representation regularization.

The EMA target path
\[
q_{\bar\theta}
\left(
Z_{t+h}\mid X_{t+h}
\right)
\]
maps the future target observation to the stochastic target representation
$
Z_{t+h}^{\mathrm T}
$
used for predictive supervision. These target representations are treated with stop-gradient, and $\bar\theta$ is not directly optimized by gradient descent. Instead, it tracks the online parameters through an \emph{exponential-moving-average} (EMA) update such as
\[
\bar\theta
\leftarrow
\tau\bar\theta
+
(1-\tau)\theta,
\qquad
0<\tau<1.
\]

Finally, the online target-view path
\[
q_\theta^{\mathrm T}
\left(
Z_{t+h}\mid X_{t+h}
\right)
\]
evaluates the trainable encoder on the same target view. Its purpose is to provide a gradient path for target-side representation regularization. Thus, $q_\theta^{\mathrm T}$ receives the corresponding regularization gradient, whereas its EMA counterpart $q_{\bar\theta}$ supplies the stop-gradient target used by the predictive likelihood and inherits changes in $\theta$ only through the EMA update.

Schematically,
\[
X_{\leq t}
\longrightarrow
q_\theta
\longrightarrow
Z_t
\longrightarrow
\text{prediction},
\]
while the target observation follows two parallel evaluations,
\[
X_{t+h}
\longrightarrow
\begin{cases}
q_\theta^{\mathrm T}
\longrightarrow
\text{target-side representation regularization},
\\[1mm]
q_{\bar\theta}
\longrightarrow
Z_{t+h}^{\mathrm T}
\longrightarrow
\text{stop-gradient predictive target}.
\end{cases}
\]

\paragraph{Target-side representation regularization.}
Let $r_T$ denote a reference distribution for the stochastic target representation. We regularize the trainable target-view path through
\begin{equation}
\cL_T
=
\E
\left[
\sum_{h=1}^{H}
\KL
\left(
q_\theta^{\mathrm T}
\left(
Z_{t+h}\mid X_{t+h}
\right)
\,\middle\|\,
r_T(Z)
\right)
\right].
\label{eq:target-kl}
\end{equation}
The role of $\cL_T$ is to \emph{regularize the distributional geometry of the future representation while preserving a valid gradient path through the online encoder}. The EMA target encoder is not optimized by this KL term directly; instead, the resulting updates to $\theta$ are transferred gradually to $\bar\theta$ through EMA.

\paragraph{Context-side information compression.}
Let $r_C$ denote a reference distribution for the current predictive state. We define
\begin{equation}
\cL_C
=
\E
\left[
\KL
\left(
q_\theta
\left(
Z_t\mid X_{\leq t}
\right)
\,\middle\|\,
r_C(Z)
\right)
\right].
\label{eq:context-kl}
\end{equation}
Unlike $\cL_T$, which regularizes the target-side stochastic representation, $\cL_C$ acts directly on the state from which future prediction is performed. Its purpose is to discourage $Z_t$ from retaining unnecessary information about the observed history while the predictive likelihood simultaneously requires $Z_t$ to preserve information useful for predicting the future.

Under the standard variational information-bottleneck decomposition,\footnote{This is the temporal state KL-to-prior decomposition derived for full PIB-VJEPA in \citet[Eq.~(42)]{huang2026ibvjepa}; the corresponding mutual-information upper bound is stated explicitly in Eq.~(29) therein. Their state prior $p_0^S(Z_t)$ corresponds to the reference distribution $r_C(Z)$ used here. Defining the aggregated state posterior
$
q_\theta(Z_t)
=
\int
p_{\mathrm{data}}(X_{\leq t})
q_\theta(Z_t\mid X_{\leq t})
\,dX_{\leq t},
$
the expected KL decomposes into the state information $\MI(X_{\leq t};Z_t)$ plus the nonnegative aggregated-posterior mismatch $\KL(q_\theta(Z_t)\|r_C(Z))$.}
\begin{align*}
\E_{X_{\leq t}}
\KL
\left(
q_\theta(Z_t\mid X_{\leq t})
\,\middle\|\,
r_C(Z)
\right)
&=
\MI(X_{\leq t};Z_t)
+
\KL
\left(
q_\theta(Z_t)
\,\middle\|\,
r_C(Z)
\right)
\\
&\geq
\MI(X_{\leq t};Z_t),
\end{align*}
and therefore
\[
\MI(X_{\leq t};Z_t)
\leq
\E_{X_{\leq t}}
\KL
\left(
q_\theta(Z_t\mid X_{\leq t})
\,\middle\|\,
r_C(Z)
\right).
\]
Minimizing $\cL_C$ therefore provides a \emph{variational upper bound} on the information carried from the observed history into the predictive state \citep{alemi2017vib,huang2026ibvjepa}. Combined with predictive supervision, this creates the desired information-bottleneck tradeoff: compress non-predictive history while retaining information required to predict the future.

The two representation losses therefore have complementary roles:
\[
\begin{aligned}
\cL_T
&:
&&
\text{regularize the stochastic target-side representation},
\\
\cL_C
&:
&&
\text{compress observed-history information in the predictive state}.
\end{aligned}
\]
They are orthogonal to the Level-I--III predictive hierarchy: $\cL_T$ and $\cL_C$ determine how the latent representation is regularized, whereas the level-specific predictive likelihood determines the structure of the learned future distribution. Setting $\beta_C=0$ removes the explicit context-information bottleneck, while $\beta_C>0$ realizes the full PIB--JEPA--MDN synthesis introduced in \cref{eq:conceptual-synthesis}.

\subsection{Preventing degenerate solutions}

Even with predictive and representation supervision, a multimodal probabilistic JEPA can converge to qualitatively different degenerate solutions. These failures occur at different levels of the model and therefore require distinct safeguards.

\paragraph{Representation collapse.}
At the representation level, the encoder may map many distinct inputs to nearly the same latent state. The EMA/stop-gradient target path provides the \emph{asymmetric} predictive-learning structure used by JEPA, but does not by itself constitute a guarantee against collapse. Representation stability is therefore supported jointly by predictive supervision, the VJEPA/PIB representation regularization, and, where necessary, standard variance or covariance constraints on the learned representations.

\paragraph{Component collapse.}
Even with an informative representation, the mixture predictor can suffer from optimization collapse if only one or a few components receive meaningful posterior responsibility despite the data containing separated predictive branches.

For Levels~II--III, define the \emph{component log score} for target trajectory $Y_i$ as
\[
\eta_{ik}
=
\log \pi_{ik}
+
\sum_{h=1}^{H}
\log
\N\left(
Z_{i,t+h}^{\mathrm T};
\mu_{ik,h},
\Sigma_{ik,h}
\right).
\]
The \emph{posterior responsibility} of component $k$ is then
\[
r_{ik}
=
\Pr(M_i=k\mid Y_i,Z_{i,t},\xi_i)
=
\frac{
\exp(\eta_{ik})
}{
\sum_{j=1}^{K}
\exp(\eta_{ij})
}.
\]
Thus, $\pi_{ik}$ is the context-dependent \emph{predictive mixture weight} before observing the target trajectory, whereas $r_{ik}$ is the \emph{posterior responsibility} after the target has been observed. For Level~I, the analogous responsibility can be computed separately at each horizon; when $H=1$, the two constructions coincide.

For a minibatch of size $B$, let
\begin{equation}
\bar r_k
=
\frac{1}{B}
\sum_{i=1}^{B}
r_{ik}
\label{eq:average_component_responsibility}
\end{equation}
denote the average responsibility of component $k$. For the fixed-$K$ core model, we use the weak dead-component penalty
\begin{equation}
\cR_{\mathrm{use}}
=
\sum_{k=1}^{K}
\left[
\rho_{\min}-\bar r_k
\right]_+^2,
\label{eq:usage-reg}
\end{equation}
where $[x]_+ \coloneqq \max(x,0)$ denotes the positive-part operator. The penalty is therefore inactive for components whose average responsibility exceeds the small threshold $\rho_{\min}$ and increases only when a component receives persistently negligible responsibility. It discourages accidental dead components without forcing uniform mixture probabilities. A genuinely unambiguous predictive context may still place almost all of its probability on a single component.

The core method treats $K$ as a fixed maximum mixture size and uses $K_{\mathrm{eff}}$ from \cref{eq:effective-k} as a utilization diagnostic. When intentional component sparsity or pruning is desired, the usage safeguard should be disabled or reformulated rather than forcing every available component to remain active.

\paragraph{Variance collapse.}
A separate degeneracy can occur within an individual Gaussian component: maximum-likelihood training may drive its covariance towards zero around isolated targets, producing pathological likelihood spikes. The explicit variance floor
$
\sigma_{\min}
$
in \cref{eq:sigma-floor} provides the default protection against this failure. Additional covariance penalties can be introduced if required, but are not necessary components of the core objective.

Together, these safeguards address three complementary failure modes: preserving an informative latent representation, maintaining useful mixture-component participation, and preventing pathological collapse within individual Gaussian components.

\subsection{Level-III consistency during optimization}

Level~III adds the distributional consistency objective $\cL_{\mathrm{CK}}$ from \cref{eq:ck-loss}. At the population level, the requirement is symmetric: \emph{the direct and recursively composed distributions should represent the same underlying $h$-step predictive law}. The training implementation additionally determines how gradients from this discrepancy are routed through the two predictive paths.

In the Level-III realization evaluated in Experiment~4, samples from the separately supervised direct predictor are treated as a stop-gradient reference \emph{within the CK term only}. Schematically,
\[
\widehat{\cL}_{\mathrm{CK}}
=
\sum_{\substack{h\in\mathcal H\\h>1}}
w_h\,
\frac{
\widehat{\mathcal E}_U
\left(
\sg\!\left[
S_\phi^{\mathrm{dir},h}
\right],
S_\phi^{\mathrm{comp},h}
\right)
}{
s_h+\epsilon
},
\]
where
$
S_\phi^{\mathrm{dir},h}
$
and
$
S_\phi^{\mathrm{comp},h}
$
denote samples from the direct and recursively composed predictive laws, respectively;
$\sg[\cdot]$ denotes stop-gradient;
$\widehat{\mathcal E}_U$ is an unbiased U-statistic estimate of energy distance;
$s_h>0$ is a fixed horizon-specific normalization scale accounting for variation in the natural dispersion of the predictive law across horizons; and
$\epsilon>0$ is a small numerical constant preventing division by zero.
As in \cref{eq:ck-loss}, $\mathcal H$ denotes the selected set of consistency horizons and $w_h\geq0$ their relative weights.
Consequently, CK gradients regularize the recursive one-step dynamics towards the separately supervised direct belief, while the direct predictor continues to receive its ordinary gradients from the trajectory likelihood in \cref{eq:mdn-nll}. This stop-gradient choice affects optimization only; the forward consistency criterion remains the distributional comparison defined in \cref{eq:ck-loss}.

Agreement alone is not sufficient for predictive correctness. Two predictive paths could agree with each other while both disagreeing with the observed target law. The CK term must therefore be used together with direct trajectory supervision and the one-step likelihood in \cref{eq:step-nll}. This motivates evaluating Level~III on both direct--composed discrepancy and the predictive fidelity of the two paths.

\subsection{Overall level-wise objective}

Having defined the predictive, representation, utilization, and consistency terms separately, we now combine them into the level-aware MDN-VJEPA objective:
\begin{equation}
\begin{aligned}
\cL_{\mathrm{MDN\text{-}VJEPA}}^{(\ell)}
={}&
\underbrace{
\cL_{\mathrm{dir}}^{(\ell)}
}_{\text{direct prediction}}
+
\underbrace{
\lambda_{\mathrm{step}}
\cL_{\mathrm{step}}
}_{\text{one-step dynamics}}
\\
&+
\underbrace{
\beta_T\cL_T
+
\beta_C\cL_C
}_{\text{representation regularization}}
+
\underbrace{
\lambda_{\mathrm{use}}\cR_{\mathrm{use}}
}_{\text{component utilization}}
\\
&+
\underbrace{
\mathbf{1}\{\ell=\mathrm{III}\}
\lambda_{\mathrm{CK}}
\cL_{\mathrm{CK}}
}_{\text{distributional consistency}},
\qquad
\ell\in
\{\mathrm I,\mathrm{II},\mathrm{III}\}.
\end{aligned}
\label{eq:full-objective}
\end{equation}
Here $\mathbf{1}\{\ell=\mathrm{III}\}$ is one for Level~III and zero otherwise, so the CK term is activated only at the consistency level.

The terms in \cref{eq:full-objective} have distinct roles:
\[
\begin{aligned}
\cL_{\mathrm{dir}}^{(\ell)}
&:
&&
\text{fit the level-specific direct predictive distribution},
\\
\cL_{\mathrm{step}}
&:
&&
\text{fit the one-step transition used for recursive rollout},
\\
\cL_T
&:
&&
\text{regularize the stochastic target-side representation},
\\
\cL_C
&:
&&
\text{compress observed-history information in the predictive state},
\\
\cR_{\mathrm{use}}
&:
&&
\text{discourage accidental dead mixture components},
\\
\cL_{\mathrm{CK}}
&:
&&
\text{encourage agreement between direct and composed predictive laws}.
\end{aligned}
\]
Their definitions are given in \cref{eq:level1-nll,eq:mdn-nll,eq:step-nll,eq:target-kl,eq:context-kl,eq:usage-reg,eq:ck-loss}.

The level index $\ell$ selects the direct predictive law. Level~I uses the horizon-wise likelihood $\cL_{\mathrm{dir}}^{(\mathrm I)}$ in \cref{eq:level1-nll}, while Levels~II and III use the shared-mode trajectory likelihood $\cL_{\mathrm{dir}}^{(\mathrm{II})}$ in \cref{eq:mdn-nll}, with $\cL_{\mathrm{dir}}^{(\mathrm{III})}\equiv\cL_{\mathrm{dir}}^{(\mathrm{II})}$. Level~III additionally requires the recursively composable one-step transition and activates the CK consistency term. A Level-II model may also train the one-step transition when recursive rollout is required, but no CK constraint is imposed.

The coefficients $\beta_T$ and $\beta_C$ control target-side and context-side representation regularization, respectively, while $\lambda_{\mathrm{use}}$ controls the weak fixed-$K$ utilization safeguard. The coefficient $\lambda_{\mathrm{step}}$ controls supervised learning of the recursive one-step dynamics: it is positive when that transition is instantiated and is required for the Level-III recursive path, but may be zero for Levels~I--II when recursive rollout is not used. Finally, $\lambda_{\mathrm{CK}}$ controls the strength of the additional Level-III consistency constraint.

The representation terms $\cL_T$ and $\cL_C$ and the component-utilization safeguard $\cR_{\mathrm{use}}$ are orthogonal to the predictive hierarchy. The covariance floor $\sigma_{\min}$ in \cref{eq:sigma-floor} is likewise an architectural safeguard rather than an additional loss term, and the EMA target encoder is updated outside gradient descent.

The resulting training organization mirrors the model construction in \cref{sec:coherent-dynamics}: Level~I learns multimodal horizon-wise predictions, Level~II couples them into coherent trajectory-level alternatives, and Level~III additionally encourages consistency between the directly predicted and recursively composed predictive laws.

\section{Decision-Theoretic Consequences of Multimodal Prediction}
\label{sec:advantages}

The preceding sections establish two properties of MDN-VJEPA that become especially important when a learned world model is used for decision making. \textit{First, distinct plausible futures remain individually represented rather than being collapsed into their conditional average}. For example, if $Z_{\mathrm L}$ and $Z_{\mathrm R}$ denote two separated and equally likely valid futures in the motivating robot-navigation problem, the mean prediction
\[
\frac{1}{2}Z_{\mathrm L}
+
\frac{1}{2}Z_{\mathrm R}
\]
may lie in a region corresponding to neither realizable outcome. \textit{Second, uncertainty decomposition.} The decomposition in \cref{eq:uncertainty-decomp} separates uncertainty over \emph{which future} occurs, represented through the mixture weights and separated component means, from variability \emph{within} a particular future, represented through the component covariances. These distinctions matter beyond predictive accuracy: \textit{once a predictive belief is used for planning, replacing the complete distribution by its mean can change the resulting decision}.

\subsection{From mean-only to distributional planning}

Consider an action sequence
$
u\equiv u_{t:t+H-1}
$
over a planning horizon $H$, treating the actions as the relevant part of the side information $\xi$ and suppressing any remaining conditioning variables for readability. For the coherent Level-II/III predictor, the future latent trajectory has the conditional distribution
\begin{equation}
p_\phi
\left(
Z_{t+1:t+H}
\mid
Z_t,u
\right)
=
\sum_{k=1}^{K}
\pi_k(Z_t,u)\,
p_{\phi,k}
\left(
Z_{t+1:t+H}
\mid
Z_t,u
\right),
\label{eq:planning-mixture}
\end{equation}
where $p_{\phi,k}$ denotes the trajectory distribution conditional on mode $M=k$. The mixture weights $\pi_k(Z_t,u)$ therefore quantify the predictive support assigned to the alternative trajectory-level branches under the candidate action sequence.

\paragraph{Mean-only planning.}
Even when the complete predictive distribution in \cref{eq:planning-mixture} is available, a planner may reduce it to the conditional mean trajectory
\begin{equation}
\bar Z_{t+1:t+H}(u)
=
\E
\left[
Z_{t+1:t+H}
\mid
Z_t,u
\right].
\label{eq:planning-mean}
\end{equation}
Given a trajectory-level cost function
$
C(Z_{t+1:t+H},u)
$,
such a mean-only planner evaluates
\begin{equation}
J_{\mathrm{mean}}(u)
=
C
\left(
\bar Z_{t+1:t+H}(u),
u
\right).
\label{eq:mean-planning}
\end{equation}
A horizon-specific cost $C_h(Z_{t+h},u)$ is recovered as a special case. When the predictive distribution is multimodal, however, the mean trajectory may pass between valid branches and need not itself correspond to a likely future. In the navigation example, averaging left- and right-going trajectories can place the predicted mean near the obstacle even though neither branch follows that path. Conversely, an apparently benign mean can conceal a low-probability but high-consequence failure branch.

\paragraph{Distribution-aware planning.}
A distribution-aware planner instead evaluates the cost under the complete predictive belief,
\begin{equation}
J_{\mathrm{dist}}(u)
=
\E
\left[
C
\left(
Z_{t+1:t+H},u
\right)
\mid
Z_t,u
\right].
\label{eq:distributional-planning}
\end{equation}
Introducing the trajectory mode $M$, the \textit{law of total expectation} gives
\begin{equation}
J_{\mathrm{dist}}(u)
=
\sum_{k=1}^{K}
\pi_k(Z_t,u)\,
\E
\left[
C
\left(
Z_{t+1:t+H},u
\right)
\mid
Z_t,u,M=k
\right].
\label{eq:branch-cost-decomp}
\end{equation}
The learned mixture therefore exposes both the probability of each branch and its mode-conditioned consequence to the planner.

The same decomposition applies to safety events. Let $\mathcal F$ denote a failure event defined over the predicted trajectory, such as collision at any point during the rollout. Then
\begin{equation}
\Pr
\left(
\mathcal F
\mid
Z_t,u
\right)
=
\sum_{k=1}^{K}
\pi_k(Z_t,u)\,
\Pr
\left(
\mathcal F
\mid
Z_t,u,M=k
\right).
\label{eq:failure-probability}
\end{equation}
Thus, a low-probability branch can contribute explicitly to the predicted failure probability even when it has little influence on the conditional mean.

The distinction between \cref{eq:mean-planning,eq:distributional-planning} is fundamental. For a nonlinear cost or constraint,
\begin{equation}
C
\left(
\E[
Z_{t+1:t+H}
\mid
Z_t,u
],
u
\right)
\neq
\E
\left[
C
\left(
Z_{t+1:t+H},u
\right)
\mid
Z_t,u
\right]
\label{eq:cost-of-mean}
\end{equation}
in general. Consequently, two candidate action sequences can have similar mean predictions while inducing substantially different expected costs or failure probabilities. The distinction is therefore not only between deterministic and probabilistic world models: even the \emph{same trained MDN-VJEPA} can lead to different decisions depending on whether the planner consumes only its mean or retains its complete predictive distribution.

\paragraph{Risk-sensitive criteria.}
Expected cost is only one way to use the predictive belief. Because MDN-VJEPA retains probability mass over distinct adverse futures, the same distribution can support tail-sensitive criteria such as
\[
\operatorname{CVaR}_{\alpha}
\left(
C(Z_{t+1:t+H},u)
\right)
\]
or chance constraints of the form
\[
\Pr
\left(
\mathcal F
\mid
Z_t,u
\right)
\leq
\delta.
\]
Such criteria depend on the distribution of possible consequences rather than only on their average latent state. This yields the central decision-theoretic implication:
\[
\textit{acceptable mean prediction}
 \not\Rightarrow 
\textit{acceptable distribution over futures}.
\]
\textit{A world model intended for planning should therefore preserve not only an average predicted outcome, but also the alternatives that support it, their probabilities, and their consequences.}

\subsection{Branch-wise approximations and sampling}

The complete predictive distribution need not always be consumed at full sampling cost. MDN-VJEPA provides several increasingly detailed ways to query the \emph{same} learned predictive belief, allowing the planner to trade computational cost against fidelity to the modeled uncertainty.

\paragraph{Mean summary.}
The least expensive representation is the conditional mean trajectory
$
\bar Z_{t+1:t+H}(u)
$
from \cref{eq:planning-mean}. It provides a compact point summary but discards the distinction between alternative branches and therefore loses the \textit{between-mode} structure of the predictive distribution.

\paragraph{Branch-wise component trajectories.}
A richer approximation retains the $K$ component-mean trajectories
\begin{equation}
\left\{
\mu_{1,1:H},
\ldots,
\mu_{K,1:H}
\right\},
\label{eq:component-trajectories}
\end{equation}
together with their predictive probabilities
$
\{\pi_1,\ldots,\pi_K\}
$.
A planner can evaluate these branches separately, for example through
\[
\sum_{k=1}^{K}
\pi_k(Z_t,u)\,
C
\left(
\mu_{k,1:H},u
\right).
\]
This retains the major branching structure while approximating each mode by its conditional mean and therefore ignoring within-mode stochastic variation. It can be useful when the alternative futures are decision-relevant but large numbers of stochastic rollouts would be unnecessarily expensive.

\paragraph{Full hierarchical sampling.}
When within-branch uncertainty also affects the decision, the planner can sample from the complete coherent mixture:
\begin{equation}
\begin{aligned}
M
&\sim
\Cat
\left(
\boldsymbol{\pi}(Z_t,u)
\right),
\\
Z_{t+1:t+H}
&\sim
p_\phi
\left(
Z_{t+1:t+H}
\mid
Z_t,u,M
\right).
\end{aligned}
\label{eq:hierarchical-planning-sampling}
\end{equation}
Sampling $M$ explores uncertainty over alternative future branches, whereas sampling the trajectory conditional on $M$ explores local variability within the selected branch. This directly mirrors the between-mode and within-mode uncertainty decomposition in \cref{eq:uncertainty-decomp}.

The three interfaces can therefore be summarized as
\[
\begin{aligned}
\textit{conditional mean}
&:
&&
\text{lowest cost, but branch structure is marginalized},
\\
\textit{component trajectories}
&:
&&
\text{retain branch structure but ignore within-mode variation},
\\
\textit{full hierarchical samples}
&:
&&
\text{retain both between- and within-mode uncertainty}.
\end{aligned}
\]
MDN-VJEPA thus exposes a single predictive belief at multiple levels of computational fidelity. The practical advantage is not merely that the predictor is stochastic, but that its stochasticity remains structured around distinct plausible futures, allowing the downstream planner to marginalize, approximate, or explicitly sample that structure according to the decision problem.

\section{Experiments}
\label{sec:experiments}

The experiments follow the same progression used to construct MDN-VJEPA: we first establish that separated futures can be represented, then test whether those futures remain \textit{coherent across a trajectory}, and finally test whether the resulting predictive beliefs remain \textit{consistent under temporal composition}. This gives a nested evaluation hierarchy in which each level introduces one additional structural requirement. Before describing the individual experiments, we specify the common comparison and evaluation protocol used throughout.
 
\subsection{Common experimental protocol and evaluation criteria}
\label{sec:common-protocol}

\paragraph{Matched comparisons.}
Across controlled comparisons, we hold fixed the latent or representation space, context--target construction, action or side-information conditioning, data splits, optimizer schedule, and training budget wherever the model definitions permit. The primary comparisons therefore change the predictive family or structural constraint rather than the underlying predictive state. Experiment~1 operates directly in latent space, Experiment~3 uses a fixed representation map, and Experiments~2 and 4 operate directly on controlled predictive states. We report predictor parameter counts and include parameter-matched controls where useful.

The Branch-JEPA-style comparison in Experiment~2 is a controlled finite-support trajectory predictor trained with the full-set Energy Score; it is not a reproduction of the complete external Branch-JEPA training stack.

Train, validation, and test data are generated independently at the context, trajectory, or episode level as appropriate to each benchmark, with no realization shared across splits. Within a given stochastic comparison and metric, competing methods use the same pre-specified sampling budget unless explicitly stated otherwise; sampling budgets may differ across evaluation tasks and are reported in the corresponding experimental details. Aggregate results are reported across independent random seeds with uncertainty estimates.

For multimodal methods, ground-truth future information may be used to \emph{score} or permutation-invariantly match predicted branches during evaluation, but never to select a branch at deployment time. This distinction is important because selecting the closest prediction after observing the target can artificially inflate the apparent utility of multi-output predictors.

\paragraph{Predictive fidelity and probabilistic quality.}
No single metric is sufficient for a multimodal predictor. When a normalized predictive density is available, we report \textit{negative log-likelihood} (NLL). For sample-based predictive evaluation, we report the \textit{energy score} against the realized target. When two complete distributions are available, for example the learned and known synthetic laws, or the direct and recursively composed predictions, we additionally use sample-based distribution discrepancies such as \textit{energy distance} or \textit{maximum mean discrepancy} (MMD). Mean squared error is retained only as a secondary point-prediction diagnostic because it rewards conditional averaging and does not measure multimodal support.

When ground-truth branches or events are available, we evaluate predicted branch probabilities using reliability/calibration measures and proper event-probability scores such as the \textit{Brier score}. We additionally report \textit{empirical predictive coverage} and, where meaningful, \textit{branch precision}, \textit{recall}, and \textit{mode coverage} after permutation-invariant matching between learned components and semantic branches. Best-of-$N$ latent or trajectory error is reported only as a fixed-$N$ coverage diagnostic, rather than as a primary probabilistic score, because it can improve simply by generating more candidates.

\paragraph{Mixture utilization, trajectory coherence, and composition.}
For mixture models, we report the \textit{mean} context-dependent \textit{effective component count}
\begin{equation*}
\frac{1}{B}
\sum_{i=1}^{B}
\exp H\!\left(\boldsymbol{\pi}_i\right)
\end{equation*}
together with the \textit{average component responsibilities} $\bar r_k$ (Eq.\ref{eq:average_component_responsibility}). For trajectory coherence, we measure \emph{semantic} mode-switching after mapping predictions to known events or ground-truth branches. We do not use raw component-index switching as the primary metric, since the shared-$M$ construction fixes the component index by design.

For compositional consistency, we report \textit{direct-versus-composed energy distance} and \textit{MMD} together with target \textit{fidelity} of both predictive paths. Finally, planning experiments report task success, collision or failure rate, predictive failure-probability calibration, and risk-sensitive performance where applicable. These common criteria allow the individual experiments below to emphasize the subset of metrics relevant to each experiment.
 
\subsection{Three-level evaluation hierarchy}
\label{sec:three-levels}

\begin{table}[h]
\caption{Three-level MDN-VJEPA evaluation hierarchy. Each level adds one structural property to the preceding model, enabling a matched test of the corresponding scientific claim.}
\label{tab:three-levels}
\centering
{\scriptsize
\setlength{\tabcolsep}{2pt}
\renewcommand{\arraystretch}{0.92}
\begin{tabular}{@{}p{0.11\linewidth}p{0.23\linewidth}p{0.31\linewidth}p{0.26\linewidth}@{}}
\toprule
Level & Added mechanism & Scientific question & Main evidence \\
\midrule
I: Multimodal
&
Per-horizon GMM prediction
&
Can JEPA preserve separated plausible futures rather than their conditional average?
&
NLL, energy score, calibration, mode coverage
\\
II: Coherent
&
Level I + shared trajectory mode $M$
&
Do the predicted alternatives form coherent future trajectories rather than incompatible horizon-wise combinations?
&
trajectory score, semantic switch rate, branch coverage
\\
III: Consistent
&
Level II + recursive one-step transition + $\cL_{\mathrm{CK}}$
&
Do direct and recursively composed multimodal beliefs represent consistent multi-horizon dynamics?
&
direct--composed discrepancy, target fidelity, branch and utilization diagnostics
\\
\bottomrule
\end{tabular}
}
\end{table}

\paragraph{Level I: multimodal prediction.}
The first level isolates the effect of the predictive distribution itself. Replacing the unimodal Gaussian predictor with the horizon-wise GMM construction in \cref{eq:level1-marginal,eq:level1-trajectory} tests whether MDN-VJEPA can preserve separated conditional futures together with their probabilities and within-component uncertainties. For a multi-horizon target, Level I treats the horizon-wise mixture predictions independently:
\begin{equation*}
p_\phi^{\mathrm{I}}
\left(
Z_{t+1:t+H}\mid Z_t,\xi
\right)
=
\prod_{h=1}^{H}
\left[
\sum_{k=1}^{K}
\pi_{k,h}(Z_t,\xi)\,
\N\!\left(
Z_{t+h};
\mu_{k,h},
\Sigma_{k,h}
\right)
\right].
\end{equation*}
Thus, each horizon can be multimodal, but the component selected at one horizon need not correspond to the component selected at another. This is the appropriate negative control for testing whether multimodal marginals alone are sufficient.

\paragraph{Level II: coherent multimodal trajectories.}
Level II replaces the independent horizon-wise mixture above with the shared-mode trajectory distribution in \cref{eq:trajectory-mixture,eq:mdn-nll}. The ordering of mixture and horizon operations is therefore reversed (Eq.\ref{eq:level1-level2-contrast}):
\begin{equation*}
\prod_{h=1}^{H}
\left[
\sum_{k=1}^{K}(\cdot)
\right]
\qquad\longrightarrow\qquad
\sum_{k=1}^{K}
\left[
\prod_{h=1}^{H}(\cdot)
\right].
\end{equation*}
This apparently small structural change has an important consequence: one latent mode $M=k$ now specifies a complete candidate future trajectory. Level II therefore tests whether explicitly coupling the horizons through a shared mode improves semantic trajectory coherence without sacrificing predictive fidelity or branch coverage.

\paragraph{Level III: compositionally consistent multimodal dynamics.}
Even coherent direct trajectories need not agree with recursive propagation of the learned one-step transition. Level III therefore adds $\cL_{\mathrm{CK}}$ from \cref{eq:ck-loss}, encouraging the direct horizon predictor and the recursively composed transition law to represent the same multi-horizon predictive distribution. Consistency alone is not sufficient: both paths must also remain faithful to the observed target distribution. Level III is therefore evaluated jointly on direct--composed discrepancy and target-prediction quality.

The representation regularizers $\cL_T$ and $\cL_C$ in \cref{eq:target-kl,eq:context-kl} provide a common probabilistic representation-learning foundation across Levels~I--III rather than defining an additional predictive level; their weighted contributions enter the overall objective in \cref{eq:full-objective}. Likewise, covariance parameterization and the fixed-$K$ component-utilization safeguard are orthogonal model-design choices rather than additional levels of the predictive hierarchy.

This hierarchy serves two purposes. Scientifically, it associates each architectural addition with a distinct falsifiable claim. Practically, it gives a staged implementation protocol: establish multimodal density learning first, then trajectory coherence, and finally compositional consistency. A failure of the complete model can therefore be localized to the stage at which the relevant property first degrades.
 
\subsection{Baselines and ablations}

The core model comparisons are summarized in \cref{tab:ablations}. The columns correspond to the principal structural ingredients of the predictive hierarchy. ``GMM'' denotes a continuous Gaussian-mixture predictive family. ``Shared traj.\ $M$'' denotes the Level-II construction in which a single trajectory-level latent mode is shared across the complete future block, so that the same component identity couples all prediction horizons rather than being selected independently at each horizon. ``PIB'' denotes inclusion of the context-information compression term $\cL_C$ in \cref{eq:context-kl}, while the stochastic target-side regularization $\cL_T$ in \cref{eq:target-kl} belongs to the surrounding probabilistic VJEPA formulation. ``CK'' denotes the Level-III direct--composed consistency objective $\cL_{\mathrm{CK}}$ in \cref{eq:ck-loss}. The Branch-JEPA-style reference is marked ``atomic'' because it represents a weighted finite support over complete trajectory hypotheses rather than Gaussian mixture components; its shared-mode entry is marked ``n/a'' because it does not use the MDN trajectory-mode variable $M$.

\begin{table}[H]
\caption{Core baseline and structural-ablation matrix for the controlled experiments. Each successive MDN-VJEPA level isolates an additional property of the predictive law.}
\label{tab:ablations}
\centering
{\scriptsize
\setlength{\tabcolsep}{2.5pt}
\renewcommand{\arraystretch}{0.90}
\begin{tabular}{@{}lcccp{0.27\linewidth}@{}}
\toprule
Model & GMM & shared traj.\ $M$ & CK & Primary role \\
\midrule
Deterministic JEPA
& -- & -- & -- & point-prediction reference \\
Gaussian VJEPA
& -- & -- & -- & unimodal probabilistic reference \\
MDN-VJEPA-I
& \checkmark & -- & -- & multimodal marginals \\
MDN-VJEPA-II
& \checkmark & \checkmark & -- & trajectory coherence \\
MDN-VJEPA-III
& \checkmark & \checkmark & \checkmark & compositional consistency \\
Branch-JEPA-style
& atomic & n/a & -- & finite-support trajectory reference \\
\bottomrule
\end{tabular}
}
\end{table}

The experiments use these comparisons selectively to isolate the claim under study rather than instantiating every row in every benchmark. Experiment~1 isolates predictive multimodality in a one-step controlled latent setting by comparing deterministic, Gaussian, and MDN predictors, including a parameter-matched Gaussian control. Experiment~2 evaluates multimodal prediction and downstream planning, includes a controlled finite-support Branch-JEPA-style reference, and compares MDN-VJEPA with and without CK regularization. Experiment~3 provides the principal Level-I versus Level-II ablation by replacing the shared trajectory mode with independent horizon-wise mixture identities while preserving comparable horizon-wise multimodal capacity. Experiment~4 isolates the Level-II versus Level-III distinction by adding the distributional CK objective to otherwise matched direct and one-step MDN predictors.

The controlled experiments are designed to isolate the predictive hierarchy rather than representation learning: Experiments~1, 2, and 4 operate directly on predictive states, while Experiment~3 uses a fixed representation map. The PIB representation terms $\cL_T$ and $\cL_C$ therefore remain part of the general MDN-VJEPA formulation developed in \cref{sec:objective}, but are not activated as experimental factors in these controlled benchmarks.

Orthogonal safeguards are reported where they are actually used. In Experiment~1, the component-utilization penalty is deliberately disabled with $\lambda_{\mathrm{use}}=0$, whereas the later multimodal trajectory experiments use the weak fixed-$K$ utilization safeguard where required. The effective component count $K_{\mathrm{eff}}$ from \cref{eq:effective-k} is reported throughout as a diagnostic of mixture utilization rather than as an adaptive component-selection mechanism.
 
\subsection{Experiment 1: controlled branching dynamics}
\label{sec:exp1-branching}

We first isolate the Level-I question, i.e. whether the predictive family itself can preserve separated plausible futures, in a controlled two-dimensional latent space where the true conditional law is known exactly. This experiment deliberately operates directly on latent states rather than training an encoder, so that differences in representation learning cannot confound the comparison among deterministic, Gaussian, and mixture predictors. For context $X=(X_1,X_2)\sim\operatorname{Unif}([-1,1]^2)$, the probability of the right branch is
\begin{equation*}
p_{\mathrm R}(X)
=
\operatorname{sigmoid}(1.8X_1-0.7X_2+b),
\qquad
M^\star\sim\operatorname{Bernoulli}(p_{\mathrm R}(X)),
\end{equation*}
where the hidden branch $M^\star$ is sampled only after the context is generated and is never included in the predictor input. Conditional on the branch, the one-step target is Gaussian,
\begin{equation*}
Z_{t+1}\mid M^\star=m,X
\sim
\N\!\left(\mu_m(X),\sigma_z^2I\right),
\end{equation*}
with
\begin{equation*}
\mu_{\mathrm L}(X)
=
\begin{bmatrix}
X_1+0.20\\
0.35X_2-\delta
\end{bmatrix},
\qquad
\mu_{\mathrm R}(X)
=
\begin{bmatrix}
X_1+0.20\\
0.35X_2+\delta
\end{bmatrix}.
\end{equation*}
The base condition uses $\delta=1.25$, $\sigma_z=0.12$, and $b=0$. We use $H=1$ here intentionally: Experiment~1 isolates multimodal conditional prediction, while trajectory-level coherence is tested separately in Experiment~3. Complete implementation and training details are given in \cref{app:exp1-details}.

\paragraph{Base comparison.}
We compare deterministic JEPA, Gaussian VJEPA, a parameter-matched Gaussian VJEPA, and MDN-VJEPA with $K=2$ against the exact oracle conditional law. The parameter-matched Gaussian has $10{,}597$ trainable parameters versus $10{,}570$ for MDN-VJEPA, so the comparison does not attribute the gain from multimodality to predictor size. Results are averaged over 5 independent seeds.

\begin{table}[H]
\caption{Controlled branching prediction in the base condition. Values are mean $\pm$ standard deviation over 5 seeds. The matched Gaussian is capacity-matched to MDN-VJEPA. Lower is better for all reported metrics. Bold denotes the best learned model.}
\label{tab:exp1-base}
\centering
{\scriptsize
\setlength{\tabcolsep}{4pt}
\renewcommand{\arraystretch}{0.95}
\begin{tabular}{@{}lccc@{}}
\toprule
Model & Energy distance $\downarrow$ & Energy score $\downarrow$ & NLL $\downarrow$ \\
\midrule
Deterministic JEPA & $0.8685\pm0.0128$ & $0.9894\pm0.0141$ & -- \\
Gaussian VJEPA (matched) & $0.1913\pm0.0024$ & $0.6493\pm0.0078$ & $0.8259\pm0.0446$ \\
\textbf{MDN-VJEPA} & $\mathbf{0.0022\pm0.0007}$ & $\mathbf{0.5554\pm0.0087}$ & $\mathbf{-0.7934\pm0.0464}$ \\
Generator oracle & -- & $0.5541\pm0.0078$ & $-0.8164\pm0.0479$ \\
\bottomrule
\end{tabular}
}
\end{table}

The deterministic predictor is close to the conditional mean but is a poor description of the full branching law, with energy distance $0.8685\pm0.0128$. Allowing a single Gaussian reduces this discrepancy substantially, but the matched Gaussian still has energy distance $0.1913\pm0.0024$ because it must cover the two branches with one unimodal density. MDN-VJEPA reduces the discrepancy to $0.0022\pm0.0007$, approximately $87\times$ smaller than the matched Gaussian. Its energy score, $0.5554\pm0.0087$, is nearly identical to the oracle value $0.5541\pm0.0078$, while its NLL is only $0.0231\pm0.0075$ above the oracle NLL. MDN-VJEPA outperforms the parameter-matched Gaussian in energy distance in all 5 seeds; the mean paired gap is $0.1891\pm0.0024$.

\begin{figure}[H]
\centering
\includegraphics[width=\linewidth]{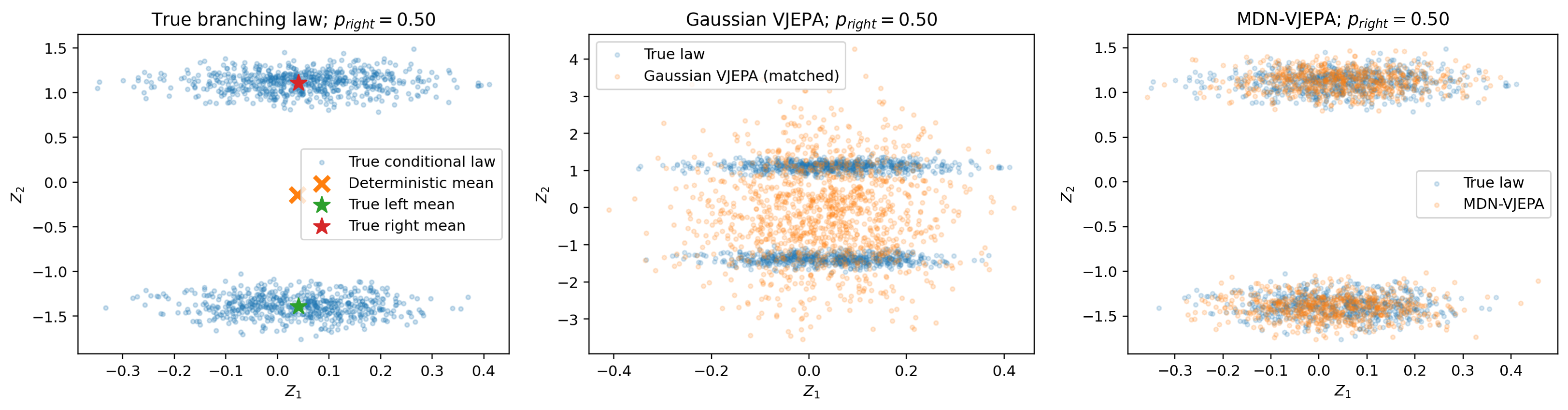}
\caption{Representative base-condition prediction at $p_{\mathrm R}\approx0.5$ using the pre-specified visualization seed. Left: the deterministic prediction lies between two concentrated valid branches. Middle: the parameter-matched Gaussian covers both alternatives only by spreading probability mass through the low-density region between them. Right: MDN-VJEPA preserves the two separated modes and their local dispersion. Ground-truth branch information is used only for synthetic evaluation and visualization.}
\label{fig:exp1-branching-comparison}
\end{figure}

\paragraph{When is multimodality necessary?}
We next vary the branch separation $\delta$ while keeping the remaining base settings fixed. \Cref{fig:exp1-separation} shows that the advantage is not universal. At $\delta\in\{0.05,0.10\}$ the branches overlap so strongly that the conditional law is effectively unimodal, and the matched Gaussian is as accurate as, or marginally better than, the mixture. Once distinct modes emerge, the difference grows rapidly: at $\delta=0.25$ the Gaussian energy distance is $0.0069\pm0.0007$ versus $0.0011\pm0.0003$ for MDN-VJEPA, and at $\delta=1.25$ the gap reaches $0.1913\pm0.0024$ versus $0.0022\pm0.0007$. Thus, the benefit of the mixture appears specifically in the regime for which it is intended---genuinely separated conditional futures---rather than from extra flexibility alone.

\begin{figure}[H]
\centering
\includegraphics[width=0.50\linewidth]{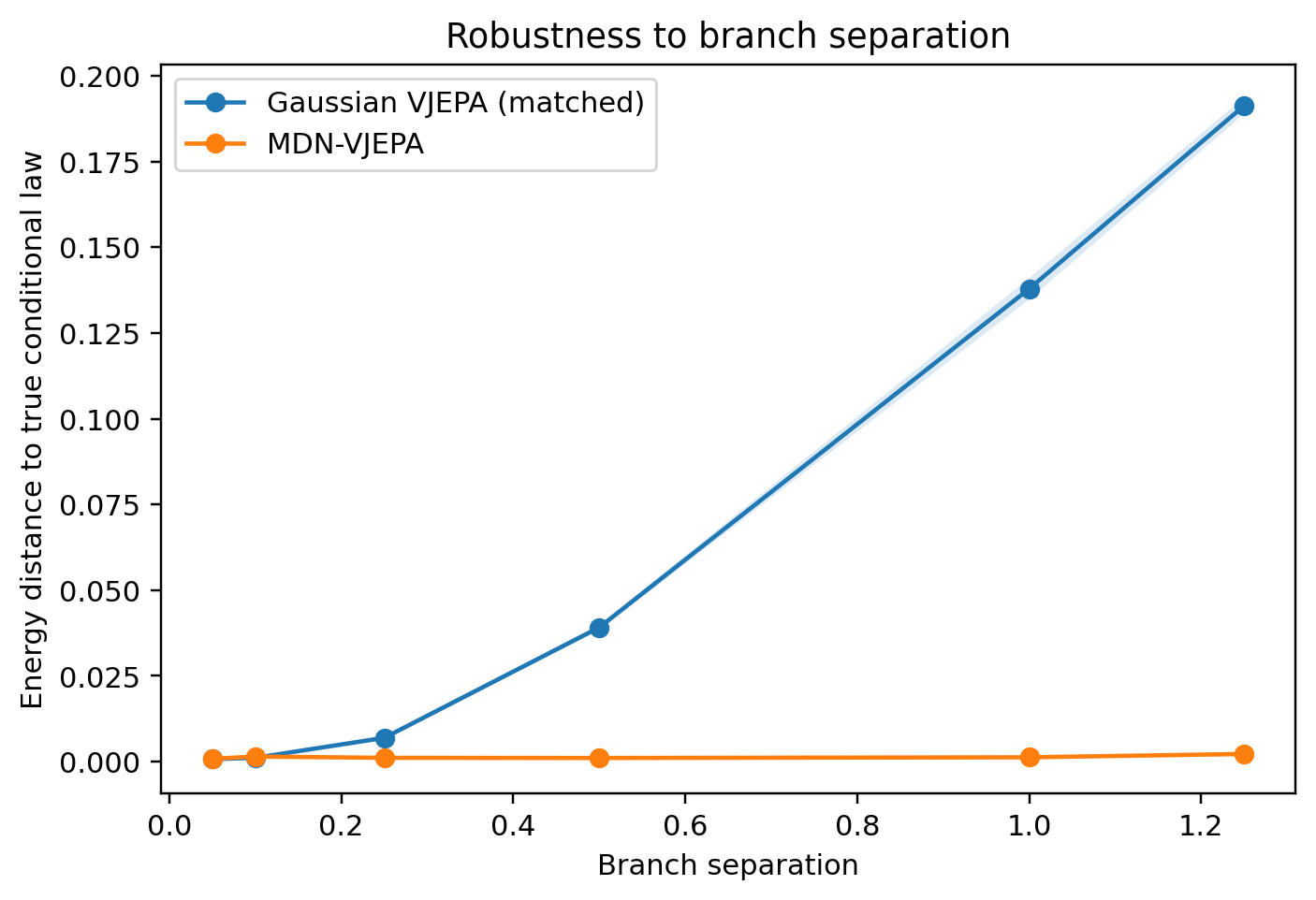}
\caption{Energy distance to the known conditional law as branch separation increases. Curves show 5-seed means with the uncertainty bands. The matched Gaussian is sufficient when the two branches nearly overlap, but its distributional error grows sharply as the future becomes genuinely multimodal, whereas MDN-VJEPA remains close to the oracle law.}
\label{fig:exp1-separation}
\end{figure}

\paragraph{Probability and uncertainty recovery.}
The learned mixture also recovers the internal structure of the true conditional law rather than merely reducing a sample discrepancy. Across seeds, the branch-probability MAE is $0.0115\pm0.0025$ and the expected calibration error is $0.0271\pm0.0058$ (Table~\ref{tab:exp1-mdn-diagnostics}). The relative errors of the within-mode and between-mode covariance terms in \cref{eq:uncertainty-decomp} are $5.71\%\pm4.12\%$ and $3.69\%\pm1.04\%$, respectively. The learned effective component count is $1.787\pm0.008$, with an absolute error of $0.0048\pm0.0062$ relative to the true mixture complexity. No component-usage penalty is active in this experiment ($\lambda_{\mathrm{use}}=0$, Eq.~\ref{eq:full-objective}), yet the two components remain approximately equally utilized: the minimum and maximum test-set mean component responsibilities are $0.488\pm0.007$ and $0.512\pm0.007$, respectively.\footnote{These quantities are not the predictive mixture probabilities $\pi_{ik}=\Pr(M_i=k\mid X_i)$. Rather, $r_{ik}=\Pr(M_i=k\mid X_i,Y_i)$ is the posterior responsibility of component $k$ after observing the realized target $Y_i$. Individual predictive probabilities $\pi_{ik}$ can differ substantially from $0.5$ depending on the context $X_i$. The values $0.488/0.512$ indicate only that, when averaged over the balanced test population, both components explain a substantial fraction of the realized targets. They therefore serve primarily as a diagnostic of component utilization or collapse, rather than as a calibration measure.} This provides no evidence of component collapse in the balanced base regime. Additional capacity controls, exact sweep tables, noise and imbalance experiments, and calibration diagnostics are reported in \cref{app:exp1-details}.
 
\subsection{Experiment 2: stochastic robot navigation}
\label{sec:exp2-navigation}

We next turn the motivating example in \cref{fig:robot-motivation} into a controlled stochastic-planning benchmark. A point robot observes its current state and evaluates candidate actions $u$. Hidden stochasticity makes the conditional future multimodal: from the same observed state--action pair, the realized trajectory may pass above or below an obstacle, with context-dependent branch probabilities and local within-branch noise. Unlike Experiment~1, which isolates one-step predictive multimodality, this benchmark uses $H=12$-step trajectories and additionally evaluates the consequences of the learned predictive distribution for downstream planning. Complete simulator, training, and evaluation details are given in \cref{app:exp2-details}.

This experiment asks:
\begin{enumerate}[label=(Q\arabic*)]
    \item Does the model assign probability mass to distinct feasible or failure branches rather than representing them only through their geometric average?
    \item Are predicted branch and failure probabilities calibrated against empirical event frequencies?
    \item Does planning with the complete predictive mixture improve success, expected cost, or collision risk relative to planning from the predictive mean alone?
    \item Does $\cL_{\mathrm{CK}}$ improve the consistency of long-horizon route beliefs under recursive rollout while retaining predictive fidelity?
\end{enumerate}

To isolate the decision-theoretic value of preserving multimodality from differences in representation quality, we evaluate both a \emph{mean-only} and a \emph{full-distribution} planner using the same trained MDN-VJEPA+CK checkpoint. The former evaluates each candidate action using the \textit{conditional mean} trajectory, whereas the latter evaluates expected cost under samples from the \textit{complete predictive mixture}. Both use the same candidate action grid; the stochastic planner uses a fixed pre-specified Monte Carlo budget reported in \cref{app:exp2-details}. Thus, this comparison changes how the same learned predictive belief is used for decision making, rather than changing the world model itself.

A particularly relevant comparison is a controlled finite-support Branch-JEPA-style predictor \citep{song2026branchjepa}. In our benchmark this serves as an atomic predictive-family control rather than a reproduction of the complete external Branch-JEPA system. It represents several weighted trajectory hypotheses but does not attach continuous local uncertainty to each hypothesis. This contrasts
\begin{equation*}
\sum_{k=1}^{K}\pi_k\delta_{\mu_k}
\qquad\text{with}\qquad
\sum_{k=1}^{K}\pi_k\N(\mu_k,\Sigma_k),
\end{equation*}
allowing us to distinguish the value of identifying separate branches from the additional value of modeling uncertainty within each branch.

\paragraph{Predictive multimodality and calibration.}
\Cref{tab:exp2-predictive} reports predictive quality over five independent seeds. Deterministic JEPA performs poorly as a distributional model, with energy distance $0.5141\pm0.0053$ to the known simulator law. A parameter-matched Gaussian improves this substantially to $0.1884\pm0.0018$, but remains unable to represent the separated route modes. The controlled atomic predictor captures the branch structure and reduces energy distance to $0.0212\pm0.0004$. MDN-VJEPA without CK further reduces it to $0.0026\pm0.0006$, approximately $8\times$ smaller than the atomic predictor, while its energy score $0.3049\pm0.0023$ is close to the oracle value $0.3037\pm0.0016$.

\begin{table}[H]
\caption{Predictive quality on stochastic robot navigation. Values are mean $\pm$ standard deviation over 5 seeds. Energy distance uses an off-diagonal finite-sample estimator and may therefore be slightly negative for two samples from the same distribution; the generator-oracle value is reported after rounding.}
\label{tab:exp2-predictive}
\centering
{\scriptsize
\setlength{\tabcolsep}{3.2pt}
\renewcommand{\arraystretch}{0.95}
\begin{tabular}{@{}lcccc@{}}
\toprule
Model
& Energy score $\downarrow$
& Energy dist. $\downarrow$
& Failure Brier $\downarrow$
& Failure MAE $\downarrow$ \\
\midrule
Deterministic JEPA
& $0.5600\pm0.0025$
& $0.5141\pm0.0053$
& $0.4738\pm0.0040$
& $0.4649\pm0.0026$ \\
Gaussian VJEPA (matched)
& $0.3979\pm0.0018$
& $0.1884\pm0.0018$
& $0.3456\pm0.0059$
& $0.3501\pm0.0093$ \\
Branch-JEPA (controlled atomic)
& $0.3145\pm0.0019$
& $0.0212\pm0.0004$
& $0.1803\pm0.0046$
& $0.0411\pm0.0032$ \\
\textbf{MDN-VJEPA}
& $\mathbf{0.3049\pm0.0023}$
& $\mathbf{0.0026\pm0.0006}$
& $\mathbf{0.1757\pm0.0027}$
& $\mathbf{0.0339\pm0.0017}$ \\
MDN-VJEPA + CK
& $0.3092\pm0.0052$
& $0.0121\pm0.0113$
& $0.1781\pm0.0071$
& $0.0496\pm0.0173$ \\
Generator oracle
& $0.3037\pm0.0016$
& $0.0000\pm0.0002$
& $0.1740\pm0.0039$
& $0.0213\pm0.0004$ \\
\bottomrule
\end{tabular}
}
\end{table}

The support-oriented diagnostics sharpen this comparison. As detailed in \cref{app:exp2-details}, the atomic predictor and both MDN variants obtain full branch coverage and weighted support precision under the evaluation tolerance, showing that all three recover the two route families. The remaining advantage of MDN-VJEPA therefore comes primarily from modeling the continuous dispersion around those branches rather than from discovering an additional route. Its predicted branch probabilities are also well calibrated (\cref{tab:exp2-calibration,fig:exp2-failure-calibration}): without CK, the branch-probability MAE is $0.0162\pm0.0048$, branch ECE is $0.0241\pm0.0102$, and failure ECE is $0.0204\pm0.0062$. Thus, the mixture weights recover not only which alternatives exist but also how predictive probability is distributed among them.

\paragraph{Planning with the mean \textit{vs.} with the complete distribution.}
The strongest decision-theoretic comparison uses the same MDN-VJEPA+CK checkpoint for both planners. As shown in \cref{tab:exp2-planning}, collapsing its prediction to the conditional mean yields a collision rate of $0.5563\pm0.0095$ and success rate of only $0.4437\pm0.0095$. Retaining the complete mixture instead reduces collision to $0.0025\pm0.0003$ and raises success to $0.9975\pm0.0003$. The corresponding true expected objective falls from $5.6473\pm0.0820$ to $1.1828\pm0.0018$, close to the oracle value $1.1769\pm0.0017$. The result directly illustrates that evaluating
\begin{equation*}
C\!\left(\E[Z\mid Z_t,u]\right)
\end{equation*}
need not produce the same decision as evaluating
\begin{equation*}
\E\!\left[C(Z)\mid Z_t,u\right]
\end{equation*}
when the predictive law is multimodal.

\begin{table}[H]
\caption{Planning performance in the stochastic navigation benchmark. Values are mean $\pm$ standard deviation over 5 seeds. The first two rows use exactly the same MDN-VJEPA+CK world model and differ only in whether planning uses its conditional mean or complete predictive distribution.}
\label{tab:exp2-planning}
\centering
{\scriptsize
\setlength{\tabcolsep}{3.0pt}
\renewcommand{\arraystretch}{0.95}
\begin{tabular}{@{}lcccc@{}}
\toprule
Planner
& Success $\uparrow$
& Collision $\downarrow$
& Path length $\downarrow$
& True objective $\downarrow$ \\
\midrule
MDN-CK mean-only
& $0.4437\pm0.0095$
& $0.5563\pm0.0095$
& $1.1967\pm0.0065$
& $5.6473\pm0.0820$ \\
\textbf{MDN-CK full-distribution}
& $\mathbf{0.9975\pm0.0003}$
& $0.0025\pm0.0003$
& $\mathbf{1.1626\pm0.0007}$
& $1.1828\pm0.0018$ \\
MDN-CK CVaR
& $0.9984\pm0.0002$
& $0.0016\pm0.0002$
& $1.1673\pm0.0089$
& $1.1799\pm0.0087$ \\
Gaussian full-distribution
& $1.0000\pm0.0000$
& $\mathbf{0.0000\pm0.0000}$
& $1.4539\pm0.0008$
& $1.4539\pm0.0008$ \\
Branch-JEPA atomic expected
& $0.9399\pm0.0305$
& $0.0601\pm0.0305$
& $1.1617\pm0.0002$
& $1.6429\pm0.2439$ \\
Oracle expected-cost
& $0.9982\pm0.0002$
& $0.0018\pm0.0002$
& $1.1623\pm0.0001$
& $\mathbf{1.1769\pm0.0017}$ \\
\bottomrule
\end{tabular}
}
\end{table}

The Gaussian planner is informative in a different way. Its collision rate is essentially zero, but its mean path length is $1.4539\pm0.0008$, approximately $25\%$ longer than the MDN route. A single broad Gaussian places predictive mass in the region between the two feasible branches and therefore induces a highly conservative detour. MDN-VJEPA can instead preserve two concentrated safe alternatives, achieving nearly oracle-level risk while retaining a short route. The CVaR planner similarly reaches a collision rate of $0.0016\pm0.0002$, showing that the learned mixture can support tail-sensitive decision criteria in addition to expected-cost planning.

\Cref{fig:exp2-planning-routes} illustrates the central mean-versus-distribution effect for the pre-specified representative seed. Although the conditional mean trajectory appears to avoid the obstacle, one of the underlying stochastic branches collides with it. Retaining the complete predictive mixture instead leads the planner to a short action for which both route alternatives avoid the obstacle; in this representative case, the selected action also agrees with the simulator-based oracle.

\begin{figure}[H]
\centering
\includegraphics[width=0.9\linewidth]{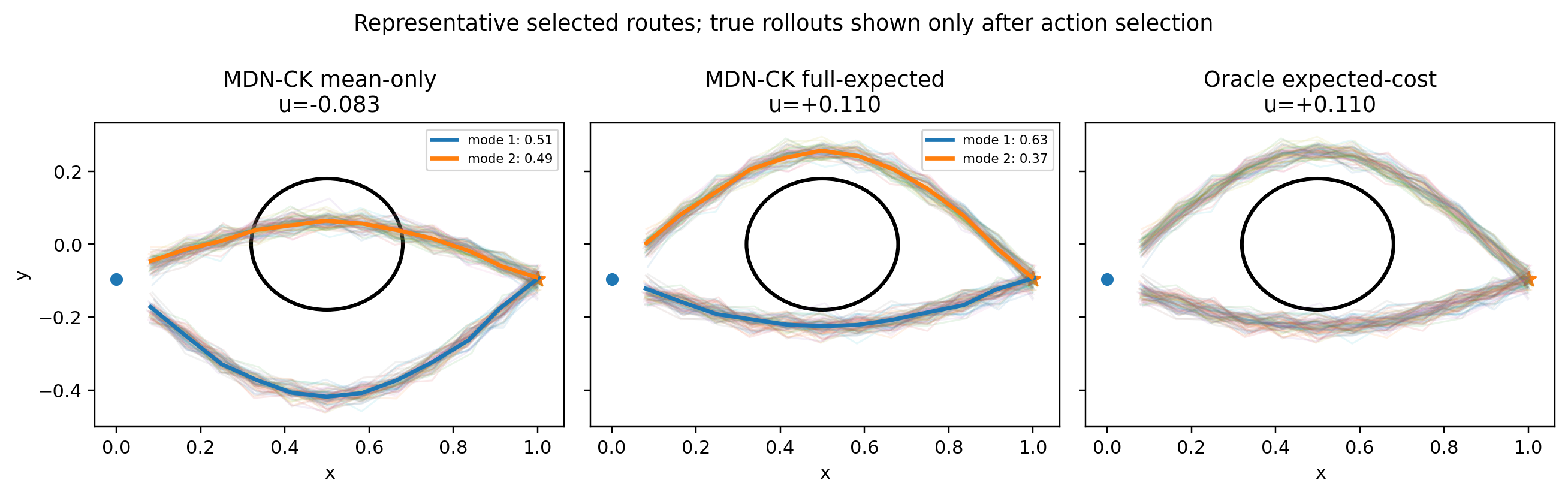}
\caption{Representative planning example using the pre-specified visualization seed. Left: planning from the conditional mean of MDN-VJEPA+CK selects an action whose averaged trajectory appears acceptable, although one of the underlying future branches intersects the obstacle. Middle: planning from the complete mixture preserves both branches and selects a short action for which both alternatives avoid the obstacle. Right: the simulator-based oracle selects the same action in this example. Thin trajectories are simulator rollouts shown only after action selection; they are never available to the learned planner.}
\label{fig:exp2-planning-routes}
\end{figure}

\paragraph{Distributional composition consistency.}
Finally, we compare MDN models initialized from the same common warm-start checkpoint and differing only in whether $\cL_{\mathrm{CK}}$ is applied. On the pre-specified $512$-context CK evaluation subset, using $96$ predictive samples per distribution, CK reduces direct--composed energy distance from $0.1585\pm0.0097$ to $0.1236\pm0.0119$, a reduction of approximately $22\%$, and the reduction occurs in all 5 seeds. The recursively composed prediction also becomes slightly more faithful to the target law: its energy score improves from $0.3806\pm0.0119$ to $0.3741\pm0.0136$. Direct predictive energy score changes only slightly, from $0.3052\pm0.0055$ to $0.3076\pm0.0089$. Thus, in this benchmark CK trades a small amount of direct predictive accuracy for substantially better agreement under temporal composition.

\begin{figure}[H]
\centering
\includegraphics[width=0.38\linewidth]{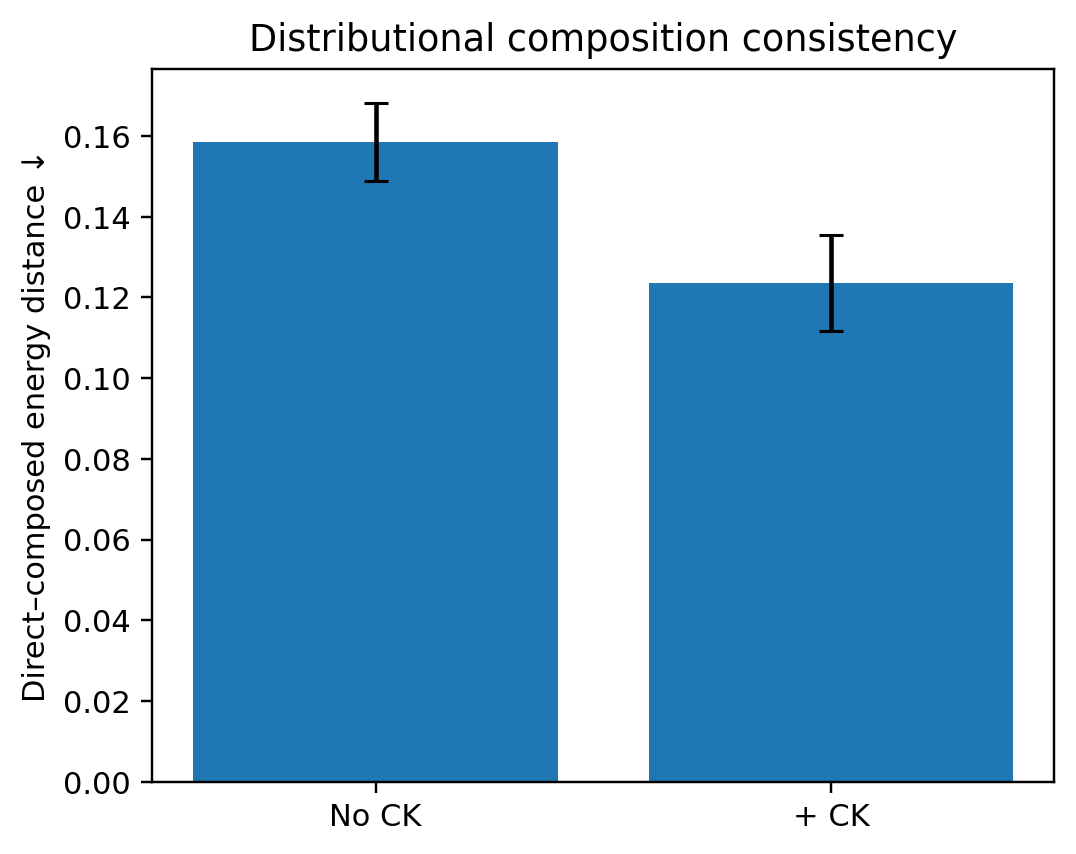}
\caption{Direct-versus-composed energy distance with and without the distributional CK objective. Both models begin from the same common MDN warm-start. CK reduces the discrepancy across all 5 seeds while preserving strong target fidelity.}
\label{fig:exp2-ck}
\end{figure}

The CK model remains genuinely multimodal after regularization (\cref{tab:exp2-calibration}): its minimum effective component count across the 5 seeds is $1.849$, and no component has near-zero posterior responsibility. The additional compositional constraint does, however, produce a modest calibration tradeoff: direct energy distance to the true law increases from $0.0026\pm0.0006$ to $0.0121\pm0.0113$ (\cref{tab:exp2-predictive}), and branch-probability MAE increases from $0.0162\pm0.0048$ to $0.0535\pm0.0281$ (\cref{tab:exp2-calibration}). We therefore interpret CK as a structural consistency regularizer rather than as an unconditional improvement to every one-step or direct predictive metric. Full diagnostics, risk sweeps, and the auxiliary MMD comparison are reported in \cref{app:exp2-details}.

\subsection{Experiment 3: stochastic video prediction in representation space}
\label{sec:exp3-video}

We next isolate the Level-II claim that multimodal predictions should form coherent futures rather than merely accurate multimodal marginals. We use a controlled stochastic-video benchmark in which the same observed motion can subsequently follow one of three semantic futures: turn downward, continue approximately straight, or turn upward. The hidden future mode is sampled only after the observed context has been formed and is never provided to the predictor.

To separate trajectory coherence from representation-learning quality, frames are mapped through a fixed, target-independent representation map to $D=12$ dimensional frame tokens, and only the predictive models are trained. The first two coordinates of the deterministic token map correspond to the object centroid and provide a fixed probe for visualization; the remaining coordinates contain nonlinear spatial features. Pixels and semantic branch labels are never supplied to the predictive models. Thus, Experiment~3 deliberately tests the temporal structure of the predictive law rather than the ability of an encoder to discover the representation itself. Full simulator, training, and evaluation details are given in \cref{app:exp3-details}.

The primary comparison follows the hierarchy introduced before (Eq.\ref{eq:level1-level2-contrast}):
\begin{equation*}
\underbrace{
\prod_{h=1}^{H}
\left[
\sum_{k=1}^{K}(\cdot)
\right]
}_{\text{Level I: independent horizon-wise mixtures}}
\qquad\textit{vs.}\qquad
\underbrace{
\sum_{k=1}^{K}
\left[
\prod_{h=1}^{H}(\cdot)
\right]
}_{\text{Level II: shared trajectory mode}}.
\end{equation*}
Level I can therefore represent all three futures correctly at every individual horizon while independently resampling their identities across time. Level II instead samples one categorical mode $M$ and uses it for the complete future block. Both variants use the same $K=3$ Gaussian component family, common component architecture and initialization, data splits, optimization protocol, and fixed stochastic evaluation budget.

This experiment asks:
\begin{enumerate}[label=(Q\arabic*)]
    \item Can Levels I and II both recover the correct multimodal support at individual horizons?
    \item Does sharing one trajectory-level mode prevent semantically incompatible switching between future branches?
    \item Does trajectory coherence improve the joint predictive distribution without sacrificing branch probabilities or marginal support?
    \item Do the horizon-specific predictions belonging to one Level-II component form one semantic future, rather than appearing coherent merely because the numerical component index is fixed?
\end{enumerate}

\paragraph{Correct marginals do not imply coherent trajectories.}
Our primary evaluation assigns every sampled future token to its nearest known semantic branch using an evaluation-only, permutation-invariant matching procedure. Semantic switching is measured only from future frame $4$ onward, where the three branches are well separated in representation space. A complete trajectory is coherent when its semantic assignment remains unchanged across all evaluated horizons. At the fixed $N=64$ sample budget, coherent branch coverage additionally asks whether every semantic branch is represented by at least one complete coherent sample.

The resulting separation between Level I and Level II is striking. As shown in \cref{tab:exp3-coherence,fig:exp3-semantic-switch}, Level I attains perfect marginal branch coverage, yet its semantic switch rate is $0.6379\pm0.0023$, and $0.9961\pm0.0004$ of sampled trajectories switch semantic branch at least once. Consequently, its coherent branch coverage is only $0.0759\pm0.0069$. Level II retains the same perfect marginal branch coverage while reducing both semantic-switch metrics to zero and increasing coherent branch coverage to $1.0000\pm0.0000$ across all 5 seeds.

\begin{table}[H]
\caption{Trajectory-coherence results for Experiment~3. Values are mean $\pm$ standard deviation over 5 seeds. Marginal branch coverage measures whether the semantic branches are represented at each horizon, whereas coherent branch coverage requires complete branch-consistent trajectories under the fixed $N=64$ sample budget.}
\label{tab:exp3-coherence}
\centering
{\scriptsize
\setlength{\tabcolsep}{2.3pt}
\renewcommand{\arraystretch}{0.95}
\begin{tabular}{@{}lccccc@{}}
\toprule
Model
& Switch rate $\downarrow$
& Any switch $\downarrow$
& Coherent cov. $\uparrow$
& Marginal cov. $\uparrow$
& Branch MAE $\downarrow$ \\
\midrule
MDN-VJEPA-I
& $0.6379\pm0.0023$
& $0.9961\pm0.0004$
& $0.0759\pm0.0069$
& $1.0000\pm0.0000$
& $\mathbf{0.0118\pm0.0049}$ \\
\textbf{MDN-VJEPA-II}
& $\mathbf{0.0000\pm0.0000}$
& $\mathbf{0.0000\pm0.0000}$
& $\mathbf{1.0000\pm0.0000}$
& $1.0000\pm0.0000$
& $0.0141\pm0.0070$ \\
Generator oracle
& $0.0000\pm0.0000$
& $0.0000\pm0.0000$
& $0.9999\pm0.0001$
& $1.0000\pm0.0000$
& $0.0000\pm0.0000$ \\
\bottomrule
\end{tabular}
}
\end{table}

\begin{figure}[H]
\centering
\includegraphics[width=0.46\linewidth]{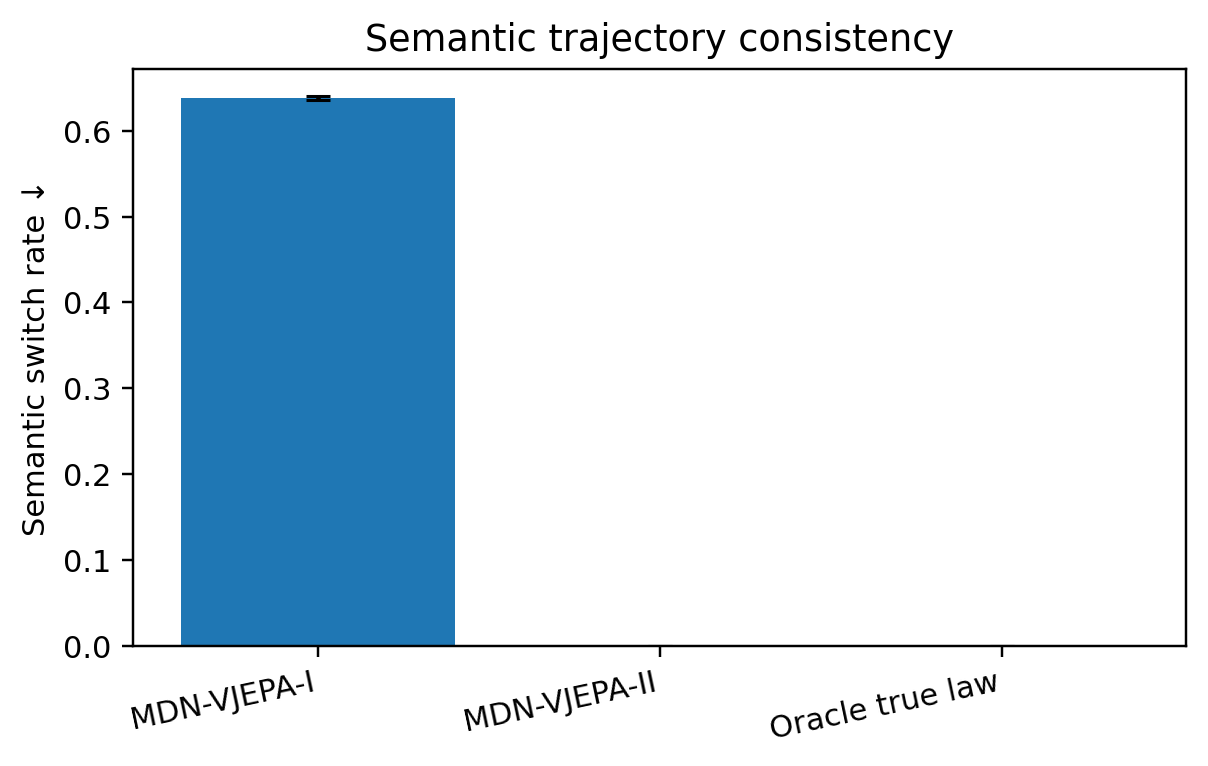}
\caption{Semantic trajectory inconsistency in Experiment~3. Level I learns accurate multimodal marginals but independently combines their modes across horizons, producing frequent semantic switching. Sharing one trajectory-level mode in Level II removes this inconsistency across all 5 seeds. Error bars show one standard deviation.}
\label{fig:exp3-semantic-switch}
\end{figure}

Importantly, the coherence gain cannot be explained by Level I failing to learn the branches. Both models achieve marginal branch coverage $1.0$, and their branch-probability errors are similarly small: $0.0118\pm0.0049$ for Level I and $0.0141\pm0.0070$ for Level II. Their mean effective component counts are also similar, $2.8515\pm0.0373$ and $2.8637\pm0.0204$, respectively; detailed support and utilization diagnostics are reported in \cref{tab:exp3-diagnostics}. Thus, Level I learns essentially the correct set of horizon-wise alternatives but gives them the wrong joint temporal factorization.

\paragraph{Joint predictive fidelity.}
The same distinction appears in proper trajectory-level scores. \Cref{tab:exp3-predictive} compares the learned joint predictive laws with deterministic, Gaussian, and oracle references. Relative to Level I, Level II improves trajectory NLL from $-230.0638\pm0.3491$ to $-239.5576\pm0.4521$, reduces Energy Score from $2.1387\pm0.0206$ to $1.8331\pm0.0205$, and reduces energy distance to the true trajectory law from $0.6246\pm0.0075$ to $0.0151\pm0.0076$, a reduction of approximately $97.6\%$. Its fixed-budget best-of-$64$ trajectory RMSE also decreases from $0.2063\pm0.0028$ to $0.0430\pm0.0003$, close to the oracle value $0.0412\pm0.0001$.

\begin{table}[H]
\caption{Joint predictive quality for stochastic representation-space video futures. Values are mean $\pm$ standard deviation over 5 seeds. NLL is unavailable for the deterministic point predictor.}
\label{tab:exp3-predictive}
\centering
{\scriptsize
\setlength{\tabcolsep}{3.0pt}
\renewcommand{\arraystretch}{0.95}
\begin{tabular}{@{}lcccc@{}}
\toprule
Model
& Traj. NLL $\downarrow$
& Energy score $\downarrow$
& Energy dist. $\downarrow$
& Best-of-$64$ RMSE $\downarrow$ \\
\midrule
Deterministic JEPA
& --
& $2.9580\pm0.0289$
& $2.2534\pm0.0070$
& $0.2701\pm0.0026$ \\
Gaussian VJEPA
& $-105.2598\pm0.2659$
& $2.1610\pm0.0190$
& $0.6697\pm0.0066$
& $0.3037\pm0.0028$ \\
MDN-VJEPA-I
& $-230.0638\pm0.3491$
& $2.1387\pm0.0206$
& $0.6246\pm0.0075$
& $0.2063\pm0.0028$ \\
\textbf{MDN-VJEPA-II}
& $\mathbf{-239.5576\pm0.4521}$
& $\mathbf{1.8331\pm0.0205}$
& $\mathbf{0.0151\pm0.0076}$
& $\mathbf{0.0430\pm0.0003}$ \\
Generator oracle
& $-244.2569\pm0.1707$
& $1.8233\pm0.0208$
& $0.0000\pm0.0031$
& $0.0412\pm0.0001$ \\
\bottomrule
\end{tabular}
}
\end{table}

The distinction is visible directly in the representation-space trajectories. \Cref{fig:exp3-trajectories} shows the pre-specified representative example. Although Level I tracks the correct three branch locations at individual horizons, complete samples repeatedly splice downward, straight, and upward futures into trajectories that do not correspond to any coherent realization. Level II instead produces three clean trajectory families that follow the true branch geometry.

\begin{figure}[H]
\centering
\includegraphics[width=0.96\linewidth]{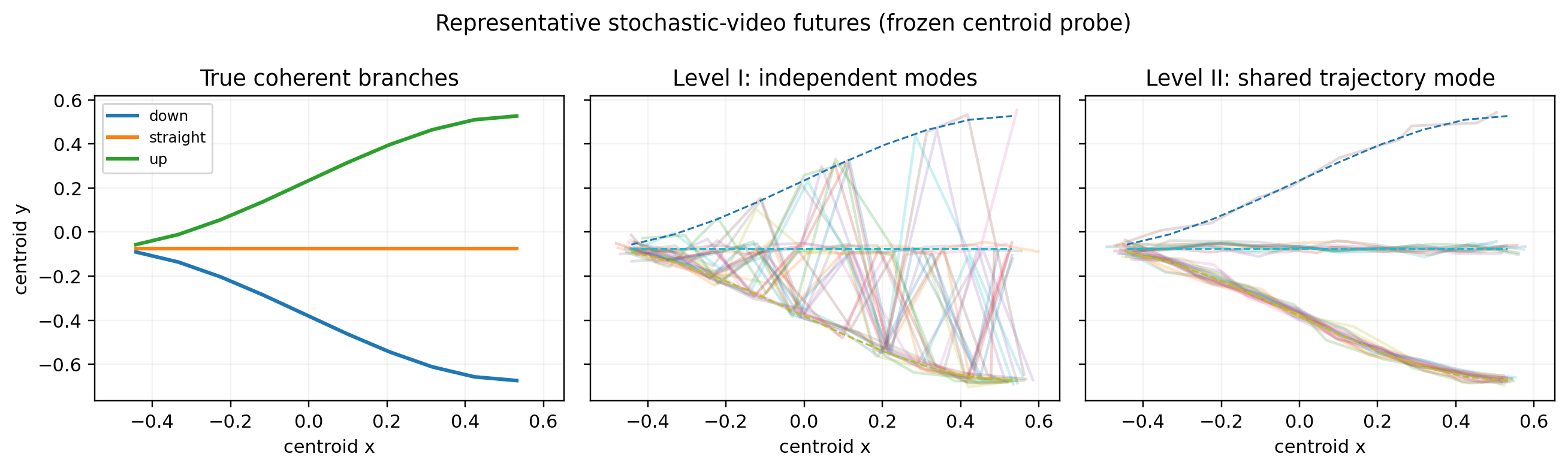}
\caption{Representative stochastic-video futures in the frozen centroid-probe space. Left: the three true coherent semantic branches. Middle: Level-I samples recover the correct horizon-wise support but independently switch among semantic branches, producing implausible complete trajectories. Right: Level II shares one trajectory-level mode and produces coherent down, straight, and up futures. Dashed curves denote the corresponding true branch means. The representative seed and example are pre-specified rather than selected after viewing test results.}
\label{fig:exp3-trajectories}
\end{figure}

Finally, the zero sampled switch rate is not merely a consequence of keeping the numerical component index fixed. We independently remap every horizon-specific Level-II component mean to its nearest semantic branch after prediction and obtain a component-mean semantic switch rate of zero for every seed. Together with the near-oracle trajectory scores, this shows that the shared-$M$ construction learns genuinely coherent semantic futures rather than only persistent mixture labels. The corresponding within-branch uncertainty tubes, probability diagnostics, capacity comparison, and finite-sample coverage results are provided in \cref{app:exp3-details}.

\subsection{Experiment 4: direct \textit{vs.} composed multi-horizon prediction}
\label{sec:exp4-ck}

The final experiment isolates the Level-III claim that a coherent multimodal world model should describe compatible future distributions whether a future state is predicted directly or reached by recursively composing the learned one-step transition. Experiment~2 already tests this property inside the stochastic navigation task; here we remove the downstream planning objective and evaluate compositional consistency directly over a controlled action-conditioned Markov process whose true predictive law can be sampled exactly.

We use a three-dimensional continuous latent state $Z_t=(x_t,y_t,r_t)$ with state-dependent branching dynamics. At every transition, a binary motion mode is sampled from a distribution conditioned only on the current modeled state and action. The continuous coordinate $r_t$ acts as a turn-memory variable, making recently selected directions increasingly persistent without introducing an unobserved persistent branch label. The resulting process is therefore Markovian in the modeled latent state. A known sequence of continuous controls is supplied as side information $\xi_{t:t+H-1}$. Complete simulator and optimization details are given in \cref{app:exp4-details}.

The primary comparison is between MDN-VJEPA-II and MDN-VJEPA-III. Both contain the same shared-mode direct trajectory predictor and the same one-step MDN transition, and both optimize direct trajectory likelihood, one-step transition likelihood, and the same component-utilization safeguard. They are first trained to a common warm-start checkpoint and then cloned exactly. The two fine-tuning arms use the same data, minibatch order, architecture, initialization, optimizer settings, and maximum training budget; the structural difference is
\begin{equation*}
\lambda_{\mathrm{CK}}=0
\qquad\textit{vs.}\qquad
\lambda_{\mathrm{CK}}>0.
\end{equation*}
For CK optimization, the directly supervised predictive distribution is treated as a stop-gradient anchor only inside the CK term; its parameters continue to be trained normally through the trajectory likelihood. CK therefore regularizes the recursively propagated transition toward the separately supervised direct belief rather than allowing both paths to move toward one another. No seed is retried or replaced based on test performance.

For the same initial state and control sequence, we compare the direct marginal
\begin{equation*}
p_\phi^{\mathrm{dir},h}
\left(
Z_{t+h}
\mid
Z_t,\xi_{t:t+h-1}
\right)
\end{equation*}
with the distribution obtained by ancestral composition of the one-step transition,
\begin{equation*}
p_\phi^{\mathrm{comp},h}
\left(
Z_{t+h}
\mid
Z_t,\xi_{t:t+h-1}
\right).
\end{equation*}
We evaluate $h\in\{1,2,4,8,12,16\}$, while $\cL_{\mathrm{CK}}$ is applied only at $h>1$. The $h=1$ comparison therefore remains an independently likelihood-trained diagnostic rather than a quantity explicitly regularized by CK.

This experiment asks:
\begin{enumerate}[label=(Q\arabic*)]
    \item How does direct--composed discrepancy vary with prediction horizon in the absence of CK?
    \item Does Level III reduce disagreement across the multi-horizon prediction range rather than only at one selected endpoint?
    \item Do the direct and recursively composed predictive paths retain fidelity to the true future law after adding CK?
    \item Is any consistency improvement obtained without collapsing the multimodal predictive distribution or distorting its branch probabilities?
\end{enumerate}

\paragraph{CK improves consistency across the multi-horizon prediction range.}
We use an unbiased U-statistic estimate of energy distance as the primary direct--composed discrepancy. Because the natural dispersion of the true future law varies substantially with horizon, we additionally normalize the energy distance at each horizon by a held-out true-law dispersion scale $s_h$ fixed before model test evaluation. Our pre-specified scalar summary is the normalized trapezoidal area under the discrepancy curve over the evaluated CK horizons
$
\mathcal H_{\mathrm{CK}}=\{2,4,8,12,16\}
$.

As shown in \cref{tab:exp4-main,fig:exp4-normalized-ck}, adding CK reduces normalized energy-distance AUC from $0.02040\pm0.00227$ to $0.01716\pm0.00162$, a reduction of approximately $15.9\%$. The corresponding raw energy-distance AUC decreases from $0.01266\pm0.00165$ to $0.01076\pm0.00120$, or approximately $15.0\%$. Both AUC reductions occur in all 5 seeds.

\begin{table}[h]
\caption{Multi-horizon composition consistency in Experiment~4. Values are mean $\pm$ standard deviation over 5 seeds. The normalized energy-distance AUC is the pre-specified primary scalar consistency summary over $h\in\{2,4,8,12,16\}$; lower is better. Energy Scores are reported at $h=16$. Bold denotes the lower learned value for the discrepancy and Energy-Score metrics; $K_{\mathrm{eff}}$ is a utilization diagnostic.}
\label{tab:exp4-main}
\centering
{\scriptsize
\setlength{\tabcolsep}{2.5pt}
\renewcommand{\arraystretch}{0.95}
\begin{tabular}{@{}lccccc@{}}
\toprule
Model
& Norm. ED-AUC $\downarrow$
& Raw ED-AUC $\downarrow$
& Direct ES $\downarrow$
& Composed ES $\downarrow$
& $K_{\mathrm{eff}}$ \\
\midrule
MDN-VJEPA-II
& $0.02040\pm0.00227$
& $0.01266\pm0.00165$
& $0.54507\pm0.01003$
& $\mathbf{0.54427\pm0.00937}$
& $1.8622\pm0.0192$ \\
\textbf{MDN-VJEPA-III}
& $\mathbf{0.01716\pm0.00162}$
& $\mathbf{0.01076\pm0.00120}$
& $\mathbf{0.54442\pm0.01066}$
& $0.54533\pm0.01105$
& $1.8681\pm0.0124$ \\
\bottomrule
\end{tabular}
}
\end{table}

\begin{figure}[h]
\centering
\includegraphics[width=0.48\linewidth]{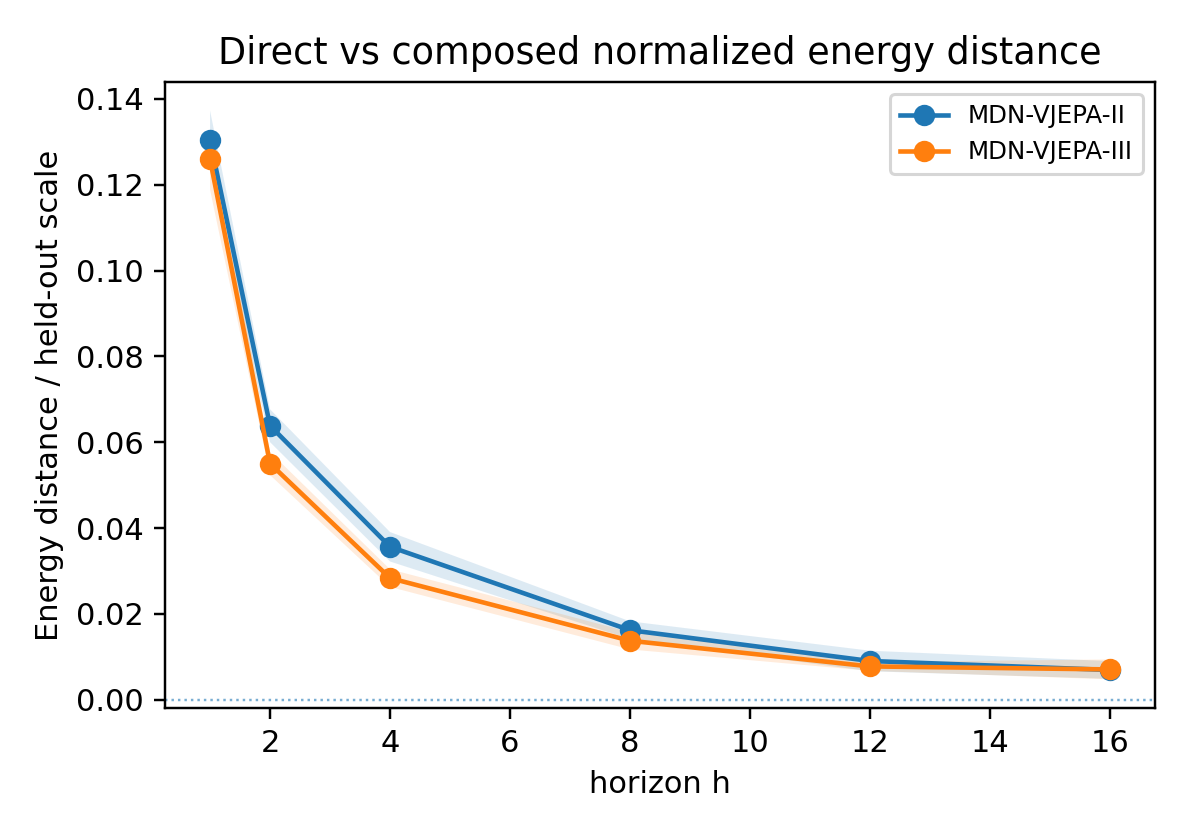}
\caption{Direct-versus-composed energy distance normalized by a held-out true-law dispersion scale at each horizon. Level III reduces discrepancy at $h=2,4,8,$ and $12$ and is comparable to Level II at $h=16$. Its integrated $h>1$ discrepancy is lower in all 5 seeds. Error bands show one standard deviation.}
\label{fig:exp4-normalized-ck}
\end{figure}

The horizon-wise result is more informative than the terminal value alone. Level III reduces raw energy distance from $0.01975$ to $0.01702$ at $h=2$, from $0.01862$ to $0.01481$ at $h=4$, from $0.01296$ to $0.01100$ at $h=8$, and from $0.00871$ to $0.00749$ at $h=12$. At $h=16$, the two models are effectively tied, with $0.00748\pm0.00222$ for Level II and $0.00762\pm0.00250$ for Level III. Thus, the empirical claim is not that CK dominates at every endpoint, but that it reduces disagreement over the multi-horizon prediction range. Detailed horizon-wise results are reported in \cref{tab:exp4-horizon}.

Interestingly, the no-CK discrepancy does not monotonically accumulate with horizon. It becomes smaller at later horizons as the controlled Markov process increasingly concentrates into persistent dynamical regimes. The relevant Level-III test is therefore whether CK reduces the integrated inconsistency present across horizons, not whether inconsistency must itself increase with $h$.

\paragraph{Predictive fidelity is largely retained.}
Consistency alone would be uninformative if both predictive paths became jointly inaccurate. We therefore independently compare the direct and composed predictions with the exact simulator law. At $h=16$, the direct Energy Score changes only from $0.54507\pm0.01003$ to $0.54442\pm0.01066$, while the composed Energy Score changes from $0.54427\pm0.00937$ to $0.54533\pm0.01105$. Both remain close to the oracle score $0.54187\pm0.00864$.

Likewise, the complete direct trajectory NLL per horizon is essentially unchanged, from $-4.19456\pm0.04725$ to $-4.19473\pm0.04524$. Direct energy distance to the true $h=16$ law changes from $0.00603\pm0.00205$ to $0.00544\pm0.00135$. The recursively composed path exhibits a small fidelity tradeoff: its terminal energy distance to the true law increases from $0.00407\pm0.00136$ to $0.00644\pm0.00157$, although its Energy Score changes by only about $0.2\%$. We therefore interpret CK as a structural consistency regularizer rather than as an unconditional improvement to every predictive metric. Full target-fidelity results are given in \cref{tab:exp4-fidelity,fig:exp4-target-ed,fig:exp4-composed-es}.

The consistency improvement is also not explained by suppressing one of the predictive branches. The effective component count remains $1.8622\pm0.0192$ without CK and $1.8681\pm0.0124$ with CK, while the mean minimum posterior responsibility is approximately $0.446$ in both cases. The terminal direct upper-mode-weight MAE changes from $0.02868\pm0.00189$ to $0.02774\pm0.00081$, while the recursively composed upper-regime event-probability MAE remains similar, changing from $0.04268\pm0.00330$ to $0.04454\pm0.00229$. Additional mixture and branch diagnostics are reported in \cref{app:exp4-details}.

\subsection{Summary of empirical findings}

Taken together, Experiments~1--4 test four complementary aspects of multimodal probabilistic world modeling: whether distinct conditional futures are represented faithfully, whether retaining the complete predictive belief changes downstream decisions, whether multimodal predictions form coherent trajectories across time, and whether the resulting dynamics remain consistent when propagated across prediction horizons.

Experiment~1 isolates the predictive-family question under a known two-branch conditional law. MDN-VJEPA reduces energy distance from $0.1913\pm0.0024$ for a parameter-matched Gaussian predictor to $0.0022\pm0.0007$, approaching the oracle predictive law. Experiment~2 then tests the decision-theoretic consequence of retaining this multimodal belief: using the complete predictive mixture rather than the conditional mean of the \emph{same} trained MDN-VJEPA model reduces collision rate from $0.5563\pm0.0095$ to $0.0025\pm0.0003$ in stochastic navigation.

Experiment~3 isolates temporal coherence. Sharing one trajectory-level mode reduces the semantic switch rate from $0.6379\pm0.0023$ to zero while preserving full marginal branch coverage, and reduces energy distance to the true trajectory law from $0.6246\pm0.0075$ to $0.0151\pm0.0076$. Finally, Experiment~4 evaluates multi-horizon compositional consistency. The Level-III Chapman--Kolmogorov regularizer reduces normalized direct--composed energy-distance AUC by approximately $15.9\%$, with improvement in all 5 seeds, while largely preserving predictive Energy Score and likelihood.

These results support the progressive picture developed throughout the paper: multimodality is needed to preserve separated futures; trajectory-level coupling is needed to organize those futures into coherent alternatives; retaining the complete predictive distribution can materially change downstream decisions; and distributional path consistency can reduce disagreement between direct and recursively composed predictive beliefs without collapsing the learned multimodal structure. More broadly, the experiments support the principle that a probabilistic JEPA should preserve not only how uncertain the future is, but which distinct futures constitute that uncertainty, how they form coherent trajectories, and how those beliefs compose through time.

\section{Discussion}
\label{sec:discussion}

\subsection{Our contribution is not ``GMM instead of Gaussian''}

A GMM predictive head is the starting point, not the complete novelty claim. MDNs predate JEPA by decades \citep{bishop1994mdn}; MDN-RNN latent world models already modeled multiple future latent outcomes \citep{ha2018worldmodels}; VJEPA explicitly proposed GMM heads as a natural extension for multimodal futures \citep{huang2026vjepa}; and Branch-JEPA provides a finite-support multimodal JEPA alternative with an explicit MDN comparison \citep{song2026branchjepa}. The research question addressed here is therefore more specific: how should a probabilistic JEPA represent continuous multimodal future beliefs, organize them into coherent trajectories, and keep those beliefs compatible under temporal composition?

The central predictive contribution is the combination
\begin{equation}
\underbrace{\text{continuous mixture prediction}}_{\text{multimodality}}
+
\underbrace{\text{shared trajectory-level modes}}_{\text{coherence}}
+
\underbrace{\text{distributional path consistency}}_{\text{composition}}.
\label{eq:novelty-synthesis}
\end{equation}
These three ingredients correspond directly to Levels~I--III of the model hierarchy. The PIB terms remain orthogonal to this predictive hierarchy: in the full representation-learning formulation they regularize how the predictive state is learned, whereas the three levels determine the structure and temporal organization of the future predictive law. The controlled experiments in this paper deliberately isolate the latter and therefore do not constitute an empirical ablation of PIB representation learning.

The Gaussian-mixture realization additionally separates uncertainty over alternative futures from uncertainty within a particular future. This is a property of the continuous mixture representation rather than of multimodality alone: a finite atomic predictor can represent several weighted alternatives, but does not parameterize local within-branch density unless an additional uncertainty model is introduced.

\subsection{Atomic branches \textit{vs.} Gaussian components}

Finite atomic branches and Gaussian mixture components emphasize different aspects of multimodal prediction:
\begin{align*}
\textit{atomic branch}
&:
\text{branch probability + representative trajectory},
\\
\textit{Gaussian component}
&:
\text{branch probability + representative trajectory + within-branch uncertainty}.
\end{align*}
Concretely, an atomic predictor represents branch $k$ as
$
\pi_k\delta_{\mu_k},
$
so that, conditional on selecting the branch, all probability mass is concentrated on the representative trajectory $\mu_k$. MDN-VJEPA instead represents the corresponding branch as
$
\pi_k\N(\mu_k,\Sigma_k),
$
so that $\mu_k$ specifies the branch location while $\Sigma_k$ describes stochastic variation around it. The distinction is therefore not whether either model can represent multiple weighted branches (both can) but whether uncertainty \emph{within} each branch is modeled explicitly.

Neither representation is uniformly preferable. Atomic support can preserve sharply separated alternatives without assigning density between them, whereas Gaussian components additionally model variation around each alternative but may merge modes, inflate covariance, or allocate probability to unsupported regions when the Gaussian approximation is poor.

Experiment~2 illustrates this distinction in a controlled setting. The Branch-JEPA-style atomic predictor and both MDN variants recover both semantic route families under the fixed support tolerance, showing that the principal difference is not branch discovery. The continuous MDN nevertheless more closely reproduces the simulator's within-branch stochastic law, yielding substantially smaller energy distance than the atomic predictor. In downstream expected-cost planning, the MDN also attains a lower collision probability than the controlled atomic predictor in this benchmark, consistent with local uncertainty around a route being decision-relevant. These results should not be interpreted as a general dominance claim over Branch-JEPA: the experiment uses a controlled finite-support predictive-family reference rather than a reproduction of the complete external system.

\subsection{Mode coherence is not global mode identity}

A trajectory-level mode makes semantic specialization possible but does not guarantee global component identity. Label switching, duplicated components, context-dependent splitting, and multiple components representing the same semantic future can all occur. We therefore distinguish \emph{trajectory coherence} from \emph{global semantic identifiability}. Coherence requires that one sampled or marginalized trajectory-level mode parameterize an internally consistent future block; it does not require component index $k$ to retain the same semantic meaning across all predictive contexts.

Experiment~3 provides evidence for the former rather than the latter. Under evaluation-only semantic matching, the Level-II predictor eliminates semantic switching while preserving the same horizon-wise branch coverage as Level~I, and the horizon-specific means belonging to each Level-II component also remain semantically consistent. This demonstrates local trajectory-level specialization in the controlled generator, but does not establish a universal semantic identity for mixture labels.

The number of active Gaussian components should likewise not be identified automatically with the number of semantically distinct futures. In general,
\begin{equation*}
\textit{number of GMM components}
\neq
\textit{number of semantic future modes}.
\end{equation*}
Several Gaussian components may cooperate to approximate one curved or strongly non-Gaussian branch, while one sufficiently broad component may cover several weakly separated outcomes. Accordingly, $K_{\mathrm{eff}}$ should primarily be interpreted as the effective complexity or utilization of the predictive mixture. Semantic interpretation requires additional evidence such as alignment with maneuver classes, event labels, or stable trajectory-level behavior.

\subsection{Consistency is not predictive accuracy}

The Level-III objective imposes a structural requirement rather than introducing another proper scoring rule. Reducing
\begin{equation*}
\mathcal E
\left(
p_\phi^{\mathrm{dir},h},
p_\phi^{\mathrm{comp},h}
\right)
\end{equation*}
only establishes that two predictive paths agree more closely; it does not by itself establish that either path is correct. The direct and one-step predictors must therefore remain supervised by the observed target law, and direct--composed discrepancy must be interpreted together with predictive fidelity.

This distinction is visible empirically. In Experiment~2, CK substantially reduces direct--composed energy distance and improves the composed-path Energy Score, while slightly worsening the direct-target Energy Score. In the more controlled Experiment~4, Level~III reduces integrated direct--composed energy-distance discrepancy across the evaluated horizons in all five seeds, while direct predictive likelihood and Energy Score remain essentially unchanged. The recursively composed path nevertheless incurs a small terminal fidelity tradeoff. CK should therefore be interpreted as a \emph{consistency regularizer}: it improves compatibility between alternative computational paths through the learned dynamics, but is not expected to improve every predictive metric simultaneously.

This also motivates the stop-gradient implementation used in Experiment~4. The separately supervised direct prediction provides a reference belief inside the CK term, while the consistency gradient acts on the recursively composed transition. This prevents the consistency objective alone from simply moving two inaccurate distributions toward one another.

\subsection{Long-horizon mixture propagation}

Nonlinear MDN composition is generally not closed within a fixed finite-GMM family. After several recursive transitions, mixture weights and component parameters depend on random intermediate states, and the resulting predictive law may become substantially more complicated than the original $K$-component parameterization. This motivates the sample-based formulation of \cref{eq:ck-loss}, which compares the resulting predictive distributions without requiring analytic closure or component-wise alignment.

The same choice introduces a computational limitation at long horizons because ancestral rollout populations must be propagated and compared. One possible extension is to periodically project the propagated distribution back into a compact mixture representation. Gaussian Mixture Approximation methods provide one possible mechanism for such projection \citep{huang2025gma}. Such compression would introduce an additional approximation problem, including the possibility of merging low-probability branches, and is therefore left outside the present core formulation.

\subsection{Limitations and scope}

The present experiments are deliberately controlled tests of the predictive mechanisms. Experiments~1, 2, and 4 operate directly on predictive latent states, while Experiment~3 uses a fixed representation map. They therefore isolate multimodal prediction, trajectory coherence, decision-theoretic use of the predictive law, and compositional consistency without confounding these questions with representation learning. The corresponding limitation is that the paper does not yet demonstrate the complete PIB--MDN-VJEPA system with a learned high-dimensional encoder on large-scale natural video or real robotic interaction. In particular, the PIB representation terms are part of the general formulation but are not empirically isolated in the controlled experiments reported here.

A second limitation is the fixed maximum component count $K$. Too few components can merge distinct futures, whereas excessive $K$ increases optimization and identifiability difficulties and does not guarantee that additional components acquire useful semantic roles. The weak utilization safeguard can discourage accidental dead components but cannot guarantee recovery of rare futures.

Third, Gaussian components may be inadequate when the conditional distribution within one semantic branch is strongly non-elliptical. Diagonal covariance is computationally attractive in high-dimensional representation spaces but ignores within-component correlations; low-rank or richer component families may be required when these dependencies are important.

Fourth, the Level-II realization assumes conditional independence across future horizons given the shared trajectory mode, apart from their common dependence on the predictive context and mode. A joint covariance or richer mode-conditioned trajectory model could capture additional residual temporal dependencies.

Fifth, sample-based CK optimization adds computational cost and Monte Carlo variability. More importantly, consistency is only desirable when both predictive paths remain faithful to data, as demonstrated by the modest fidelity tradeoffs observed in the experiments. Finally, the discrete variables used here are trajectory-level or local transition modes rather than persistent switching states. MDN-VJEPA therefore does not constitute a full switching state-space model with explicit discrete-mode dynamics.

\section{Conclusion}

We introduced MDN-VJEPA, a multimodal probabilistic JEPA that represents several plausible future latent states together with their probabilities and within-mode uncertainties. Rather than averaging separated futures into a single prediction, the model preserves them as distinct alternatives that can be carried forward into prediction and planning.

The resulting framework develops multimodal prediction along three complementary dimensions: preserving distinct future modes, organizing them into coherent trajectories, and encouraging predictive beliefs to remain consistent under temporal composition. Controlled experiments support each distinction: MDN-VJEPA closely recovers known branching laws when genuine multimodality is present, retaining the complete predictive distribution can substantially change planning decisions, trajectory-level coupling prevents incoherent switching between futures, and distributional Chapman--Kolmogorov regularization reduces disagreement between direct and recursively composed predictions.

These results suggest that multimodality in a probabilistic world model is not only about representing greater uncertainty, but about preserving the structure of that uncertainty so that it remains meaningful across time and useful for decision making. A probabilistic JEPA should preserve not only how uncertain the future is, but which distinct futures constitute that uncertainty, how they form coherent trajectories, and how those beliefs compose through time.

\bibliography{mdn_vjepa}
\bibliographystyle{tmlr}

\appendix 

\section{Symbols, Objectives, and Notation}
\label{app:notation}

This appendix collects the notation used throughout the paper and summarizes the relationship among the principal predictive distributions and training objectives. Symbols that are specific to an individual controlled experiment are defined locally in the corresponding experimental section or appendix.

\subsection{Core variables and indices}

\begin{table}[H]
\caption{Core variables and indices used throughout the paper.}
\label{tab:notation-core}
\centering
{\small
\setlength{\tabcolsep}{5pt}
\renewcommand{\arraystretch}{1.08}
\begin{tabular}{@{}p{0.22\linewidth}p{0.71\linewidth}@{}}
\toprule
Symbol & Meaning \\
\midrule
$t$
& Current time index. \\

$h$
& Prediction horizon measured from $t$; typically $h\in\{1,\ldots,H\}$. \\

$H$
& Length of the directly predicted future block or planning horizon. \\

$X_{\leq t}$
& Observed history available at the current time. \\

$X_{t+h}$
& Future observation used to construct the target representation at horizon $h$. \\

$Z_t$
& Current predictive latent state produced from the observed history. \\

$Z_{t+h}^{\mathrm T}$
& Stochastic future target representation used for predictive supervision. The superscript $\mathrm T$ denotes the target role. \\

$Y_{t,H}^{\mathrm T}$
& Complete target trajectory
$
(Z_{t+1}^{\mathrm T},\ldots,Z_{t+H}^{\mathrm T})
$.
\\

$\xi_t$
& Side information associated with the transition from $t$ to $t+1$, such as an action, elapsed time, goal, or physical constraint. \\

$\xi_{t:t+h-1}$
& Side-information sequence required to propagate from $t$ to horizon $t+h$. \\

$u\equiv u_{t:t+H-1}$
& Candidate action sequence used in the planning section; treated as the relevant part of $\xi$. \\

$d$
& Dimension of the continuous latent representation. \\

$K$
& Fixed maximum number of Gaussian-mixture components. \\

$k$
& Mixture-component index, $k\in\{1,\ldots,K\}$. \\

$\ell$
& Predictive hierarchy level,
$
\ell\in\{\mathrm I,\mathrm{II},\mathrm{III}\}
$.
\\

$B$
& Minibatch size when defining batch-level mixture-utilization quantities. \\

$i$
& Example or trajectory index within a dataset or minibatch. \\
\bottomrule
\end{tabular}
}
\end{table}

\subsection{Encoder and predictor notation}

The probabilistic JEPA formulation uses one trainable encoder parameter set $\theta$, its exponential-moving-average copy $\bar\theta$, and predictive parameters $\phi$. The same trainable encoder is evaluated in different roles; these roles should not be interpreted as separate independently parameterized encoders.

\begin{table}[H]
\caption{Encoder, predictor, and distribution notation.}
\label{tab:notation-model}
\centering
{\small
\setlength{\tabcolsep}{5pt}
\renewcommand{\arraystretch}{1.08}
\begin{tabular}{@{}p{0.30\linewidth}p{0.63\linewidth}@{}}
\toprule
Notation & Meaning \\
\midrule
$q_\theta(Z_t\mid X_{\leq t})$
& Trainable online context path producing the current predictive state. \\

$q_\theta^{\mathrm T}(Z_{t+h}\mid X_{t+h})$
& Trainable target-view evaluation used for target-side representation regularization. The superscript $\mathrm T$ denotes the target-view role, not a separate parameter set. \\

$q_{\bar\theta}(Z_{t+h}\mid X_{t+h})$
& EMA target encoder producing the stochastic target representation used for predictive supervision. \\

$\theta$
& Trainable encoder parameters. \\

$\bar\theta$
& EMA copy of $\theta$ used by the target encoder. \\

$\phi$
& Predictive parameters; when both direct and recursive heads are present, $\phi$ collectively denotes their predictive parameters. \\

$p_\phi^{(1)}$
& Learned one-step transition kernel. The superscript $(1)$ means \emph{one transition step}. \\

$p_\phi^{\mathrm I}$
& Level-I trajectory law: a product of horizon-wise mixtures. The Roman superscript $\mathrm I$ denotes the predictive hierarchy level and is distinct from $(1)$. \\

$p_\phi^{\mathrm{II}}$
& Level-II trajectory law: a mixture of trajectory products sharing one trajectory-level mode. \\

$p_\phi^{\mathrm{dir},h}$
& Horizon-$h$ distribution obtained from the direct multi-horizon predictor. \\

$p_\phi^{\mathrm{comp},h}$
& Horizon-$h$ distribution obtained by recursively composing the learned one-step transition. \\

$p_h^\star$
& True horizon-$h$ predictive law when available in a controlled setting. \\
\bottomrule
\end{tabular}
}
\end{table}

For clarity, the principal superscripts therefore encode different concepts:
\begin{equation*}
\begin{aligned}
\mathrm T
&:
&&
\text{target representation or target-view role},
\\
(1)
&:
&&
\text{one-step transition},
\\
\mathrm I,\mathrm{II},\mathrm{III}
&:
&&
\text{predictive hierarchy levels},
\\
\mathrm{dir},\mathrm{comp}
&:
&&
\text{direct and recursively composed prediction paths}.
\end{aligned}
\end{equation*}

\subsection{Mixture and latent-mode notation}

For a Gaussian-mixture predictive distribution, component $k$ is parameterized by
\begin{equation*}
\left(
\pi_{\phi,k},
\mu_{\phi,k,h},
\Sigma_{\phi,k,h}
\right),
\end{equation*}
with the dependence on the predictive context suppressed when unambiguous.

\begin{table}[H]
\caption{Mixture and uncertainty notation.}
\label{tab:notation-mixture}
\centering
{\small
\setlength{\tabcolsep}{5pt}
\renewcommand{\arraystretch}{1.08}
\begin{tabular}{@{}p{0.24\linewidth}p{0.69\linewidth}@{}}
\toprule
Symbol & Meaning \\
\midrule
$\pi_{\phi,k}$
& Predictive probability assigned to mixture component $k$, before observing the target. \\

$\boldsymbol{\pi}_\phi$
& Vector of mixture probabilities
$
(\pi_{\phi,1},\ldots,\pi_{\phi,K})
$.
\\

$\mu_{\phi,k,h}$
& Mean of component $k$ at prediction horizon $h$. \\

$\Sigma_{\phi,k,h}$
& Covariance of component $k$ at horizon $h$; represents within-mode predictive variation. \\

$\sigma_{\min}$
& Positive standard-deviation floor used to prevent pathological variance collapse. \\

$M$
& Level-II trajectory-level categorical mode shared across the complete directly predicted future block. \\

$M_{t+1}$
& Local categorical mode associated with one stochastic transition of the recursive one-step MDN kernel. \\

$\bar\mu_h$
& Mixture predictive mean at horizon $h$,
$
\bar\mu_h=\sum_k\pi_k\mu_{k,h}
$.
\\

$r_{ik}$
& Posterior responsibility of component $k$ for target trajectory $i$ after that target has been observed. \\

$\bar r_k$
& Average posterior responsibility of component $k$ over a minibatch. \\

$K_{\mathrm{eff}}$
& Entropy effective component count,
$
\exp(H(\boldsymbol{\pi}))
$; a mixture-utilization or effective-complexity diagnostic rather than a semantic mode count. \\

$\rho_{\min}$
& Small minimum-responsibility threshold used by the weak dead-component safeguard. \\

$[x]_+$
& Positive-part operator,
$
[x]_+=\max(x,0)
$.
\\
\bottomrule
\end{tabular}
}
\end{table}

The trajectory-level variable $M$ and the one-step variable $M_{t+1}$ have related but distinct meanings:
\begin{equation*}
\begin{aligned}
M
&:
&&
\text{one mode shared by a complete \emph{directly predicted} future block},
\\
M_{t+1}
&:
&&
\text{one local mixture choice in the recursive transition kernel}.
\end{aligned}
\end{equation*}
Recursive rollout does not assume that the numerical component label persists across time, and the model does not introduce a separate discrete transition law $p(M_{t+1}\mid M_t)$.

\subsection{Three predictive levels}

The predictive hierarchy can be summarized compactly as follows.

\paragraph{Level I: horizon-wise multimodality.}
Each horizon has a multimodal marginal, while the complete future factorizes as a product of mixtures:
\begin{equation*}
p_\phi^{\mathrm I}
\left(
Z_{t+1:t+H}
\mid
Z_t,\xi_{t:t+H-1}
\right)
=
\prod_{h=1}^{H}
\left[
\sum_{k=1}^{K}
\pi_{\phi,k,h}
\N
\left(
Z_{t+h};
\mu_{\phi,k,h},
\Sigma_{\phi,k,h}
\right)
\right].
\end{equation*}
The mixture identity may therefore vary independently across prediction horizons.

\paragraph{Level II: trajectory coherence.}
A single trajectory-level mode is shared across all horizons:
\begin{equation*}
p_\phi^{\mathrm{II}}
\left(
Z_{t+1:t+H}
\mid
Z_t,\xi_{t:t+H-1}
\right)
=
\sum_{k=1}^{K}
\pi_{\phi,k}
\prod_{h=1}^{H}
\N
\left(
Z_{t+h};
\mu_{\phi,k,h},
\Sigma_{\phi,k,h}
\right).
\end{equation*}
Thus, Level II replaces a product of mixtures by a mixture of trajectory products.

\paragraph{Level III: compositional consistency.}
Level III retains the Level-II direct trajectory law and additionally instantiates a recursively composable one-step kernel. It encourages
\begin{equation*}
p_\phi^{\mathrm{dir},h}
\approx
p_\phi^{\mathrm{comp},h},
\qquad
h>1,
\end{equation*}
while both paths remain supervised for predictive fidelity.

The progression is therefore
\begin{equation*}
\underbrace{\text{multimodal marginals}}_{\mathrm{Level~I}}
\quad\longrightarrow\quad
\underbrace{\text{coherent multimodal trajectories}}_{\mathrm{Level~II}}
\quad\longrightarrow\quad
\underbrace{\text{compositionally consistent dynamics}}_{\mathrm{Level~III}}.
\end{equation*}

\subsection{Predictive uncertainty}

At a fixed horizon $h$, the Level-II predictive mean is
\begin{equation*}
\bar\mu_h
=
\sum_{k=1}^{K}
\pi_k\mu_{k,h}.
\end{equation*}
The corresponding conditional covariance decomposes as
\begin{equation*}
\operatorname{Cov}
\left(
Z_{t+h}
\mid
Z_t,\xi
\right)
=
\underbrace{
\sum_{k=1}^{K}
\pi_k\Sigma_{k,h}
}_{\text{within-mode uncertainty}}
+
\underbrace{
\sum_{k=1}^{K}
\pi_k
(\mu_{k,h}-\bar\mu_h)
(\mu_{k,h}-\bar\mu_h)^\top
}_{\text{between-mode / branching uncertainty}}.
\end{equation*}
This is a decomposition of \emph{conditional predictive uncertainty}; it is not, by itself, an aleatoric--epistemic decomposition.

\subsection{Training objectives}

The complete training objective is
\begin{equation}
\begin{aligned}
\cL_{\mathrm{MDN\text{-}VJEPA}}^{(\ell)}
={}&
\cL_{\mathrm{dir}}^{(\ell)}
+
\lambda_{\mathrm{step}}\cL_{\mathrm{step}}
+
\beta_T\cL_T
+
\beta_C\cL_C
\\
&+
\lambda_{\mathrm{use}}\cR_{\mathrm{use}}
+
\mathbf{1}\{\ell=\mathrm{III}\}
\lambda_{\mathrm{CK}}\cL_{\mathrm{CK}},
\qquad
\ell\in\{\mathrm I,\mathrm{II},\mathrm{III}\}.
\end{aligned}
\label{eq:app-full-objective} \tag{Eq.\ref{eq:full-objective}}
\end{equation}
This restates \cref{eq:full-objective} for reference. The individual terms are summarized in \cref{tab:notation-objectives}.

\begin{table}[H]
\caption{Training objectives and their roles.}
\label{tab:notation-objectives}
\centering
{\small
\setlength{\tabcolsep}{4pt}
\renewcommand{\arraystretch}{1.10}
\begin{tabular}{@{}p{0.22\linewidth}p{0.27\linewidth}p{0.44\linewidth}@{}}
\toprule
Term & Main reference & Role \\
\midrule
$\cL_{\mathrm{dir}}^{\mathrm I}$
& \cref{eq:level1-nll}
& Level-I horizon-wise multimodal negative log-likelihood. \\

$\cL_{\mathrm{dir}}^{\mathrm{II}}$
& \cref{eq:mdn-nll}
& Level-II shared-mode trajectory negative log-likelihood. \\

$\cL_{\mathrm{dir}}^{\mathrm{III}}$
& $\cL_{\mathrm{dir}}^{\mathrm{III}}
\equiv
\cL_{\mathrm{dir}}^{\mathrm{II}}$
& Direct Level-III prediction retains the Level-II coherent trajectory law. \\

$\cL_{\mathrm{step}}$
& \cref{eq:step-nll}
& Fits the one-step transition used for recursive rollout. \\

$\cL_T$
& \cref{eq:target-kl}
& Regularizes the stochastic target-side representation through the trainable target-view encoder path. \\

$\cL_C$
& \cref{eq:context-kl}
& Compresses unnecessary observed-history information in the current predictive state. \\

$\cR_{\mathrm{use}}$
& \cref{eq:usage-reg}
& Weak fixed-$K$ safeguard against persistently dead mixture components. \\

$\cL_{\mathrm{CK}}$
& \cref{eq:ck-loss}
& Encourages distributional agreement between direct and recursively composed predictive laws. \\
\bottomrule
\end{tabular}
}
\end{table}

The principal objective coefficients are:

\begin{table}[H]
\caption{Loss coefficients and architectural safeguards.}
\label{tab:notation-coefficients}
\centering
{\small
\setlength{\tabcolsep}{5pt}
\renewcommand{\arraystretch}{1.08}
\begin{tabular}{@{}p{0.23\linewidth}p{0.70\linewidth}@{}}
\toprule
Symbol & Meaning \\
\midrule
$\lambda_{\mathrm{step}}$
& Weight on supervised one-step transition learning. It is required when the recursive Level-III path is instantiated and may be zero when Levels~I--II use no recursive transition. \\

$\beta_T$
& Weight on target-side representation regularization $\cL_T$. \\

$\beta_C$
& Weight on context-side information compression $\cL_C$; $\beta_C>0$ activates the explicit PIB context bottleneck. \\

$\lambda_{\mathrm{use}}$
& Weight on the weak component-utilization safeguard. \\

$\lambda_{\mathrm{CK}}$
& Weight on the Level-III distributional consistency objective. \\

$w_h$
& Relative weight assigned to horizon $h$ inside the multi-horizon CK objective. \\

$\mathcal H$
& Set of horizons at which direct--composed consistency is evaluated. \\

$\sigma_{\min}$
& Architectural standard-deviation floor; unlike the quantities above, it is not an additional loss term. \\
\bottomrule
\end{tabular}
}
\end{table}

The representation terms $\cL_T$ and $\cL_C$ and the utilization safeguard $\cR_{\mathrm{use}}$ are orthogonal to the Level-I--III predictive hierarchy. Likewise, the covariance floor $\sigma_{\min}$ is an architectural safeguard rather than a loss, and $\bar\theta$ is updated by EMA rather than by direct gradient descent.

\subsection{Direct--composed consistency notation}

Let
\begin{equation*}
\mathcal E(P,Q)
\end{equation*}
denote the energy distance between two predictive distributions. The population Level-III objective is (Eq.\ref{eq:ck-loss})
\begin{equation*}
\cL_{\mathrm{CK}}
=
\sum_{\substack{h\in\mathcal H\\h>1}}
w_h
\,
\mathcal E
\left(
p_\phi^{\mathrm{dir},h},
p_\phi^{\mathrm{comp},h}
\right).
\end{equation*}

For nonlinear MDN transitions, $p_\phi^{\mathrm{comp},h}$ is generally not available as a finite closed-form $K$-component GMM. It is therefore represented by ancestral samples. Sample-based energy distance is permutation invariant with respect to mixture-component labels, so direct and composed components need not be aligned before comparison.

In the Level-III implementation used in Experiment~4, direct samples serve as a stop-gradient reference within the CK term, while the recursive samples remain differentiable with respect to the one-step transition. The direct predictor continues to receive its ordinary predictive-likelihood gradients.

\subsection{Planning notation}

\begin{table}[H]
\caption{Notation used for distribution-aware planning.}
\label{tab:notation-planning}
\centering
{\small
\setlength{\tabcolsep}{5pt}
\renewcommand{\arraystretch}{1.08}
\begin{tabular}{@{}p{0.27\linewidth}p{0.66\linewidth}@{}}
\toprule
Symbol & Meaning \\
\midrule
$u=u_{t:t+H-1}$
& Candidate future action sequence. \\

$C(Z_{t+1:t+H},u)$
& Trajectory-level cost associated with a predicted future and candidate action sequence. \\

$\bar Z_{t+1:t+H}(u)$
& Conditional mean future trajectory under the predictive law. \\

$J_{\mathrm{mean}}(u)$
& Mean-only planning objective,
$
C(\bar Z_{t+1:t+H}(u),u)
$.
\\

$J_{\mathrm{dist}}(u)$
& Distribution-aware objective,
$
\E[C(Z_{t+1:t+H},u)\mid Z_t,u]
$.
\\

$\mathcal F$
& Failure event defined over a predicted trajectory, such as collision. \\

$\Pr(\mathcal F\mid Z_t,u)$
& Predicted probability of the failure event under the complete predictive law. \\

$\operatorname{CVaR}_{\alpha}$
& Conditional value-at-risk at tail level $\alpha$. \\

$\delta$
& Chance-constraint tolerance in
$
\Pr(\mathcal F\mid Z_t,u)\leq\delta
$.
\\
\bottomrule
\end{tabular}
}
\end{table}

Experiment-specific notation may reuse a symbol locally when its meaning is explicitly redefined; for example, the branch-separation parameter in Experiment~1 is local to that experiment.

\subsection{Evaluation terminology}

For completeness, \cref{tab:notation-metrics} summarizes the probabilistic diagnostics used in the experiments.

\begin{table}[H]
\caption{Principal evaluation quantities.}
\label{tab:notation-metrics}
\centering
{\small
\setlength{\tabcolsep}{5pt}
\renewcommand{\arraystretch}{1.08}
\begin{tabular}{@{}p{0.22\linewidth}p{0.71\linewidth}@{}}
\toprule
Quantity & Interpretation \\
\midrule
NLL
& Negative log-likelihood of a realized target under a normalized predictive density. \\

Energy Score (ES)
& Proper sample-based score comparing a predictive distribution with a realized target. \\

Energy distance (ED)
& Sample-based discrepancy between two complete probability distributions. \\

MMD
& Maximum mean discrepancy, used as an auxiliary sample-based distributional discrepancy. \\

Brier score
& Proper score for predicted probabilities of binary events such as branch identity or failure. \\

ECE
& Expected calibration error for predicted branch or failure probabilities. \\

$K_{\mathrm{eff}}$
& Effective mixture complexity derived from predictive weights; used as a component-utilization diagnostic. \\

$\bar r_k$
& Mean posterior responsibility of component $k$; used to diagnose whether components receive target assignments. \\

Best-of-$N$
& Fixed-sample coverage diagnostic based on the closest among $N$ predictive samples; not used as a primary proper probabilistic score. \\
\bottomrule
\end{tabular}
}
\end{table}

Two distinctions are particularly important throughout the paper. First, the predictive mixture weight $\pi_{ik}$ is available \emph{before} observing target $i$, whereas the posterior responsibility $r_{ik}$ is computed \emph{after} observing that target. Second, Energy Score compares a predictive distribution with an observed outcome, whereas energy distance compares two complete distributions.

\section{Marginalization of the trajectory-level mixture}
\label{app:level2-marginalization}

This appendix derives the horizon-wise marginal of the Level-II trajectory distribution in \cref{eq:trajectory-mixture}. The derivation shows that sharing one trajectory-level mode across the future block preserves the Gaussian-mixture form at every individual horizon, and therefore preserves the capacity for horizon-wise multimodality while changing the joint coupling across time.

For compactness, define the predictive conditioning information
\[
C_t
\coloneqq
\left(
Z_t,\xi_{t:t+H-1}
\right).
\]
Restating the Level-II trajectory distribution from \cref{eq:trajectory-mixture},
\begin{equation}
p_\phi^{\mathrm{II}}
\left(
Z_{t+1:t+H}\mid C_t
\right)
=
\sum_{k=1}^{K}
\pi_{\phi,k}
\prod_{j=1}^{H}
\N\left(
Z_{t+j};
\mu_{\phi,k,j},
\Sigma_{\phi,k,j}
\right),
\tag{\ref{eq:trajectory-mixture}}
\end{equation}
where the dependence of
$\pi_{\phi,k}$,
$\mu_{\phi,k,j}$, and
$\Sigma_{\phi,k,j}$
on $C_t$ is suppressed for readability.

\paragraph{Direct marginalization.}
Fix a horizon
$
h\in\{1,\ldots,H\}
$.
To obtain the predictive distribution of $Z_{t+h}$, we integrate out all remaining future states:
\[
p_\phi^{\mathrm{II}}
\left(
Z_{t+h}\mid C_t
\right)
=
\int
p_\phi^{\mathrm{II}}
\left(
Z_{t+1:t+H}\mid C_t
\right)
\prod_{\substack{j=1\\j\neq h}}^{H}
dZ_{t+j}.
\]
Substituting \cref{eq:trajectory-mixture} gives
\begin{align*}
p_\phi^{\mathrm{II}}
\left(
Z_{t+h}\mid C_t
\right)
&=
\int
\left[
\sum_{k=1}^{K}
\pi_{\phi,k}
\prod_{j=1}^{H}
\N\left(
Z_{t+j};
\mu_{\phi,k,j},
\Sigma_{\phi,k,j}
\right)
\right]
\prod_{\substack{j=1\\j\neq h}}^{H}
dZ_{t+j}
\\
&=
\sum_{k=1}^{K}
\pi_{\phi,k}
\int
\prod_{j=1}^{H}
\N\left(
Z_{t+j};
\mu_{\phi,k,j},
\Sigma_{\phi,k,j}
\right)
\prod_{\substack{j=1\\j\neq h}}^{H}
dZ_{t+j},
\end{align*}
where the finite sum has been exchanged with the integral.

Under the Level-II factorization in \cref{eq:trajectory-mixture}, the horizon-$h$ factor can be separated:
\begin{align*}
p_\phi^{\mathrm{II}}
\left(
Z_{t+h}\mid C_t
\right)
&=
\sum_{k=1}^{K}
\pi_{\phi,k}
\N\left(
Z_{t+h};
\mu_{\phi,k,h},
\Sigma_{\phi,k,h}
\right)
\\
&\qquad\times
\prod_{\substack{j=1\\j\neq h}}^{H}
\int
\N\left(
Z_{t+j};
\mu_{\phi,k,j},
\Sigma_{\phi,k,j}
\right)
dZ_{t+j}.
\end{align*}
Every Gaussian factor is a normalized probability density, so
\[
\int
\N\left(
Z_{t+j};
\mu_{\phi,k,j},
\Sigma_{\phi,k,j}
\right)
dZ_{t+j}
=
1.
\]
Consequently,
\begin{equation}
p_\phi^{\mathrm{II}}
\left(
Z_{t+h}\mid C_t
\right)
=
\sum_{k=1}^{K}
\pi_{\phi,k}
\N\left(
Z_{t+h};
\mu_{\phi,k,h},
\Sigma_{\phi,k,h}
\right),
\tag{\ref{eq:level2-marginal}}
\end{equation}
which is the horizon-wise Gaussian-mixture marginal in \cref{eq:level2-marginal}. Thus, trajectory-level coupling does not collapse the marginal to a single component: the horizon-$h$ law retains the same $K$-term mixture representation and therefore retains the capacity to express separated predictive modes.

\paragraph{Latent-variable derivation.}
The same result follows directly from the trajectory-level latent variable. Restating the shared-mode construction from \cref{eq:traj-mode},
\begin{equation}
M
\sim
\Cat
\left(
\pi_{\phi,1},\ldots,\pi_{\phi,K}
\right),
\tag{\ref{eq:traj-mode}}
\end{equation}
with
\[
p_\phi
\left(
Z_{t+h}\mid M=k,C_t
\right)
=
\N\left(
Z_{t+h};
\mu_{\phi,k,h},
\Sigma_{\phi,k,h}
\right).
\]
Marginalizing $M$ by the law of total probability gives
\begin{align*}
p_\phi^{\mathrm{II}}
\left(
Z_{t+h}\mid C_t
\right)
&=
\sum_{k=1}^{K}
p_\phi
\left(
M=k\mid C_t
\right)
p_\phi
\left(
Z_{t+h}\mid M=k,C_t
\right)
\\
&=
\sum_{k=1}^{K}
\pi_{\phi,k}
\N\left(
Z_{t+h};
\mu_{\phi,k,h},
\Sigma_{\phi,k,h}
\right).
\end{align*}
This derivation makes the mixture semantics explicit. Although the latent variable $M$ is marginalized out in the horizon-wise predictive law, the mixture structure induced by its possible values remains. Sharing $M$ across the complete future block therefore changes how predictions are coupled jointly without removing their Gaussian-mixture form at an individual horizon.

\paragraph{Level I \textit{vs.} Level II.}
The distinction between Levels~I and II is therefore not necessarily visible in any individual horizon-wise marginal; it lies in the joint coupling of those marginals across the future block. Restating the Level-I product-of-mixtures construction from \cref{eq:level1-trajectory},
\begin{equation}
p_\phi^{\mathrm I}
\left(
Z_{t+1:t+H}\mid C_t
\right)
=
\prod_{j=1}^{H}
\left[
\sum_{k=1}^{K}
\pi_{\phi,k,j}
\N\left(
Z_{t+j};
\mu_{\phi,k,j},
\Sigma_{\phi,k,j}
\right)
\right],
\tag{\ref{eq:level1-trajectory}}
\end{equation}
so the mixture identity can be selected independently at different horizons. By contrast, Level II uses the mixture-of-products construction in \cref{eq:trajectory-mixture},
\begin{equation}
p_\phi^{\mathrm{II}}
\left(
Z_{t+1:t+H}\mid C_t
\right)
=
\sum_{k=1}^{K}
\pi_{\phi,k}
\left[
\prod_{j=1}^{H}
\N\left(
Z_{t+j};
\mu_{\phi,k,j},
\Sigma_{\phi,k,j}
\right)
\right],
\tag{\ref{eq:trajectory-mixture}}
\end{equation}
so one shared latent value $M=k$ couples all prediction horizons.

The distinction is especially clear in the matched-marginal case. Suppose that, at every horizon $h$,
\[
\pi_{\phi,k,h}
=
\pi_{\phi,k},
\qquad
\mu_{\phi,k,h}^{\mathrm I}
=
\mu_{\phi,k,h}^{\mathrm{II}},
\qquad
\Sigma_{\phi,k,h}^{\mathrm I}
=
\Sigma_{\phi,k,h}^{\mathrm{II}}
\]
for all $k$. Then Levels~I and II have exactly the same marginal distribution at every horizon,
\[
p_\phi^{\mathrm I}
\left(
Z_{t+h}\mid C_t
\right)
=
p_\phi^{\mathrm{II}}
\left(
Z_{t+h}\mid C_t
\right),
\qquad
h=1,\ldots,H,
\]
while their joint trajectory distributions generally remain different:
\[
p_\phi^{\mathrm I}
\left(
Z_{t+1:t+H}\mid C_t
\right)
\neq
p_\phi^{\mathrm{II}}
\left(
Z_{t+1:t+H}\mid C_t
\right).
\]
The reason is structural: Level I independently mixes at each horizon, whereas Level II marginalizes a single shared trajectory-level mode. Hence,
\[
\text{identical horizon-wise predictive marginals}
\;\not\Rightarrow\;
\text{identical joint trajectory distributions}.
\]

This distinction is precisely what Experiment~3 isolates. Level I and Level II can preserve the same relevant alternatives at individual horizons, yet the shared trajectory-level mode in Level II determines whether those alternatives combine into coherent complete futures.

\paragraph{Extension beyond conditionally factorized horizons.}
The derivation above follows the conditionally factorized Level-II realization in \cref{eq:trajectory-mixture}. If the mode-conditioned future is instead represented by a joint Gaussian over the concatenated trajectory,
\[
Z_{t+1:t+H}
\mid
M=k,C_t
\sim
\N
\left(
\mu_{k,1:H},
\Sigma_{k,1:H}
\right),
\]
the same qualitative conclusion holds. Marginalizing a joint Gaussian to horizon $h$ yields the corresponding Gaussian marginal, and marginalizing over $M$ therefore again produces a Gaussian mixture with the shared trajectory-level weights $\pi_{\phi,k}$. Thus, preservation of the horizon-wise mixture form is not restricted to the conditionally independent within-mode realization.

\section{Additional Experimental Details and Results}
\label{app:experiments}

\paragraph{Compute environment.}
All four experiments were run in Google Colab \citep{colab2024}. The recorded host configuration used an x86\_64 Intel Xeon CPU at 2.00\,GHz with one physical core and two logical CPUs, approximately 13.6\,GB of system memory, and approximately 121\,GB of available disk space. The CPU supported AVX2 and AVX-512 instructions and was hosted under KVM virtualization. The experiment scripts selected CUDA when available and otherwise used the CPU.

\subsection{Experiment 1: controlled branching dynamics}
\label{app:exp1-details}

\paragraph{Protocol and implementation.}
Experiment~1 is a controlled latent-space test of the predictive family rather than a test of representation learning. For every seed, train, validation, and test contexts are generated independently, and the hidden branch realization is sampled only after the observed context has been generated. The test split is used only for final evaluation. Semantic left/right matching uses the known synthetic branch prototypes only after prediction, so it cannot leak future information into optimization or branch selection.

The primary predictor is an MLP with two hidden layers and SiLU activations. The deterministic model is trained by squared error, whereas Gaussian VJEPA and MDN-VJEPA are trained by negative log-likelihood with a diagonal covariance and variance floor. Early stopping uses the validation predictive loss. The fixed-$K$ usage regularizer is intentionally disabled in the primary experiment so that successful use of both components is not imposed by the objective. The main settings are summarized in \cref{tab:exp1-hparams}.

\begin{table}[H]
\caption{Experiment~1 implementation settings.}
\label{tab:exp1-hparams}
\centering
{\scriptsize
\setlength{\tabcolsep}{4pt}
\renewcommand{\arraystretch}{0.95}
\begin{tabular}{@{}ll@{}}
\toprule
Setting & Value \\
\midrule
Seeds & $\{0,1,2,3,4\}$ \\
Train / validation / test size & $6000/1500/2000$ \\
Latent dimension & $2$ \\
Prediction horizon & $H=1$ \\
MDN components & $K=2$ \\
Base branch separation & $\delta=1.25$ \\
Base within-mode standard deviation & $\sigma_z=0.12$ \\
Base branch-logit bias & $b=0$ \\
MLP hidden width & $96$ \\
Gaussian standard-deviation floor & $\sigma_{\min}=10^{-3}$ \\
Maximum epochs / patience & $180/20$ \\
Batch size & $256$ \\
Optimizer & AdamW \\
Learning rate / weight decay & $2\times10^{-3}/10^{-4}$ \\
Usage penalty & $\lambda_{\mathrm{use}}=0$ \\
Evaluation samples per conditional law & $128$ \\
Matched-Gaussian hidden width & $99$ \\
MDN / matched-Gaussian parameters & $10{,}570/10{,}597$ \\
\bottomrule
\end{tabular}
}
\end{table}

\paragraph{Full base comparison and capacity control.}
For completeness, \cref{tab:exp1-base-full} includes both the standard Gaussian predictor and the parameter-matched Gaussian control. Increasing Gaussian capacity from the standard predictor to the matched predictor does not materially change the result: their energy distances are $0.1922\pm0.0034$ and $0.1913\pm0.0024$, respectively. By contrast, MDN-VJEPA reaches $0.0022\pm0.0007$. Across the 5 paired seeds, the matched-Gaussian-minus-MDN energy-distance gap is $0.1891\pm0.0024$, and the corresponding NLL gap is $1.6192\pm0.0090$.

\begin{table}[H]
\caption{Full base-condition comparison. Values are mean $\pm$ standard deviation over 5 seeds.}
\label{tab:exp1-base-full}
\centering
{\scriptsize
\setlength{\tabcolsep}{3.5pt}
\renewcommand{\arraystretch}{0.95}
\begin{tabular}{@{}lccc@{}}
\toprule
Model & Energy distance $\downarrow$ & Energy score $\downarrow$ & NLL $\downarrow$ \\
\midrule
Deterministic JEPA & $0.8685\pm0.0128$ & $0.9894\pm0.0141$ & -- \\
Gaussian VJEPA & $0.1922\pm0.0034$ & $0.6508\pm0.0074$ & $0.8248\pm0.0466$ \\
Gaussian VJEPA (matched) & $0.1913\pm0.0024$ & $0.6493\pm0.0078$ & $0.8259\pm0.0446$ \\
MDN-VJEPA & $0.0022\pm0.0007$ & $0.5554\pm0.0087$ & $-0.7934\pm0.0464$ \\
Generator oracle & -- & $0.5541\pm0.0078$ & $-0.8164\pm0.0479$ \\
\bottomrule
\end{tabular}
}
\end{table}

\paragraph{Mixture recovery diagnostics.}
\Cref{tab:exp1-mdn-diagnostics} reports diagnostics that use the known synthetic law. The probability error and calibration error are small, both terms of the covariance decomposition are recovered to within a few percent on average, and the effective component count closely matches the true predictive-mixture complexity. The near-balanced responsibilities are expected in the base regime because the global branch-logit bias is zero; importantly, this behavior is obtained with $\lambda_{\mathrm{use}}=0$.

\begin{table}[H]
\caption{MDN-VJEPA recovery diagnostics in the base condition. Values are mean $\pm$ standard deviation over 5 seeds.}
\label{tab:exp1-mdn-diagnostics}
\centering
{\scriptsize
\setlength{\tabcolsep}{5pt}
\renewcommand{\arraystretch}{0.95}
\begin{tabular}{@{}lr@{}}
\toprule
Diagnostic & Value \\
\midrule
Excess NLL over oracle & $0.0231\pm0.0075$ \\
Branch-probability MAE & $0.0115\pm0.0025$ \\
Branch ECE & $0.0271\pm0.0058$ \\
Within-covariance relative error & $5.71\%\pm4.12\%$ \\
Between-covariance relative error & $3.69\%\pm1.04\%$ \\
Effective component count & $1.7873\pm0.0079$ \\
Effective-$K$ absolute error & $0.0048\pm0.0062$ \\
Minimum average responsibility & $0.4882\pm0.0074$ \\
Maximum average responsibility & $0.5118\pm0.0074$ \\
\bottomrule
\end{tabular}
}
\end{table}

\paragraph{Exact separation sweep.}
The branch-separation values underlying \cref{fig:exp1-separation} are given in \cref{tab:exp1-separation-values}. When $\delta\in\{0.05,0.10\}$, the components overlap strongly and the decomposition into individually meaningful branches is weakly identifiable even though the predictive density can be fitted accurately. This illustrates an important distinction: density recovery does not by itself guarantee semantic identification of individual mixture components when those components are nearly coincident.

\begin{table}[H]
\caption{Energy distance to the true law under increasing branch separation. Values are mean $\pm$ standard deviation over 5 seeds.}
\label{tab:exp1-separation-values}
\centering
{\scriptsize
\setlength{\tabcolsep}{5pt}
\renewcommand{\arraystretch}{0.95}
\begin{tabular}{@{}ccc@{}}
\toprule
$\delta$ & Gaussian VJEPA (matched) & MDN-VJEPA \\
\midrule
$0.05$ & $0.0007\pm0.0004$ & $0.0008\pm0.0004$ \\
$0.10$ & $0.0011\pm0.0006$ & $0.0015\pm0.0008$ \\
$0.25$ & $0.0069\pm0.0007$ & $0.0011\pm0.0003$ \\
$0.50$ & $0.0391\pm0.0008$ & $0.0010\pm0.0005$ \\
$1.00$ & $0.1379\pm0.0031$ & $0.0012\pm0.0003$ \\
$1.25$ & $0.1913\pm0.0024$ & $0.0022\pm0.0007$ \\
\bottomrule
\end{tabular}
}
\end{table}

\paragraph{Within-mode noise and branch imbalance.}
The remaining one-factor-at-a-time sweeps test complementary ways in which the target law moves toward or away from clear multimodality. Increasing $\sigma_z$ makes the two branches overlap more strongly, so the matched Gaussian improves from energy distance $0.2205\pm0.0049$ at $\sigma_z=0.08$ to $0.0144\pm0.0015$ at $\sigma_z=0.80$; MDN-VJEPA remains close to the true law throughout. Likewise, the Gaussian approximation is worst near a balanced mixture and improves when one branch becomes dominant, whereas MDN-VJEPA remains near $2\times10^{-3}$ across the tested imbalance range. Exact values are reported in \cref{tab:exp1-secondary-sweeps} and visualized in \cref{fig:exp1-secondary-sweeps}.

\begin{table}[H]
\caption{Secondary robustness sweeps. Each entry is energy distance to the true conditional law, reported as mean $\pm$ standard deviation over 5 seeds.}
\label{tab:exp1-secondary-sweeps}
\centering
{\scriptsize
\setlength{\tabcolsep}{3pt}
\renewcommand{\arraystretch}{0.95}
\begin{tabular}{@{}ccc@{\hspace{1.2em}}ccc@{}}
\toprule
$\sigma_z$ & Gaussian & MDN & $b$ & Gaussian & MDN \\
\midrule
$0.08$ & $0.2205\pm0.0049$ & $0.0026\pm0.0018$ & $-2$ & $0.1537\pm0.0104$ & $0.0017\pm0.0013$ \\
$0.12$ & $0.1913\pm0.0024$ & $0.0022\pm0.0007$ & $-1$ & $0.1807\pm0.0068$ & $0.0020\pm0.0004$ \\
$0.25$ & $0.1201\pm0.0021$ & $0.0020\pm0.0004$ & $0$ & $0.1913\pm0.0024$ & $0.0022\pm0.0007$ \\
$0.50$ & $0.0454\pm0.0016$ & $0.0026\pm0.0007$ & $1$ & $0.1819\pm0.0017$ & $0.0025\pm0.0016$ \\
$0.80$ & $0.0144\pm0.0015$ & $0.0034\pm0.0013$ & $2$ & $0.1613\pm0.0097$ & $0.0019\pm0.0014$ \\
\bottomrule
\end{tabular}
}
\end{table}

\begin{figure}[H]
\centering
\includegraphics[width=\linewidth]{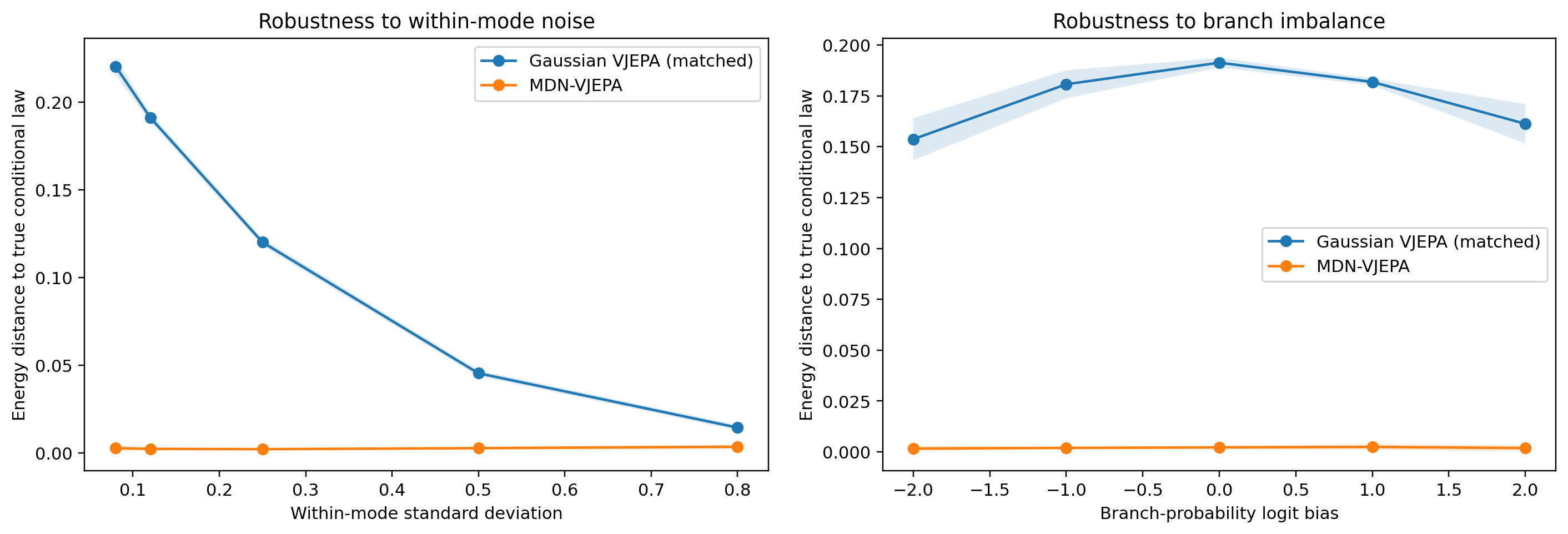}
\caption{Secondary robustness tests. Left: increasing within-mode noise progressively blurs the two modes, making a single Gaussian increasingly adequate. Right: the Gaussian approximation is most challenged near balanced branch probabilities and improves as one branch dominates. MDN-VJEPA remains close to the known law throughout both sweeps.}
\label{fig:exp1-secondary-sweeps}
\end{figure}

\paragraph{Calibration.}
\Cref{fig:exp1-calibration} shows the branch-probability reliability diagram for the pre-specified representative seed. The aggregate 5-seed ECE is $0.0271\pm0.0058$. The calibration plot is diagnostic only: branch labels and ground-truth probabilities are available because this is a synthetic controlled law and are never used by the predictor during training.

\begin{figure}[H]
\centering
\includegraphics[width=0.38\linewidth]{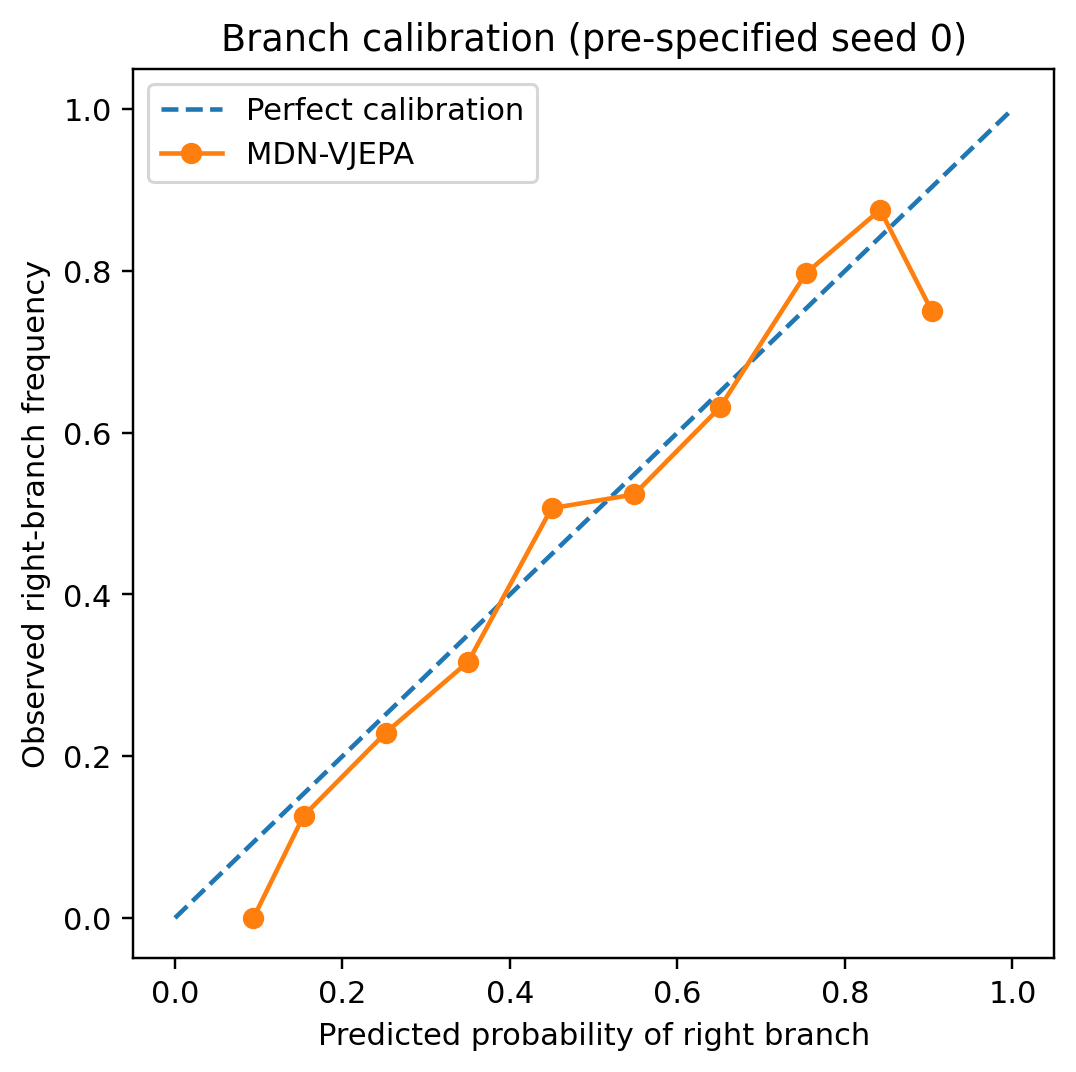}
\caption{Representative branch-probability reliability diagram. The dashed diagonal denotes perfect calibration.}
\label{fig:exp1-calibration}
\end{figure}

\subsection{Experiment 2: stochastic robot navigation}
\label{app:exp2-details}

\paragraph{Controlled navigation simulator.}
The benchmark operates directly in a two-dimensional predictive state space so that the conditional trajectory law and downstream collision events are known exactly. Each episode starts from
\begin{equation}
Z_t=
\begin{bmatrix}
0\\
Y_0
\end{bmatrix},
\qquad
Y_0\sim\N(0,0.04^2),
\end{equation}
and an action
\begin{equation}
u\sim\operatorname{Unif}[-0.55,0.55]
\end{equation}
is observed by the predictor. The probability of the upper branch is
\begin{equation}
p_{\uparrow}(Y_0,u)
=
\operatorname{sigmoid}
\left(
-1.4u+0.8Y_0
\right).
\end{equation}
The realized hidden branch
\begin{equation}
M^\star
\sim
\operatorname{Bernoulli}
\left(
p_{\uparrow}(Y_0,u)
\right)
\end{equation}
is sampled only after the observable state--action context has been formed and is never provided to the predictive model.

For horizon index $j\in\{1,\ldots,H\}$, let
\begin{equation}
x_j=\frac{j}{H},
\qquad
s_j=\sin(\pi x_j).
\end{equation}
The two branch means are
\begin{align}
\mu_j^{\downarrow}
&=
\begin{bmatrix}
x_j\\
Y_0+(u-a_{\mathrm b})s_j
\end{bmatrix},
\\
\mu_j^{\uparrow}
&=
\begin{bmatrix}
x_j\\
Y_0+(u+a_{\mathrm b})s_j
\end{bmatrix},
\end{align}
with branch amplitude $a_{\mathrm b}=0.24$. After the branch is sampled, independent Gaussian perturbations with standard deviations $\sigma_x=0.004$ and $\sigma_y=0.018$ are added at every horizon. The resulting law therefore contains two concentrated curved trajectory families with context-dependent probabilities and nonzero within-branch uncertainty.

The obstacle is a disc centered at $(0.5,0)$ with radius $0.18$. A rollout is counted as a collision if any predicted state enters the disc. The goal is centered at $(1,Y_0)$ with tolerance $0.10$, and path length is the sum of Euclidean increments from the initial state through the $H$ predicted states. The base planning objective is
\begin{equation}
C(\tau)
=
L(\tau)
+
\lambda_{\mathrm{col}}
\mathbf{1}\{\tau\text{ collides}\},
\end{equation}
with $\lambda_{\mathrm{col}}=8$.

\paragraph{Data separation and leakage safeguards.}
For each of 5 seeds we independently generate $8000/1800/2200$ training, validation, and test trajectories. Exact state--action overlap between splits is explicitly checked. One-step transition training uses only teacher-forced transitions extracted from trajectories within the corresponding training or validation split; no trajectory fragments cross split boundaries. Test targets never affect checkpoint selection.

Planning uses a separately generated collection of initial states that is independent of all training, validation, and test trajectories. Learned planners have access only to the learned world model during action selection; for these planners, the true simulator is invoked only after the action has been chosen. The oracle planner instead uses simulator samples during action selection, from a random stream distinct from the stream subsequently used to evaluate the selected action. Ground-truth branch information is used only for held-out diagnostics such as semantic branch matching and never to select a component at deployment time.

\paragraph{Training and evaluation settings.}
\Cref{tab:exp2-hparams} summarizes the primary configuration.

\begin{table}[H]
\caption{Experiment~2 training and evaluation settings.}
\label{tab:exp2-hparams}
\centering
{\scriptsize
\setlength{\tabcolsep}{4pt}
\renewcommand{\arraystretch}{0.95}
\begin{tabular}{@{}ll@{}}
\toprule
Setting & Value \\
\midrule
Seeds & $\{0,1,2,3,4\}$ \\
Train / validation / test trajectories & $8000/1800/2200$ \\
Prediction horizon & $H=12$ \\
State dimension & $2$ \\
MDN components & $K=2$ \\
Hidden width & $128$ \\
Maximum training epochs & $200$ \\
Early-stopping patience & $30$ \\
Batch size & $256$ \\
Optimizer & AdamW \\
Learning rate / weight decay & $10^{-3}/10^{-4}$ \\
Gradient-norm clipping & $5$ \\
Gaussian standard-deviation floor & $10^{-3}$ \\
MDN common warm-start & $25$ epochs \\
Usage safeguard $(\lambda_{\mathrm{use}},\rho_{\min})$ & $(10,0.10)$ \\
Direct / one-step likelihood weights & $1/1$ \\
Final CK weight & $5$ \\
CK ramp duration & $30$ epochs \\
CK training horizons & $\{4,8,12\}$ \\
CK samples per training context & $16$ \\
Gumbel temperature & $1.0\rightarrow0.25$ \\
Direct predictive evaluation samples & $128$ \\
CK test contexts / samples & $512/96$ \\
Planning contexts & $160$ \\
Candidate actions & $41$ \\
Standard planning samples & $128$ \\
CVaR planning samples & $512$ \\
Simulator evaluation rollouts & $768$ \\
Oracle-selection rollouts & $256$ \\
Collision penalty & $8$ \\
Chance-constraint threshold & $0.05$ \\
CVaR level & $0.90$ \\
CK contexts per update & $64$ \\
CK update frequency & every $2$ direct minibatches \\
CK training estimator & differentiable energy-distance V-statistic \\
CK validation contexts / samples & $128/24$ \\
\bottomrule
\end{tabular}
}
\end{table}

The no-CK and CK MDNs first share an identical $25$-epoch warm-start and are then cloned from the same parameters. The warm-start combines the direct trajectory NLL, the one-step NLL, and the fixed-$K$ component-utilization safeguard. Before combination, the direct trajectory NLL is divided by $H$ so that its numerical scale does not grow simply with the number of predicted horizons; the normalized direct and one-step likelihood terms then receive weights $1/1$. During fine-tuning, both likelihoods are optimized jointly from the same training-trajectory minibatch. For CK training, hard straight-through Gumbel--Softmax samples are used in the forward pass so that each generated state lies on one actual mixture component while retaining a differentiable surrogate gradient for the categorical choice. The CK coefficient is increased gradually over the first $30$ fine-tuning epochs.

Checkpoint selection is validation-only. For the CK arm, the validation score includes both predictive likelihood and an exact-categorical sample estimate of validation direct-versus-composed discrepancy. No test, planning, collision, or oracle metric participates in checkpoint selection.

\paragraph{Capacity control.}
\Cref{tab:exp2-capacity} summarizes the predictor capacities used in this experiment. The direct MDN predictor contains $29{,}666$ trainable parameters. We choose a Gaussian predictor with hidden width $148$, yielding $29{,}796$ parameters, so its direct predictive capacity is closely matched to the MDN direct head. The complete MDN world model contains $48{,}108$ parameters because it additionally contains the one-step transition network required for recursive propagation and CK evaluation.

\begin{table}[H]
\caption{Predictor parameter counts used in Experiment~2. For the MDN, the total additionally includes the one-step transition network used for recursive composition.}
\label{tab:exp2-capacity}
\centering
{\scriptsize
\setlength{\tabcolsep}{5pt}
\renewcommand{\arraystretch}{0.95}
\begin{tabular}{@{}lrr@{}}
\toprule
Model & Direct predictor & Total world model \\
\midrule
Deterministic JEPA & $20{,}120$ & $20{,}120$ \\
Gaussian VJEPA (matched) & $29{,}796$ & $29{,}796$ \\
Branch-JEPA controlled atomic & $23{,}474$ & $23{,}474$ \\
MDN-VJEPA no CK & $29{,}666$ & $48{,}108$ \\
MDN-VJEPA + CK & $29{,}666$ & $48{,}108$ \\
\bottomrule
\end{tabular}
}
\end{table}

\paragraph{Calibration, branch support, and component utilization.}
\Cref{tab:exp2-calibration} provides the detailed diagnostics underlying the compact predictive comparison in the main text. Both MDN-VJEPA and the controlled atomic predictor identify both semantic route families: weighted support precision and branch coverage are $1.0$ for every multimodal model under the fixed evaluation tolerance. MDN-VJEPA additionally reproduces the continuous law around those branches, explaining its substantially smaller energy distance.

\begin{table}[H]
\caption{Calibration and mixture diagnostics for Experiment~2. Values are mean $\pm$ standard deviation over 5 seeds.}
\label{tab:exp2-calibration}
\centering
{\scriptsize
\setlength{\tabcolsep}{3.2pt}
\renewcommand{\arraystretch}{0.95}
\begin{tabular}{@{}lcccc@{}}
\toprule
Model
& Failure ECE $\downarrow$
& Branch ECE $\downarrow$
& Branch MAE $\downarrow$
& $K_{\mathrm{eff}}$ \\
\midrule
Branch-JEPA controlled atomic
& $0.0339\pm0.0073$
& $0.0200\pm0.0093$
& $0.0129\pm0.0071$
& $1.9561\pm0.0061$ \\
MDN-VJEPA
& $0.0204\pm0.0062$
& $0.0241\pm0.0102$
& $0.0162\pm0.0048$
& $1.9530\pm0.0101$ \\
MDN-VJEPA + CK
& $0.0386\pm0.0170$
& $0.0508\pm0.0235$
& $0.0535\pm0.0281$
& $1.8947\pm0.0311$ \\
Generator oracle
& $0.0147\pm0.0033$
& $0.0159\pm0.0060$
& $0$
& $1.9535\pm0.0004$ \\
\bottomrule
\end{tabular}
}
\end{table}

The corresponding matched-support trajectory errors are $0.0039\pm0.0005$, $0.0047\pm0.0005$, and $0.0047\pm0.0006$ for the atomic, MDN, and MDN+CK predictors, respectively. Thus, CK does not destroy the learned route geometry even though it modestly perturbs the mixture probabilities. Across all seeds, the minimum effective component count is $1.942$ for MDN-VJEPA without CK and $1.849$ with CK. The minimum average posterior responsibility is $0.4582$, providing no evidence of accidental component collapse. The representative collision-probability reliability diagram in \cref{fig:exp2-failure-calibration} provides a complementary visual calibration check, with bin counts shown explicitly to indicate the empirical support of each reliability estimate.

\begin{figure}[H]
\centering
\includegraphics[width=0.42\linewidth]{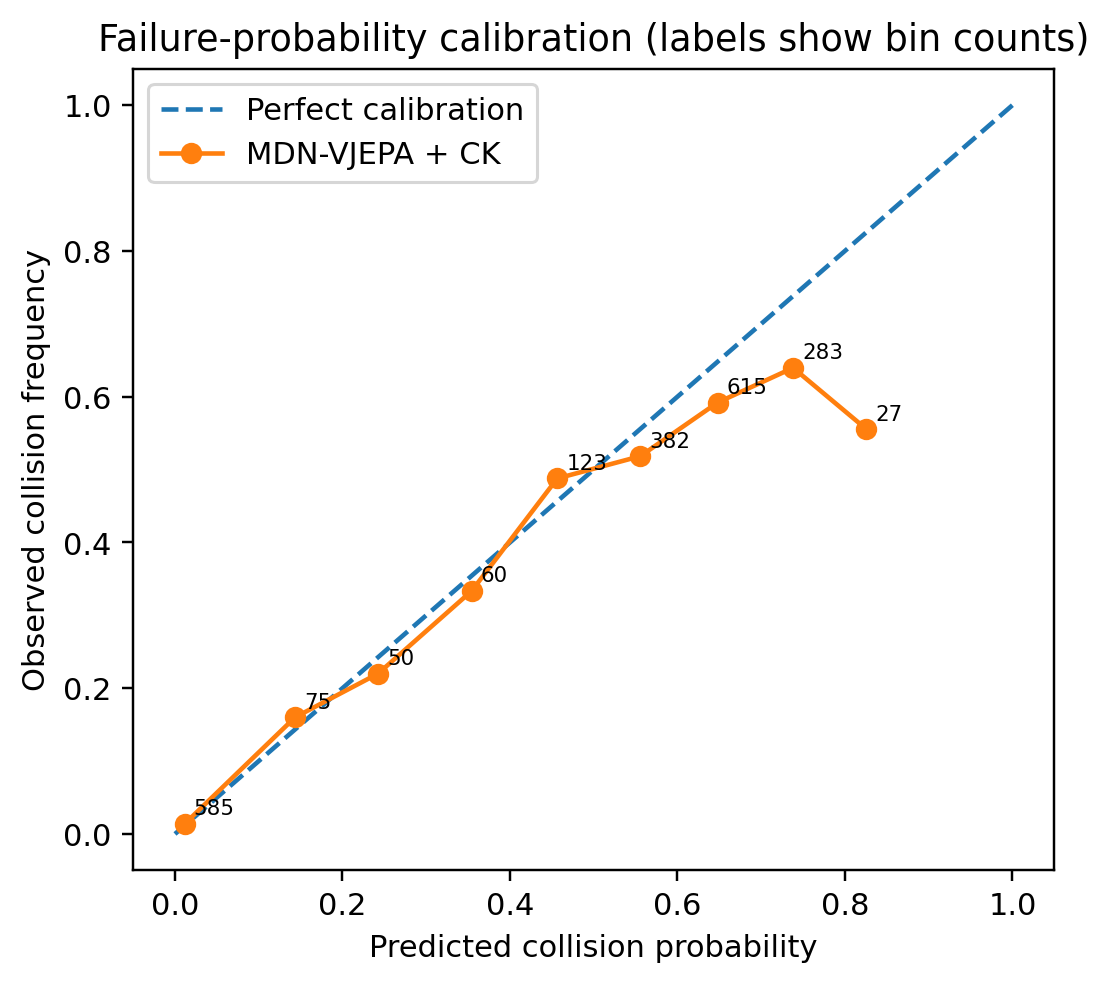}
\caption{Representative collision-probability reliability diagram for MDN-VJEPA+CK. Numbers denote the number of test examples in each probability bin. Ground-truth collision events are used only for evaluation.}
\label{fig:exp2-failure-calibration}
\end{figure}

\paragraph{Full planning comparison.}
\Cref{tab:exp2-planning-full} includes all planning variants. Goal rate is $1.0$ for every method, so differences in success arise from collisions rather than failure to reach the terminal region. The expected-cost and CVaR planners based on the MDN mixture both approach the oracle objective. In contrast, the Gaussian achieves essentially zero collisions through a substantially longer route, while the atomic predictor retains short trajectories but underestimates some local collision uncertainty.

\begin{table}[H]
\caption{Full planning comparison for Experiment~2. Values are mean $\pm$ standard deviation over 5 seeds.}
\label{tab:exp2-planning-full}
\centering
{\scriptsize
\setlength{\tabcolsep}{2.5pt}
\renewcommand{\arraystretch}{0.94}
\begin{tabular}{@{}lcccc@{}}
\toprule
Planner
& Success $\uparrow$
& Collision $\downarrow$
& Path length $\downarrow$
& True objective $\downarrow$ \\
\midrule
MDN-CK mean-only
& $0.4437\pm0.0095$
& $0.5563\pm0.0095$
& $1.1967\pm0.0065$
& $5.6473\pm0.0820$ \\
MDN-CK full-expected
& $0.9975\pm0.0003$
& $0.0025\pm0.0003$
& $1.1626\pm0.0007$
& $1.1828\pm0.0018$ \\
MDN-CK chance
& $0.9905\pm0.0028$
& $0.0095\pm0.0028$
& $1.1626\pm0.0002$
& $1.2383\pm0.0224$ \\
MDN-CK CVaR
& $0.9984\pm0.0002$
& $0.0016\pm0.0002$
& $1.1673\pm0.0089$
& $1.1799\pm0.0087$ \\
MDN-noCK full-expected
& $0.9975\pm0.0004$
& $0.0025\pm0.0004$
& $1.1623\pm0.0001$
& $1.1821\pm0.0029$ \\
Gaussian full-expected
& $1.0000\pm0.0000$
& $0.0000\pm0.0000$
& $1.4539\pm0.0008$
& $1.4539\pm0.0008$ \\
Branch-JEPA atomic expected
& $0.9399\pm0.0305$
& $0.0601\pm0.0305$
& $1.1617\pm0.0002$
& $1.6429\pm0.2439$ \\
Deterministic point
& $0.4409\pm0.0012$
& $0.5591\pm0.0012$
& $1.1964\pm0.0007$
& $5.6695\pm0.0095$ \\
Oracle expected-cost
& $0.9982\pm0.0002$
& $0.0018\pm0.0002$
& $1.1623\pm0.0001$
& $1.1769\pm0.0017$ \\
\bottomrule
\end{tabular}
}
\end{table}

\paragraph{Robustness to the collision penalty.}
We repeat the same mean-only versus full-distribution planning comparison for collision penalties $\lambda_{\mathrm{col}}\in\{2,4,8,16\}$ without retraining the world model. As shown in \cref{fig:exp2-risk-sweep}, the mean-only collision rate remains approximately $0.56$ throughout, whereas full-distribution planning remains near $0.0025$. Increasing the penalty therefore cannot recover the branch information that has already been discarded by collapsing the predictive law to its mean.

\begin{figure}[H]
\centering
\includegraphics[width=0.45\linewidth]{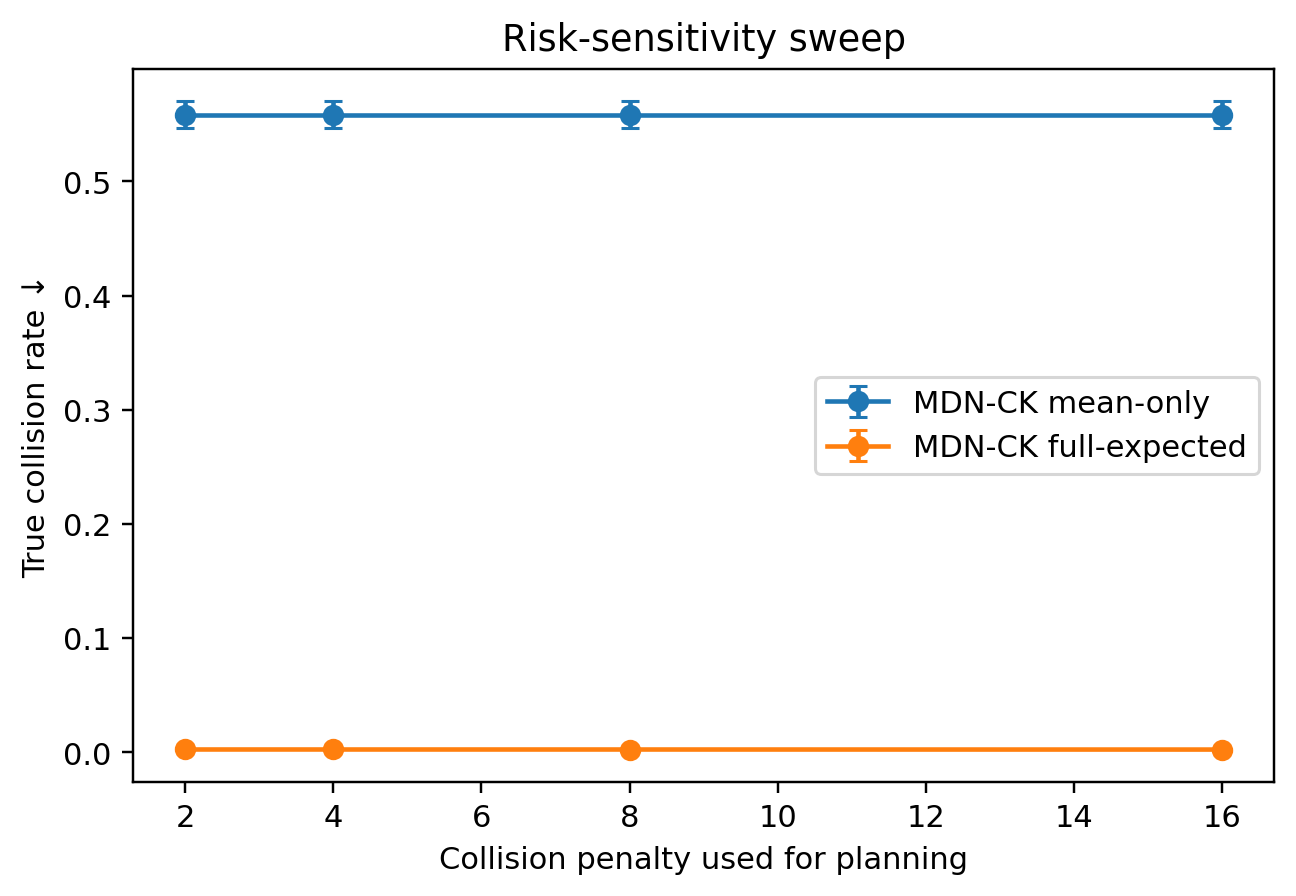}
\caption{Collision-rate robustness to the planning collision penalty. Mean-only and full-distribution planners use the same MDN-VJEPA+CK model. Retaining the full predictive distribution remains beneficial throughout the tested penalty range.}
\label{fig:exp2-risk-sweep}
\end{figure}

\paragraph{Detailed CK comparison.}
\Cref{tab:exp2-ck-full} reports both predictive fidelity and direct-versus-composed discrepancies. CK training uses a differentiable finite-sample V-statistic estimate of energy distance, whereas the reported test discrepancy uses the off-diagonal estimator defined by the evaluation protocol. The resulting direct--composed energy distance decreases from $0.1585\pm0.0097$ to $0.1236\pm0.0119$. The reduction occurs in every seed. At the same time, recursively composed target Energy Score improves from $0.3806\pm0.0119$ to $0.3741\pm0.0136$, while direct-target Energy Score changes only slightly.

\begin{table}[H]
\caption{Detailed distributional composition results on the pre-specified $512$-context CK evaluation subset, using $96$ samples per predictive distribution. Values are mean $\pm$ standard deviation over 5 seeds.}
\label{tab:exp2-ck-full}
\centering
{\scriptsize
\setlength{\tabcolsep}{3.0pt}
\renewcommand{\arraystretch}{0.95}
\begin{tabular}{@{}lcccc@{}}
\toprule
Model
& Dir.--comp. ED $\downarrow$
& Dir.--comp. MMD $\downarrow$
& Direct target ES $\downarrow$
& Composed target ES $\downarrow$ \\
\midrule
MDN-VJEPA no CK
& $0.1585\pm0.0097$
& $0.0980\pm0.0123$
& $0.3052\pm0.0055$
& $0.3806\pm0.0119$ \\
MDN-VJEPA + CK
& $0.1236\pm0.0119$
& $0.1065\pm0.0089$
& $0.3076\pm0.0089$
& $0.3741\pm0.0136$ \\
\bottomrule
\end{tabular}
}
\end{table}

The auxiliary RBF-MMD does not show the same improvement, increasing from $0.0980\pm0.0123$ to $0.1065\pm0.0089$ in \cref{tab:exp2-ck-full}. In this experiment its bandwidth is chosen separately for each distribution pair using that pair's median inter-sample distance. The resulting numerical MMD values therefore use different kernels across the no-CK and CK comparisons and are treated as a secondary diagnostic rather than the primary CK statistic. Experiment~4 provides the dedicated test of multi-horizon composition and is the more appropriate setting for a common fixed-bandwidth MMD comparison.

Overall, the navigation experiment exposes a three-way tradeoff. The unconstrained MDN provides the most accurate direct conditional law; CK sacrifices a small amount of direct calibration to improve recursive consistency; and retaining the resulting complete predictive distribution during planning is substantially more useful than collapsing that same belief to its mean.

\subsection{Experiment 3: stochastic video prediction in representation space}
\label{app:exp3-details}

\paragraph{Controlled stochastic-video generator.}
Experiment~3 is designed specifically to isolate the difference between multimodal marginals and a coherent multimodal trajectory law. Each episode contains $C=4$ observed frames followed by $H=10$ future frames. The observed object has initial vertical coordinate and horizontal speed
\begin{equation*}
Y_0
\sim
\operatorname{Unif}[-0.08,0.08],
\qquad
V
\sim
\operatorname{Unif}[0.9,1.1].
\end{equation*}
The observed context follows a straight pre-branch trajectory. For context frame $c\in\{1,\ldots,C\}$,
\begin{equation*}
x_c^{\mathrm C}
=
-0.55-0.08V(C-c),
\qquad
y_c^{\mathrm C}
=
Y_0.
\end{equation*}

The future contains three semantic branches
\begin{equation*}
b\in
\{\mathrm{down},\mathrm{straight},\mathrm{up}\},
\qquad
s_b\in\{-1,0,+1\},
\end{equation*}
where $s_b$ denotes the corresponding branch sign. Let
\begin{equation*}
r_h=\frac{h}{H},
\qquad
g_h=3r_h^2-2r_h^3,
\qquad
h=1,\ldots,H.
\end{equation*}
The branch mean positions are
\begin{align*}
x_{b,h}
&=
-0.55+0.12Vh,
\\
y_{b,h}
&=
Y_0+0.60\,s_b\,g_h.
\end{align*}
The smoothstep factor $g_h$ makes the three futures initially close and then progressively separates them.

Branch probabilities depend only on the pre-branch episode variables $Y_0$ and $V$, which determine the noiseless observed context, and not on the subsequently sampled future mode. Defining
\begin{equation*}
c=\frac{Y_0}{0.08},
\qquad
v=\frac{V-1}{0.1},
\end{equation*}
the branch probabilities are
\begin{equation*}
\boldsymbol p
=
\operatorname{softmax}
\left(
-0.55c-0.25v, 
0.45, 
0.55c+0.25v
\right).
\end{equation*}
Only after the pre-branch episode variables, and hence the noiseless context, have been determined do we sample the hidden future mode
\begin{equation*}
M^\star
\sim
\Cat(\boldsymbol p).
\end{equation*}
The realized mode is never included in the predictor input.

\paragraph{Frozen representation-space tokens.}
Rather than training an observation encoder in this experiment, each frame position $(x,y)$ is deterministically mapped to the target-independent token
\begin{equation*}
\psi(x,y)
=
\begin{bmatrix}
x,\,
y,\,
xy,\,
x^2,\,
y^2,\,
\sin(\pi x),\,
\cos(\pi x),\,
\sin(\pi y),\,
\cos(\pi y),\,
\sin\!\left(\frac{\pi}{2}(x+y)\right),\,
\cos\!\left(\frac{\pi}{2}(x-y)\right),\,
\sin(2\pi y)
\end{bmatrix}^{\!\top}.
\end{equation*}
Thus $D=12$. Before the additive representation noise described below, the first two coordinates coincide with the object centroid $(x,y)$ and therefore provide a fixed centroid-aligned visualization probe.

For a realized branch $M^\star=b$, the future representation is
\begin{equation*}
Z_{t+h}
=
\psi(x_{b,h},y_{b,h})
+
\epsilon_h,
\end{equation*}
with
\begin{equation*}
\epsilon_h
\sim
\N
\left(
0,
\operatorname{diag}
\left(
0.018^2,\,
0.018^2,\,
0.035^2,\ldots,0.035^2
\right)
\right).
\end{equation*}
The context tokens receive independent isotropic representation noise with standard deviation $0.010$. Pixel-space Gaussian-blob frames are rendered only for optional visualization and never enter training or evaluation of the predictive models.

\paragraph{Data separation and evaluation-only semantics.}
For each of five independent seeds, we generate $6000/1500/1800$ train, validation, and test episodes. Exact context overlap across splits is explicitly checked. The test set is used only for final evaluation.

Semantic branch identities and exact branch prototypes are never used in optimization, checkpoint selection, routing, or inference. They are used only after prediction for permutation-invariant evaluation. Because the true branches are initially close, semantic coherence is scored only over future frames $4$--$10$. The mean nearest-branch separation over these evaluated horizons ranges from approximately $1.07$ to $1.99$ in the $12$-dimensional representation space, substantially larger than the learned marginal support errors reported below.

\paragraph{Generator-oracle reference.}
For evaluation, we additionally report a \emph{generator oracle} constructed from the known synthetic data-generating process. The oracle uses the underlying episode variables and their exact branch probabilities and branch prototypes before context-representation noise is added. It therefore represents the conditional future law of the controlled generator given these underlying variables, rather than the exact Bayes conditional law given only the noisy context tokens supplied to the learned predictors. The generator oracle is used solely as a reference for predictive metrics and is not a realizable predictor available at inference time.

\paragraph{Model and training protocol.}
All learned models use a two-layer SiLU MLP with hidden width $128$ applied to the flattened $C\times D$ context tokens. The deterministic predictor outputs one $H\times D$ trajectory. The Gaussian predictor outputs one mean and diagonal standard deviation per future token.

MDN-VJEPA-I and MDN-VJEPA-II use the same MLP trunk architecture and component-head parameterization, producing
\begin{equation*}
\left\{
\mu_{k,h},
\sigma_{k,h}
\right\}_{k=1,h=1}^{K,H}.
\end{equation*}
They differ in the temporal factorization of the routing distribution. Level I outputs $H\times K$ horizon-specific logits and evaluates
\begin{equation*}
\log p_{\mathrm I}(Z_{t+1:t+H}\mid Z_t)
=
\sum_{h=1}^{H}
\log
\sum_{k=1}^{K}
\pi_{k,h}
p_{k,h}(Z_{t+h}),
\end{equation*}
whereas Level II outputs one $K$-way routing distribution and evaluates
\begin{equation*}
\log p_{\mathrm {II}}(Z_{t+1:t+H}\mid Z_t)
=
\log
\sum_{k=1}^{K}
\pi_k
\prod_{h=1}^{H}
p_{k,h}(Z_{t+h}).
\end{equation*}
This is the controlled distinction tested by Experiment~3: both predictors expose the same horizon-specific Gaussian component parameterization, but Level I factorizes the mixture choice across horizons whereas Level II shares one mixture choice across the complete trajectory.

To avoid giving either model semantic supervision, both MDNs are initialized from the same label-free $K$-means procedure applied only to centered training target trajectories. Centering removes sample-specific offsets and makes initialization depend primarily on trajectory shape rather than absolute location. Neither semantic labels nor test information enter this initialization, and the mixture-routing logits are initialized neutrally.

Training uses AdamW with validation-only early stopping. We include a weak fixed-$K$ usage safeguard against accidental dead components; it acts only when a component's average responsibility falls below $\rho_{\min}=0.08$ and does not force equal routing probabilities. The primary settings are summarized in \cref{tab:exp3-hparams}.

\begin{table}[H]
\caption{Experiment~3 training and evaluation settings.}
\label{tab:exp3-hparams}
\centering
{\scriptsize
\setlength{\tabcolsep}{4pt}
\renewcommand{\arraystretch}{0.95}
\begin{tabular}{@{}ll@{}}
\toprule
Setting & Value \\
\midrule
Seeds & $\{0,1,2,3,4\}$ \\
Train / validation / test episodes & $6000/1500/1800$ \\
Context frames & $C=4$ \\
Future horizon & $H=10$ \\
Representation dimension & $D=12$ \\
Semantic branches / MDN components & $K=3$ \\
Semantic evaluation frames & future frames $4$--$10$ \\
Hidden width & $128$ \\
Maximum epochs & $140$ \\
Early-stopping patience & $20$ \\
Batch size & $256$ \\
Optimizer & AdamW \\
Learning rate / weight decay & $10^{-3}/10^{-4}$ \\
Gradient-norm clipping & $5$ \\
Predictive standard-deviation floor & $5\times10^{-3}$ \\
Usage safeguard $(\lambda_{\mathrm{use}},\rho_{\min})$ & $(5,0.08)$ \\
$K$-means restarts & $10$ \\
Evaluation samples & $64$ \\
Best-of-$N$ budget & $N=64$ \\
Nominal sample-interval coverage & $90\%$ \\
Marginal support tolerance & $0.15$ \\
\bottomrule
\end{tabular}
}
\end{table}

\paragraph{Capacity comparison.}
The parameter counts are shown in \cref{tab:exp3-capacity}. Level I is slightly larger than Level II because it predicts an independent $K$-way routing distribution for every horizon. The large Level-II improvement therefore cannot be explained by additional predictor capacity.

\begin{table}[H]
\caption{Trainable predictor parameter counts in Experiment~3.}
\label{tab:exp3-capacity}
\centering
{\scriptsize
\setlength{\tabcolsep}{6pt}
\renewcommand{\arraystretch}{0.95}
\begin{tabular}{@{}lr@{}}
\toprule
Model & Parameters \\
\midrule
Deterministic JEPA & $38{,}264$ \\
Gaussian VJEPA & $53{,}744$ \\
MDN-VJEPA-I & $119{,}534$ \\
MDN-VJEPA-II & $116{,}051$ \\
\bottomrule
\end{tabular}
}
\end{table}

\paragraph{Semantic-coherence metrics.}
For a predictive trajectory sample $\widetilde Z_{t+1:t+H}$, we assign each evaluated horizon to its nearest semantic branch prototype,
\begin{equation*}
s_h(\widetilde Z)
=
\arg\min_b
\left\|
\widetilde Z_{t+h}
-
\mu_{b,h}^{\star}
\right\|_2.
\end{equation*}
For the evaluated horizon set $\mathcal H_{\mathrm{sem}}$, the semantic switch rate is
\begin{equation*}
R_{\mathrm{switch}}
=
\frac{1}
{N\left(|\mathcal H_{\mathrm{sem}}|-1\right)}
\sum_{n=1}^{N}
\sum_{\substack{h,h'\in\mathcal H_{\mathrm{sem}}\\h'=h+1}}
\mathbf{1}
\left\{
s_{h'}(\widetilde Z^{(n)})
\neq
s_h(\widetilde Z^{(n)})
\right\}.
\end{equation*}
The any-switch rate is the fraction of predictive samples with at least one such change. Coherent branch coverage asks whether, among the fixed $N=64$ samples, each semantic branch appears in at least one trajectory whose assignment remains on that branch throughout the evaluated horizon. Marginal branch coverage ignores temporal coupling and instead asks whether all branches are represented at each individual horizon.

The resulting coherence statistics are reported in the main text in \cref{tab:exp3-coherence}. They distinguish two different questions:
\begin{equation*}
\textit{Are all futures represented?}
\qquad\textit{vs.}\qquad
\textit{Are they represented as coherent futures?}
\end{equation*}
Level I answers the first question correctly but fails the second: it attains full marginal branch coverage while exhibiting frequent semantic switching. Level II preserves the same marginal coverage while recovering coherent complete trajectories.

\paragraph{Additional probabilistic and component diagnostics.}
Detailed diagnostics are given in \cref{tab:exp3-diagnostics}. The marginal support errors of Levels I and II are nearly identical, $0.0097\pm0.0004$ and $0.0096\pm0.0006$, respectively, while both retain all three branches at every evaluated horizon. Their effective component counts are also similar and close to the oracle value. Across individual seeds, the minimum $K_{\mathrm{eff}}$ is $2.787$ for Level I and $2.841$ for Level II, providing no evidence of component collapse.

The Level-II component-mean diagnostic independently assigns every horizon-specific component mean to its nearest semantic branch after prediction. Its semantic switch rate and any-switch rate are both exactly zero for all 5 seeds. This is stronger than simply observing a fixed numerical component index: it verifies that each learned component-mean sequence follows one semantic future throughout the evaluated block.

\begin{table}[H]
\caption{Additional probabilistic and mixture diagnostics for Experiment~3. Values are mean $\pm$ standard deviation over 5 seeds. The $90\%$ column is a finite-$N=64$ predictive sample-interval diagnostic rather than an analytic interval.}
\label{tab:exp3-diagnostics}
\centering
{\scriptsize
\setlength{\tabcolsep}{2.8pt}
\renewcommand{\arraystretch}{0.95}
\begin{tabular}{@{}lccccc@{}}
\toprule
Model
& $90\%$ cov.
& Branch Brier $\downarrow$
& Support err. $\downarrow$
& Marginal cov. $\uparrow$
& $K_{\mathrm{eff}}$ \\
\midrule
Gaussian VJEPA
& $0.8805\pm0.0041$
& --
& --
& --
& -- \\
MDN-VJEPA-I
& $0.8762\pm0.0019$
& $0.6398\pm0.0036$
& $0.0097\pm0.0004$
& $1.0000\pm0.0000$
& $2.8515\pm0.0373$ \\
MDN-VJEPA-II
& $0.8758\pm0.0015$
& $0.6400\pm0.0039$
& $0.0096\pm0.0006$
& $1.0000\pm0.0000$
& $2.8637\pm0.0204$ \\
Generator oracle
& $0.8724\pm0.0014$
& $0.6389\pm0.0032$
& $0.0000\pm0.0000$
& $1.0000\pm0.0000$
& $2.8531\pm0.0013$ \\
\bottomrule
\end{tabular}
}
\end{table}

The nominal $90\%$ sample interval yields empirical coverage around $0.87$--$0.88$ for all probabilistic models, including the oracle. Because the interval is constructed from only $64$ predictive samples using empirical quantiles, we interpret this metric relative to the oracle under the same finite-sample protocol rather than as an analytic calibration interval. The close agreement among Gaussian VJEPA, both MDN variants, and the oracle indicates that the dramatic Level-I/Level-II difference is not explained by a gross difference in marginal interval width.

\paragraph{Predictive-score visualization.}
The aggregate Energy Score comparison from \cref{tab:exp3-predictive} is visualized in \cref{fig:exp3-energy-score}. Level II is close to the generator oracle and substantially improves over both the independent-mixture Level I model and the unimodal Gaussian reference.

\begin{figure}[H]
\centering
\includegraphics[width=0.50\linewidth]{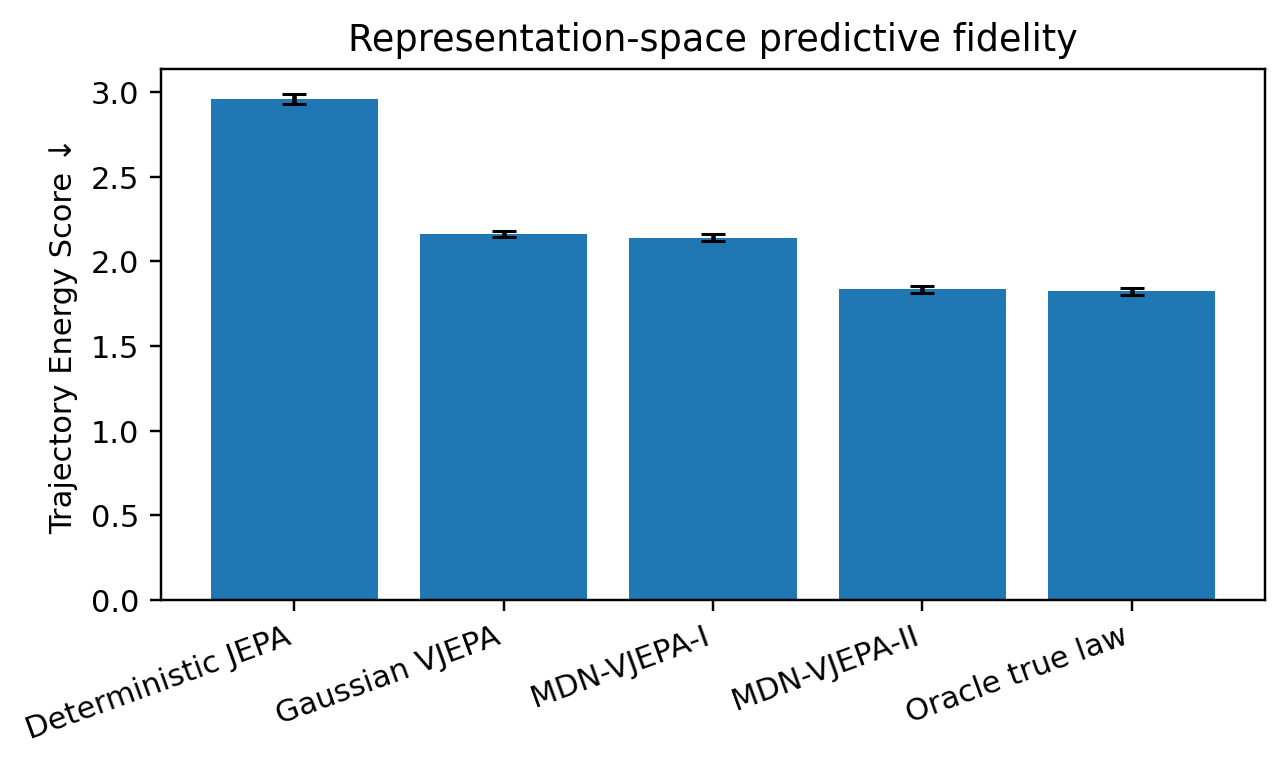}
\caption{Trajectory Energy Score on the stochastic-video representation benchmark. Level II approaches the oracle trajectory law, whereas Level I remains close to the unimodal Gaussian despite accurately modeling the individual horizon-wise branch supports. Error bars show one standard deviation over 5 seeds.}
\label{fig:exp3-energy-score}
\end{figure}

\paragraph{Learned branch uncertainty.}
\Cref{fig:exp3-uncertainty-tubes} provides a complementary visualization of the Level-II components for the pre-specified representative example. Components are globally matched to semantic branches only for visualization. Solid curves are learned component means, dashed curves are the corresponding true branch means, and shaded regions show the learned $90\%$ within-component envelope in the frozen centroid coordinate. The learned component means closely follow all three semantic futures while retaining narrow local uncertainty around each branch.

\begin{figure}[H]
\centering
\includegraphics[width=0.50\linewidth]{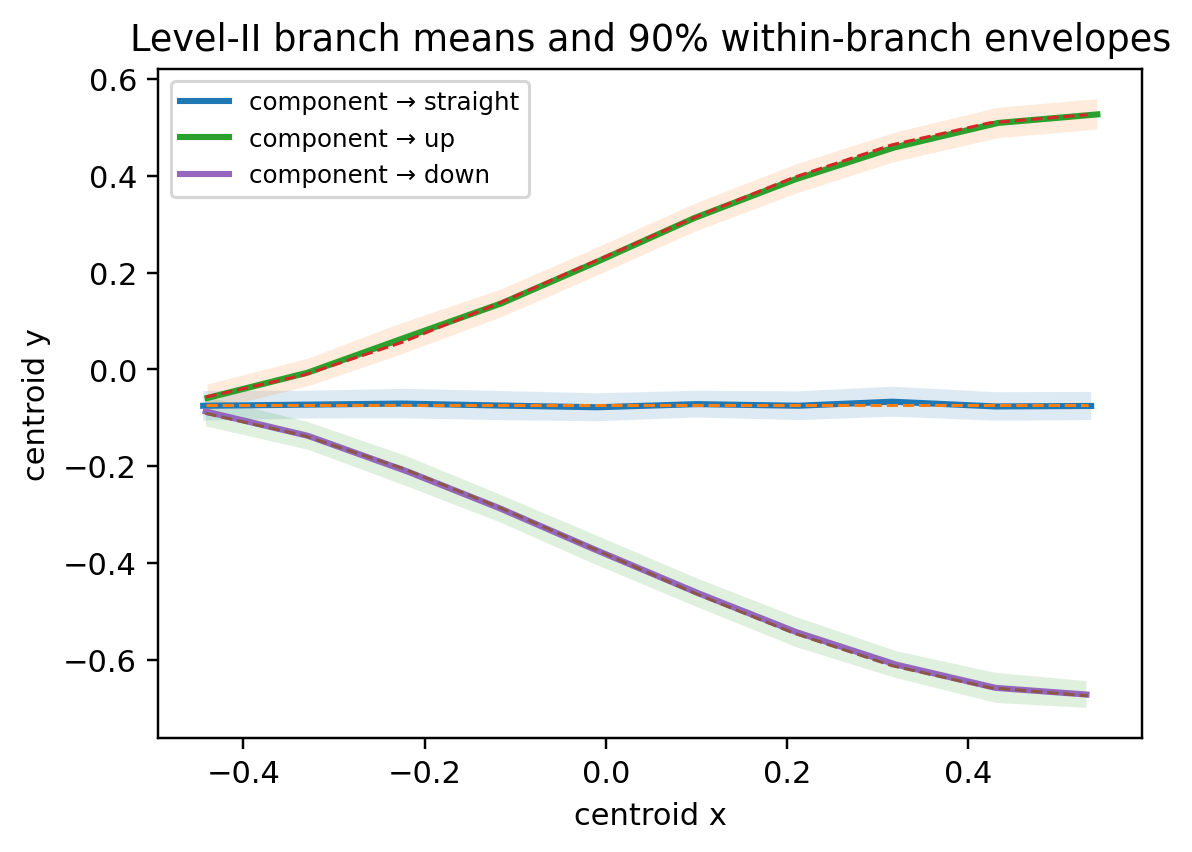}
\caption{Level-II component means and within-branch uncertainty for the pre-specified representative example. Each learned component is matched to a semantic branch using an evaluation-only global assignment. Solid curves denote learned component means, dashed curves denote the corresponding true branch means, and shaded regions show the learned $90\%$ within-component envelope in centroid $y$. The result illustrates the intended Level-II semantics: one component describes one coherent branch together with local uncertainty around it.}
\label{fig:exp3-uncertainty-tubes}
\end{figure}

Overall, Experiment~3 isolates the structural distinction between multimodal marginal accuracy and multimodal trajectory coherence. MDN-VJEPA-I learns all three branches, their probabilities, and their local per-horizon support, yet its independent mixture factorization assigns substantial probability to cross-horizon combinations that do not correspond to any realizable semantic future. MDN-VJEPA-II changes the temporal coupling of those modes and recovers a joint predictive distribution that is both semantically coherent and close to the oracle trajectory law.

\subsection{Experiment 4: direct versus composed multi-horizon prediction}
\label{app:exp4-details}

\paragraph{Controlled action-conditioned Markov process.}
Experiment~4 uses a three-dimensional continuous state
\begin{equation*}
Z_t=(x_t,y_t,r_t),
\end{equation*}
where $r_t$ is a continuous turn-memory coordinate. The initial state is sampled as
\begin{equation*}
x_0\sim\operatorname{Unif}[-0.18,0.18],
\qquad
y_0\sim\operatorname{Unif}[-0.14,0.14],
\qquad
r_0\sim\operatorname{Unif}[-0.025,0.025].
\end{equation*}

Each episode contains a known action sequence $u_{0:H-1}$. We generate smooth heterogeneous controls from
\begin{equation*}
u_j
=
b
+
a_1
\sin\!\left(
\frac{2\pi j}{H-1}+\varphi_1
\right)
+
a_2
\sin\!\left(
\frac{4\pi j}{H-1}+\varphi_2
\right),
\end{equation*}
where
\begin{equation*}
b\sim\operatorname{Unif}[-0.60,0.60],
\qquad
a_1\sim\operatorname{Unif}[0,0.20],
\qquad
a_2\sim\operatorname{Unif}[0,0.10],
\end{equation*}
and $\varphi_1,\varphi_2\sim\operatorname{Unif}[0,2\pi]$. Actions are clipped to $[-0.85,0.85]$.

At every transition, the upper mode $m_t=+1$ is sampled according to the state-dependent probability
\begin{equation*}
\Pr(m_t=+1\mid Z_t,u_t)
=
\operatorname{sigmoid}
\Big(
9r_t
+
1.15u_t
+
1.35y_t
-
0.40x_t
+
0.30\sin\!\big(\pi(x_t+0.20)\big)
-
0.15\sin(2\pi y_t)
\Big).
\end{equation*}
The lower mode corresponds to $m_t=-1$. Importantly, $m_t$ is resampled conditionally on the currently modeled state at every step; there is no persistent hidden categorical variable outside $Z_t$.

Let
\begin{equation*}
\widetilde u_t=\tanh(1.4u_t)
\end{equation*}
and
\begin{equation*}
d_t
=
0.050
+
0.022\sin\!\big(1.7\pi(x_t+0.28)\big)
+
0.015\cos(1.3\pi u_t)
+
0.010\tanh\!\left(2y_t r_t\right).
\end{equation*}
The conditional transition mean is
\begin{align*}
\mu_{x,t+1}
={}&
x_t
+
0.085
+
0.032\widetilde u_t
+
0.018\sin\!\big(\pi(y_t+0.45r_t)\big)
+
0.006\sin(2\pi x_t),
\\
\mu_{y,t+1}
={}&
0.915y_t
+
0.045\widetilde u_t
+
m_td_t
+
0.024\tanh(2.5y_t)
+
0.012r_t\sin\!\big(\pi(x_t+0.1)\big),
\\
\mu_{r,t+1}
={}&
0.84r_t
+
0.16m_t
+
0.020\tanh(1.8u_t)
+
0.010\sin(\pi y_t).
\end{align*}
The realized next state is
\begin{equation*}
Z_{t+1}
=
\mu_{t+1}
+
\epsilon_t,
\qquad
\epsilon_t
\sim
\N
\left(
0,
\operatorname{diag}
\left(
0.007^2,\,
0.015^2,\,
0.009^2
\right)
\right).
\end{equation*}

The $r_t$ dynamics make recently selected directions persistent but not deterministic. Consequently, recursive transition errors can accumulate nontrivially while the complete simulator remains Markovian in the modeled state.

\paragraph{Data separation and matched comparison.}
For each of 5 seeds we independently generate $7000/1800/2200$ training, validation, and test episodes. Exact initial-state/action-sequence overlap between splits is checked explicitly. The test set is used only for final evaluation.

Both Level-II and Level-III models contain a coherent direct trajectory MDN and a one-step MDN, each with $K=2$ diagonal-Gaussian components. Both models are first trained for a common $25$-epoch warm-start using direct trajectory likelihood, one-step likelihood, and the same weak fixed-$K$ component-utilization safeguard. The resulting checkpoint is cloned exactly before fine-tuning. Both arms subsequently receive the same deterministic minibatch order and optimization budget. Level II uses $\lambda_{\mathrm{CK}}=0$, whereas Level III adds the CK discrepancy.

For CK only, samples from the direct predictive distribution are treated as a stop-gradient anchor. The Experiment~4 implementation uses the horizon-normalized sample discrepancy
\begin{equation*}
\widehat{\cL}_{\mathrm{CK}}^{(h)}
=
\frac{
\widehat{\mathcal E}_U
\left(
\operatorname{sg}\!\left[
S_\phi^{\mathrm{dir},h}
\right],
S_\phi^{\mathrm{comp},h}
\right)
}{
s_h+\epsilon
},
\qquad
h\in\mathcal H_{\mathrm{CK}},
\end{equation*}
where $S_\phi^{\mathrm{dir},h}$ and $S_\phi^{\mathrm{comp},h}$ denote direct and recursively composed samples and $\epsilon=10^{-8}$. During minibatch optimization, the implementation cycles through
$
\mathcal H_{\mathrm{CK}}=\{2,4,8,12,16\}
$
and evaluates one horizon per minibatch, yielding an equal-weight stochastic realization of the multi-horizon CK objective.

\paragraph{Unbiased sample discrepancies.}
The initial version of this experiment used finite-sample V-statistics whose diagonal self-pairs introduce a positive finite-sample bias. The final experiment instead uses off-diagonal U-statistic estimators. Given direct samples $\{X_i\}_{i=1}^{M}$ and composed samples $\{Y_j\}_{j=1}^{N}$,
\begin{equation*}
\widehat{\mathcal E}_U(X,Y)
=
\frac{2}{MN}
\sum_{i=1}^{M}
\sum_{j=1}^{N}
\|X_i-Y_j\|
-
\frac{1}{M(M-1)}
\sum_{i\neq i'}
\|X_i-X_{i'}\|
-
\frac{1}{N(N-1)}
\sum_{j\neq j'}
\|Y_j-Y_{j'}\|.
\end{equation*}
The estimator is not clamped during training or test evaluation and may therefore be slightly negative at finite sample size.

The complementary RBF-MMD diagnostic uses the corresponding unbiased estimator
\begin{equation*}
\widehat{\operatorname{MMD}}_U^2
=
\frac{1}{M(M-1)}
\sum_{i\neq i'}
k(X_i,X_{i'})
+
\frac{1}{N(N-1)}
\sum_{j\neq j'}
k(Y_j,Y_{j'})
-
\frac{2}{MN}
\sum_{i,j}
k(X_i,Y_j).
\end{equation*}
For checkpoint selection only, a negative finite-sample validation CK estimate is truncated at zero so that random negative U-statistic fluctuations are not rewarded.

\paragraph{Horizon normalization and fixed calibration quantities.}
Raw energy distance changes scale as the true future law becomes more or less dispersed. We therefore estimate a true-law dispersion
\begin{equation*}
s_h
=
\E
\left[
\|Z_{t+h}-Z'_{t+h}\|
\right],
\qquad
Z_{t+h},Z'_{t+h}\sim p_h^\star,
\end{equation*}
using an independent held-out simulator calibration set. The normalized discrepancy is
\begin{equation*}
\widetilde{\mathcal E}_h
=
\frac{
\widehat{\mathcal E}_{U,h}
}{
s_h
}.
\end{equation*}
The primary scalar summary is the trapezoidal AUC of $\widetilde{\mathcal E}_h$ over $h\in\{2,4,8,12,16\}$, divided by the corresponding horizon interval.

The RBF bandwidth and $s_h$ are selected once from the same held-out calibration distribution and then frozen for all seeds, both model arms, and all test comparisons. Their values are reported in \cref{tab:exp4-calibration}.

\begin{table}[H]
\caption{Held-out calibration quantities fixed before model test evaluation. The same MMD bandwidth and energy-distance normalization scale are used for both arms and all seeds at each horizon.}
\label{tab:exp4-calibration}
\centering
{\scriptsize
\setlength{\tabcolsep}{6pt}
\renewcommand{\arraystretch}{0.95}
\begin{tabular}{@{}rcc@{}}
\toprule
$h$ & RBF bandwidth $\sigma_h$ & ED scale $s_h$ \\
\midrule
1  & $0.044688$ & $0.178499$ \\
2  & $0.313595$ & $0.309250$ \\
4  & $0.316808$ & $0.522428$ \\
8  & $0.271706$ & $0.799783$ \\
12 & $0.207838$ & $0.964058$ \\
16 & $0.177795$ & $1.078036$ \\
\bottomrule
\end{tabular}
}
\end{table}

\paragraph{Training and evaluation settings.}
The primary implementation settings are summarized in \cref{tab:exp4-hparams}. CK training uses $32$ samples for each of $24$ selected contexts, compared with $64$ samples for validation and $128$ samples for final evaluation. The CK coefficient is increased linearly over the first $30$ fine-tuning epochs. Checkpoint selection is validation-only, and no seed is retried, replaced, or selected according to test results.

\begin{table}[H]
\caption{Experiment~4 training and evaluation settings.}
\label{tab:exp4-hparams}
\centering
{\scriptsize
\setlength{\tabcolsep}{4pt}
\renewcommand{\arraystretch}{0.95}
\begin{tabular}{@{}ll@{}}
\toprule
Setting & Value \\
\midrule
Seeds & $\{0,1,2,3,4\}$ \\
Train / validation / test episodes & $7000/1800/2200$ \\
State dimension / components & $D=3$, $K=2$ \\
Prediction horizon & $H=16$ \\
Evaluation horizons & $\{1,2,4,8,12,16\}$ \\
CK horizons & $\{2,4,8,12,16\}$ \\
Hidden width & $128$ \\
Common warm-start & $25$ epochs \\
Maximum fine-tuning & $140$ epochs \\
Early-stopping patience & $30$ \\
Batch size & $256$ \\
Optimizer & AdamW \\
Learning rate / weight decay & $10^{-3}/10^{-4}$ \\
Gradient-norm clipping & $5$ \\
Predictive standard-deviation floor & $0.004$ \\
Usage safeguard $(\lambda_{\mathrm{use}},\rho_{\min})$ & $(5,0.05)$ \\
Final CK coefficient / ramp & $4.0$ / $30$ epochs \\
CK train contexts / samples & $24/32$ \\
CK validation contexts / samples & $128/64$ \\
Evaluation contexts / samples & $512/128$ \\
True-law samples per population & $128$ \\
Calibration contexts / samples & $256/64$ \\
CK estimator & unbiased energy-distance U-statistic \\
MMD estimator & unbiased fixed-bandwidth RBF MMD$^2$ U-statistic \\
CK gradient & stop-gradient direct anchor $\rightarrow$ recursive transition \\
\bottomrule
\end{tabular}
}
\end{table}

\paragraph{Capacity.}
Level II and Level III use exactly the same trainable architecture; CK introduces no additional model parameters. The direct trajectory head contains $44{,}098$ parameters and the one-step transition contains $18{,}958$, for a total of $63{,}056$ parameters in both arms.

\begin{table}[H]
\caption{Trainable parameter counts in Experiment~4. The Level-II and Level-III arms use exactly the same architecture; CK introduces no additional trainable parameters.}
\label{tab:exp4-capacity}
\centering
{\scriptsize
\setlength{\tabcolsep}{6pt}
\renewcommand{\arraystretch}{0.95}
\begin{tabular}{@{}lrrr@{}}
\toprule
Model & Direct head & One-step head & Total \\
\midrule
MDN-VJEPA-II  & 44,098 & 18,958 & 63,056 \\
MDN-VJEPA-III & 44,098 & 18,958 & 63,056 \\
\bottomrule
\end{tabular}
}
\end{table}

\paragraph{Horizon-wise consistency.}
\Cref{tab:exp4-horizon} gives the detailed direct--composed discrepancies underlying the AUC result in the main text. CK improves both raw and normalized energy distance at $h=2,4,8,$ and $12$. At $h=16$, the difference is negligible and slightly favors Level II. The main conclusion therefore comes from the complete horizon range rather than from selecting the terminal point.

\begin{table}[H]
\caption{Horizon-wise direct--composed discrepancies in Experiment~4. Energy distance and MMD use unbiased U-statistic estimators. MMD uses one held-out fixed bandwidth per horizon. Lower is better.}
\label{tab:exp4-horizon}
\centering
{\scriptsize
\setlength{\tabcolsep}{2.0pt}
\renewcommand{\arraystretch}{0.95}
\begin{tabular}{@{}rcccccc@{}}
\toprule
& \multicolumn{2}{c}{Norm. ED $\downarrow$}
& \multicolumn{2}{c}{Raw ED $\downarrow$}
& \multicolumn{2}{c}{MMD$^2$ $\downarrow$} \\
\cmidrule(lr){2-3}
\cmidrule(lr){4-5}
\cmidrule(lr){6-7}
$h$ & Level II & Level III & Level II & Level III & Level II & Level III \\
\midrule
2
& $0.06386\pm0.00383$
& $0.05503\pm0.00269$
& $0.01975\pm0.00119$
& $0.01702\pm0.00083$
& $0.00788\pm0.00155$
& $0.00543\pm0.00087$ \\
4
& $0.03564\pm0.00343$
& $0.02834\pm0.00206$
& $0.01862\pm0.00179$
& $0.01481\pm0.00107$
& $0.01002\pm0.00207$
& $0.00871\pm0.00112$ \\
8
& $0.01621\pm0.00206$
& $0.01375\pm0.00203$
& $0.01296\pm0.00165$
& $0.01100\pm0.00162$
& $0.01036\pm0.00215$
& $0.00872\pm0.00107$ \\
12
& $0.00903\pm0.00240$
& $0.00777\pm0.00104$
& $0.00871\pm0.00231$
& $0.00749\pm0.00100$
& $0.00958\pm0.00278$
& $0.00837\pm0.00113$ \\
16
& $0.00694\pm0.00206$
& $0.00707\pm0.00231$
& $0.00748\pm0.00222$
& $0.00762\pm0.00250$
& $0.01061\pm0.00200$
& $0.01067\pm0.00183$ \\
\bottomrule
\end{tabular}
}
\end{table}

The corresponding raw energy-distance curves are shown in \cref{fig:exp4-raw-ck}. Although the discrepancy is not monotone in horizon, the Level-III curve is systematically lower over most of the evaluated interval.

\begin{figure}[H]
\centering
\includegraphics[width=0.48\linewidth]{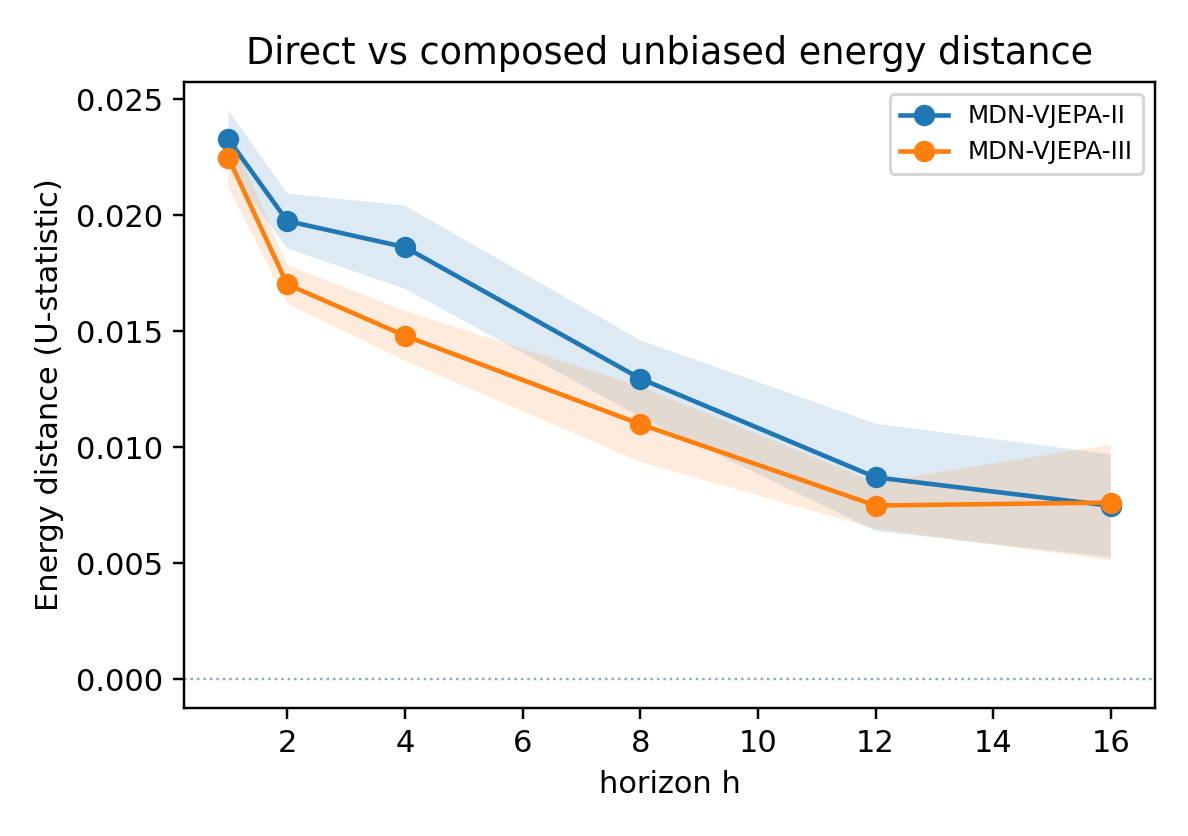}
\caption{Raw unbiased direct-versus-composed energy distance across horizons. Level III reduces discrepancy through $h=12$ and is comparable to Level II at the terminal horizon.}
\label{fig:exp4-raw-ck}
\end{figure}

\paragraph{Complementary MMD diagnostic.}
The fixed-bandwidth unbiased MMD$^2$-AUC decreases on average from $0.00992\pm0.00206$ to $0.00866\pm0.00111$, a reduction of approximately $12.7\%$. This secondary statistic is less stable across seeds than energy distance, however: the MMD$^2$-AUC improves in 2 of the 5 seeds. We therefore retain energy distance as the primary evidence for the CK claim and interpret MMD only as a complementary diagnostic.

\begin{figure}[H]
\centering
\includegraphics[width=0.48\linewidth]{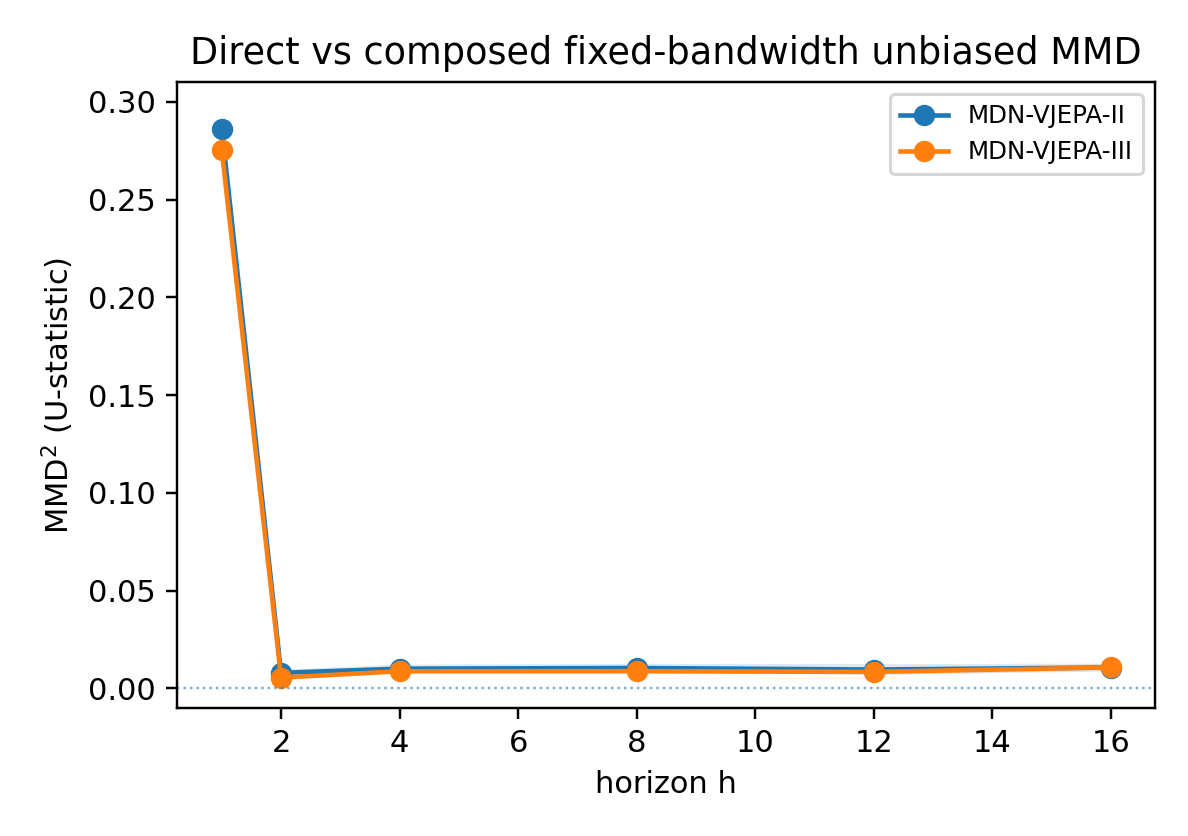}
\caption{Direct-versus-composed unbiased RBF-MMD$^2$ using one held-out fixed bandwidth per horizon. The mean Level-III curve is generally lower at intermediate horizons, but the seed-wise MMD$^2$-AUC improvement is less consistent than the primary energy-distance result.}
\label{fig:exp4-mmd}
\end{figure}

\paragraph{Target fidelity.}
\Cref{tab:exp4-fidelity} reports the terminal target-fidelity diagnostics. The direct path is essentially unaffected by CK: direct-to-true energy distance changes from $0.00603\pm0.00205$ to $0.00544\pm0.00135$, and direct Energy Score changes from $0.54507\pm0.01003$ to $0.54442\pm0.01066$. The recursively composed path incurs a small tradeoff, with composed-to-true energy distance increasing from $0.00407\pm0.00136$ to $0.00644\pm0.00157$. Its Energy Score nevertheless changes only from $0.54427\pm0.00937$ to $0.54533\pm0.01105$, with both values close to the oracle score $0.54187\pm0.00864$.

\begin{table}[H]
\caption{Endpoint target-fidelity and branch diagnostics at $h=16$. Values are mean $\pm$ standard deviation over 5 seeds. Energy distances use unbiased U-statistics. The direct mode-weight diagnostic compares the predictive weight of the component with the larger mean $r$ coordinate against the simulator estimate of $\Pr(r_{t+h}>0)$; the composed event diagnostic estimates $\Pr(r_{t+h}>0)$ directly from recursive samples.}
\label{tab:exp4-fidelity}
\centering
{\tiny
\setlength{\tabcolsep}{2.0pt}
\renewcommand{\arraystretch}{0.95}
\begin{tabular}{@{}lcccccc@{}}
\toprule
Model
& Direct--true ED $\downarrow$
& Composed--true ED $\downarrow$
& Direct ES $\downarrow$
& Composed ES $\downarrow$
& Direct mode-wt.\ MAE $\downarrow$
& Composed event MAE $\downarrow$ \\
\midrule
MDN-VJEPA-II
& $0.00603\pm0.00205$
& $\mathbf{0.00407\pm0.00136}$
& $0.54507\pm0.01003$
& $\mathbf{0.54427\pm0.00937}$
& $0.02868\pm0.00189$
& $\mathbf{0.04268\pm0.00330}$ \\
\textbf{MDN-VJEPA-III}
& $\mathbf{0.00544\pm0.00135}$
& $0.00644\pm0.00157$
& $\mathbf{0.54442\pm0.01066}$
& $0.54533\pm0.01105$
& $\mathbf{0.02774\pm0.00081}$
& $0.04454\pm0.00229$ \\
Oracle MC reference
& $0.00022\pm0.00087$
& $0.00022\pm0.00087$
& $0.54187\pm0.00864$
& $0.54187\pm0.00864$
& $0$
& $0$ \\
\bottomrule
\end{tabular}
}
\end{table}

The complete horizon-wise target-energy-distance comparison is shown in \cref{fig:exp4-target-ed}. The recursively composed Level-II predictor is particularly accurate in this controlled process. Anchoring it toward the separately supervised direct distribution therefore produces the small long-horizon fidelity cost visible above, while reducing direct--composed disagreement over the complete multi-horizon interval.

\begin{figure}[H]
\centering
\includegraphics[width=0.52\linewidth]{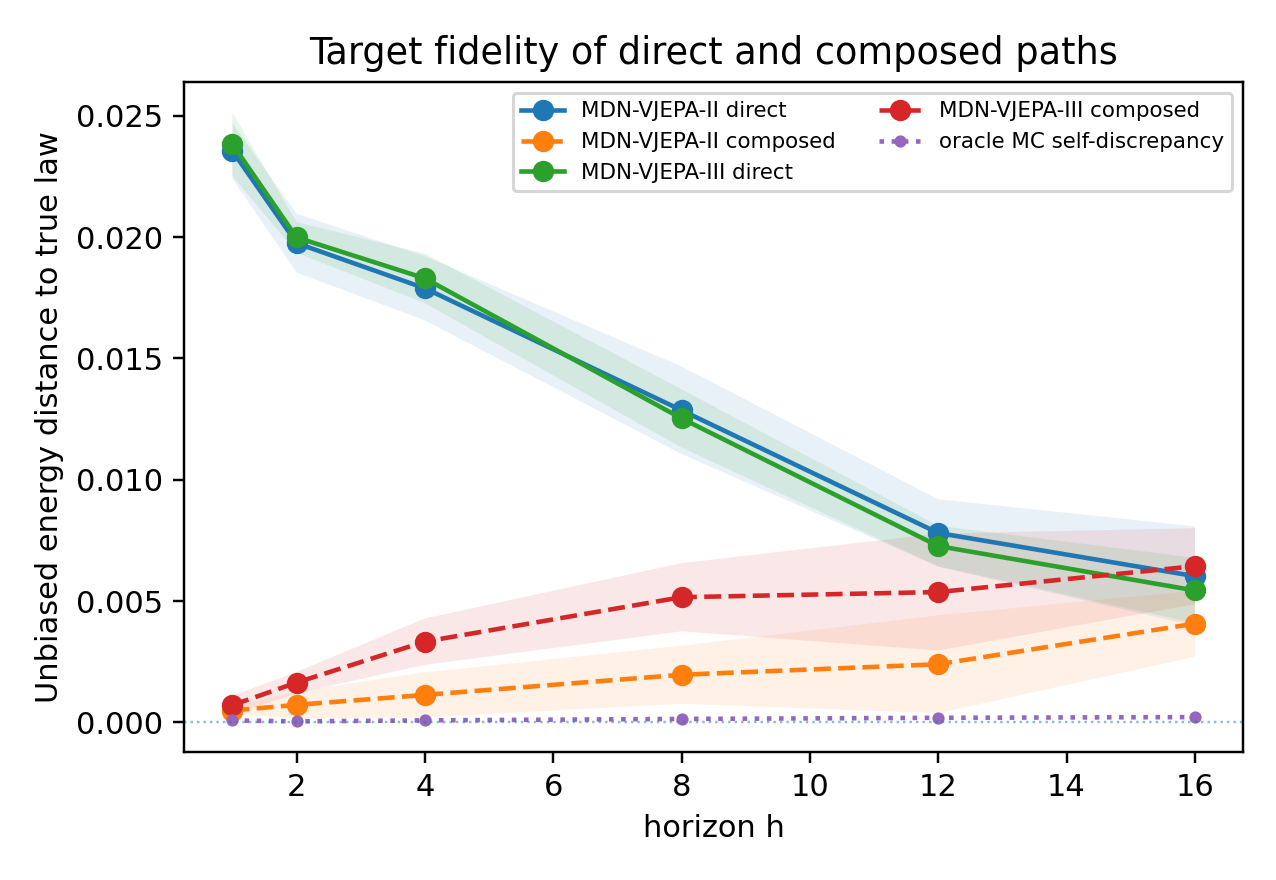}
\caption{Target fidelity of the direct and recursively composed predictive paths. Both remain close to the true conditional law. CK preserves direct fidelity while introducing a modest long-horizon composed-path tradeoff. The dotted reference shows the unbiased Monte Carlo self-discrepancy between two independent simulator populations.}
\label{fig:exp4-target-ed}
\end{figure}

The same conclusion is visible under the proper Energy Score in \cref{fig:exp4-composed-es}. Differences between the two learned models remain small relative to their common gap to the realized-target variability, and both remain close to the simulator oracle.

\begin{figure}[H]
\centering
\includegraphics[width=0.48\linewidth]{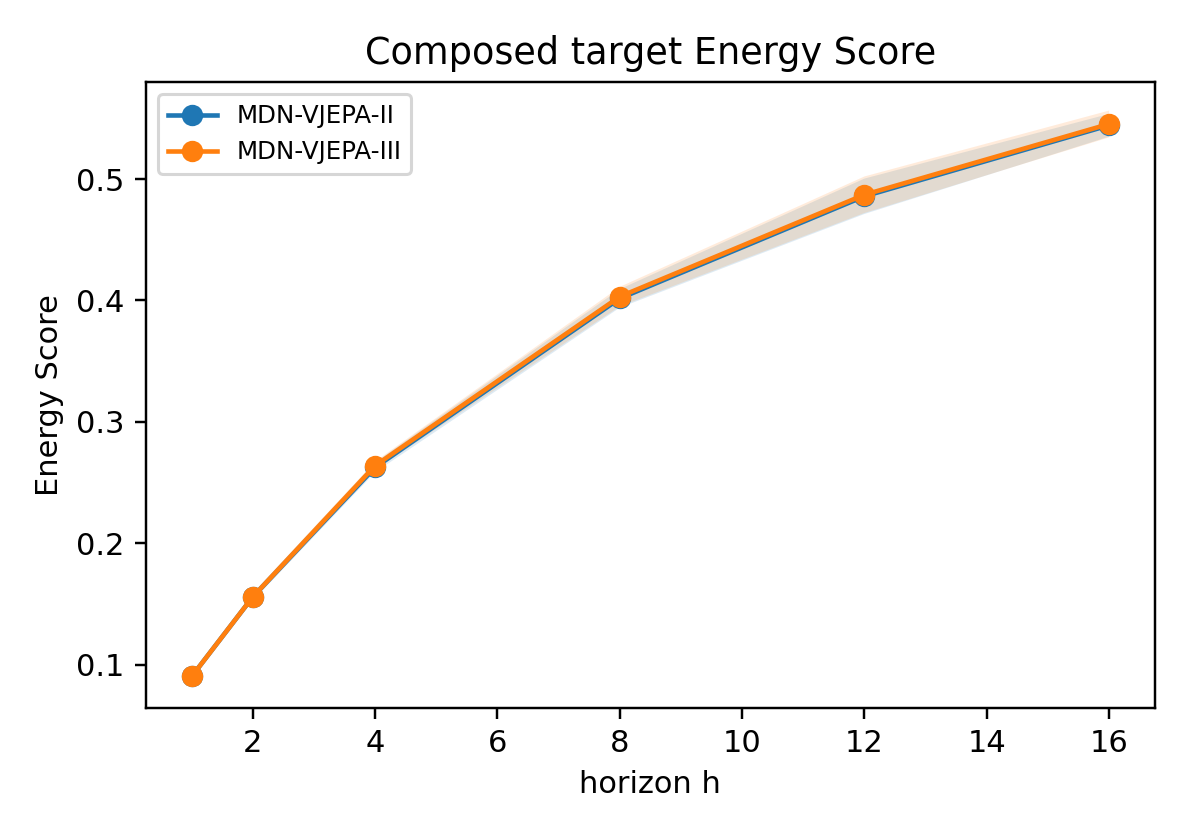}
\caption{Energy Score of the recursively composed prediction across horizons. Level-II and Level-III scores remain close throughout the prediction range and near the corresponding simulator reference, showing that the consistency improvement is not obtained through a large deterioration in sample-based predictive quality.}
\label{fig:exp4-composed-es}
\end{figure}

\paragraph{Mode and event-probability fidelity and component utilization.}
The consistency gain is not caused by eliminating a predictive mode. At $h=16$, the direct upper-mode-weight MAE is $0.02868\pm0.00189$ without CK and $0.02774\pm0.00081$ with CK, while the recursively composed upper-regime event-probability MAE is $0.04268\pm0.00330$ and $0.04454\pm0.00229$, respectively. Their horizon-wise behavior is shown in \cref{fig:exp4-branch-mae}.

\begin{figure}[H]
\centering
\includegraphics[width=0.50\linewidth]{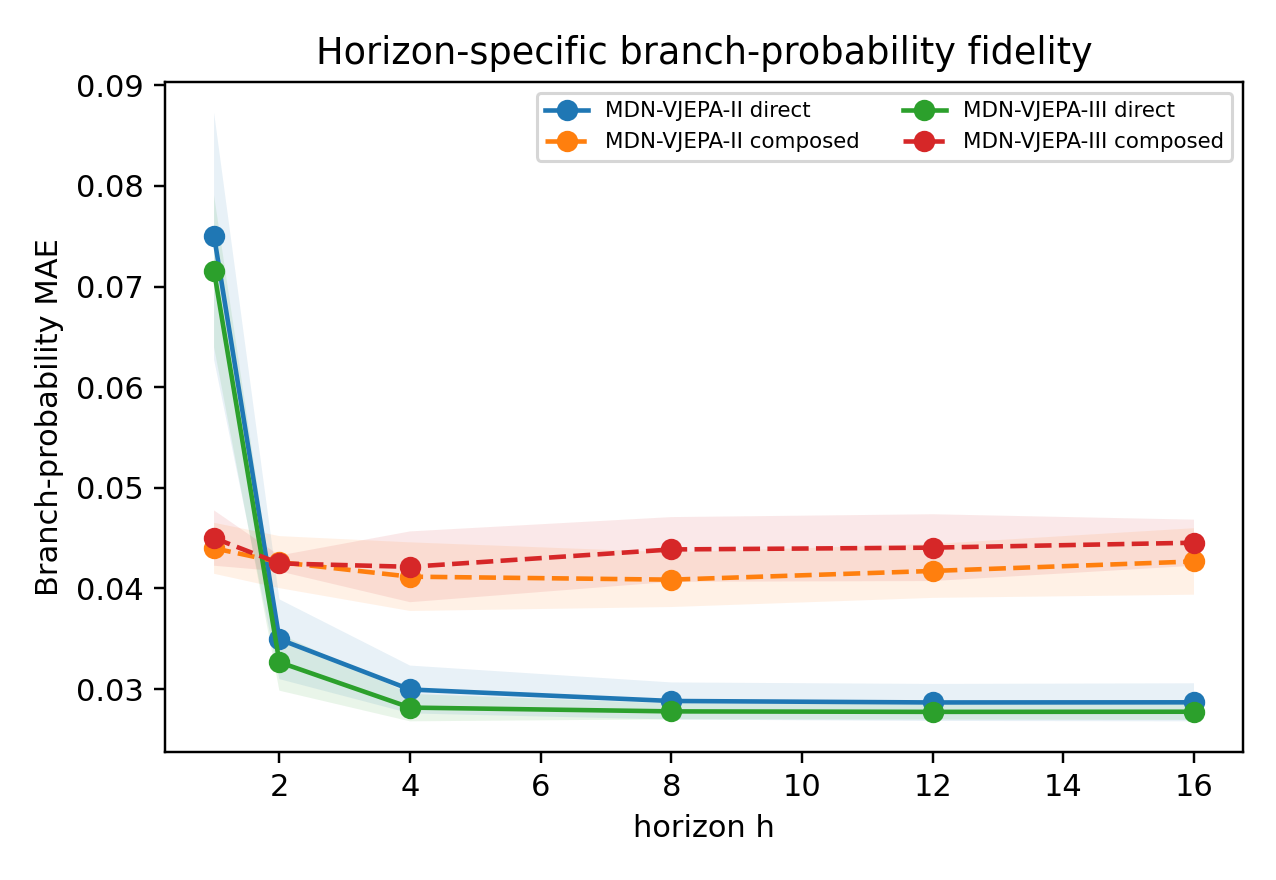}
\caption{Horizon-specific mode/event probability diagnostics. For the direct predictor, the diagnostic compares the weight of the component with the larger mean $r$ coordinate against the simulator estimate of $\Pr(r_{t+h}>0)$. For the recursive predictor, the corresponding event probability is estimated directly from composed samples. Both diagnostics remain similar with and without CK. Ground-truth horizon-specific event probabilities are estimated from independent exact-simulator rollout populations.}
\label{fig:exp4-branch-mae}
\end{figure}

The mean effective component counts are $1.8622\pm0.0192$ and $1.8681\pm0.0124$ for Levels II and III, respectively. The mean minimum average posterior responsibility is $0.4456\pm0.0108$ in both arms. Average predictive component weights are also nearly unchanged, with minimum mean weights $0.4466\pm0.0107$ and $0.4470\pm0.0076$, respectively. These diagnostics provide no evidence that CK obtains agreement through component collapse or severe routing distortion.

\paragraph{Representative endpoint distributions.}
\Cref{fig:exp4-endpoints} visualizes the horizon-$16$ predictive distributions for the pre-specified representative seed. Both models retain the multimodal endpoint geometry. The figure also illustrates why consistency and target fidelity should be reported separately: direct and composed distributions can move closer to each other without either becoming a point prediction or losing the underlying branch structure.

\begin{figure}[H]
\centering
\includegraphics[width=0.92\linewidth]{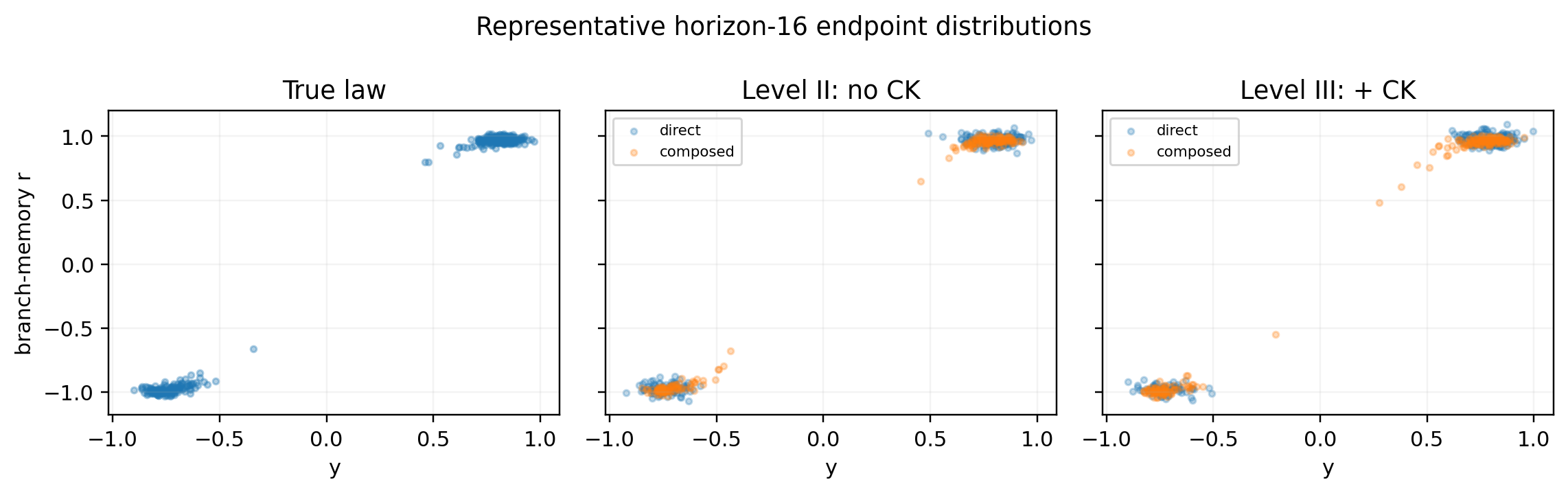}
\caption{Representative horizon-$16$ endpoint distributions for a pre-specified visualization seed. Left: samples from the exact simulator law. Middle: Level-II direct and recursively composed samples. Right: Level-III direct and composed samples after CK regularization. Both learned models retain the separated multimodal endpoint structure; CK changes the agreement between computational paths rather than collapsing them to a single mode.}
\label{fig:exp4-endpoints}
\end{figure}

Overall, Experiment~4 supports the Level-III interpretation in a deliberately qualified form. Distributional CK reduces integrated direct--composed disagreement over the tested multi-horizon range, with improvements in normalized and raw energy-distance AUC in all 5 seeds. The gain is obtained without component collapse and with essentially unchanged direct likelihood and Energy Score. The recursively composed predictor exhibits a small target-fidelity tradeoff at the longest horizons, reinforcing the interpretation of CK as a structural consistency regularizer rather than a universally improving predictive objective.

\section*{Disclaimer}

This work was developed with assistance from ChatGPT \citep{openai2026gpt} for aspects of idea development, formulation, writing and editing, experimental design and implementation. The central research direction, scientific objectives, and key methodological decisions were originated by the author.
The author estimates the overall contribution to the development of the work to be approximately $70\%{:}30\%$ between the author and ChatGPT. This estimate is necessarily subjective and is intended only to provide transparency about the extent of AI assistance. The work reflects a human--AI collaborative research process in which the human researcher provided the originating research ideas, scientific direction, methodological judgement, assessment, critical examination, and final decisions, while the AI system supported iterative formulation and execution.

In particular, AI-generated suggestions were not accepted uncritically. Mathematical derivations, technical claims, experimental designs, implementations, numerical results, interpretations, citations, and the final manuscript were examined, revised, and verified by the author. Experimental results reported in the paper were obtained from human implemented experiments rather than generated or inferred by the AI system. The author takes full responsibility for the contents and conclusions of the final work.

At the time of writing, the author does not consider the research process described above to be one that an AI system could reliably carry out independently without substantial human direction, scrutiny, correction, and refinement. Nevertheless, despite these verification efforts, errors or inaccuracies may remain, and readers are encouraged to assess the technical claims and empirical results critically.

\end{document}